\documentclass[twoside]{article}

\usepackage[preprint]{aistats2027}
\usepackage[round]{natbib}

\usepackage{booktabs}
\usepackage{hyperref}

\usepackage{amsfonts}       % blackboard math symbols
\usepackage{xcolor}         % colors
\usepackage{mathtools}
\usepackage{amsthm}
\usepackage{dsfont} % AJout Abdoulaye pour \mathds
\usepackage{cleveref}

\usepackage{minitoc}

\usepackage{tkz-tab}
\usepackage{subcaption}

\usepackage[scr=boondox]{mathalpha}
\newcommand{\Myt}{\ensuremath{s}}

\theoremstyle{plain}
\newtheorem{theorem}{Theorem}[section]
\newtheorem{proposition}[theorem]{Proposition}
\newtheorem{lemma}[theorem]{Lemma}

\theoremstyle{definition}
\newtheorem{definition}[theorem]{Definition}
\newtheorem{assumption}[theorem]{Assumption}
\theoremstyle{remark}

\usepackage{snaptodo}
\snaptodoset{block rise=2em}
\snaptodoset{margin block/.style={font=\tiny}}
\newcommand{\as}[1]{}
\newcommand{\es}[1]{}
\newcommand{\emmm}[1]{}

\begin{document}

% If your paper is accepted and the title of your paper is very long,
% the style will print as headings an error message. Use the following
% command to supply a shorter title of your paper so that it can be
% used as headings.
%
\runningtitle{Weighting-Based Strategies for Classification}

% If your paper is accepted and the number of authors is large, the
% style will print as headings an error message. Use the following
% command to supply a shorter version of the author names so that
% they can be used as headings (for example, use only the surnames)
%
%\runningauthor{Surname 1, Surname 2, Surname 3, ...., Surname n}

\twocolumn[

\aistatstitle{Weighting-Based Strategies for Classification: \\ Analysis and Local Peak Dynamics with CART Trees}

%\aistatsauthor{ Author 1 \And Author 2 \And  Author 3 }
%\aistatsaddress{ Institution 1 \And  Institution 2 \And Institution 3 } 

\aistatsauthor{ Abdoulaye Sakho \And Emmanuel Malherbe \And Erwan Scornet }

\aistatsaddress{
Artefact Research Center\\
Sorbonne Université, Université \\ 
Paris Cité, CNRS, LPSM, France
\And
Artefact Research Center\\
France
\And
Sorbonne Université, Université \\ 
Paris Cité, CNRS, LPSM, France
}
]

\begin{abstract}
Tabular learning tasks often suffer from class imbalance, which is often addressed via weighting-based strategies. In this work, we compare the naive approach (unweighted) with the weighting-based rebalancing strategies. 
Somewhat surprisingly, we demonstrate that, asymptotically, as the predictor complexity increases with sample size, weighting strategies yield predictive performance identical to unweighted approaches across all evaluation metrics, up to a shift in the final threshold. 
To go beyond, we theoretically study the estimation of a linear regression function through weighted and unweighted CART trees with limited depths. In this setting, we prove that the points in the PR space for both strategies are distributed along the same curve, thus illustrating the similar non-asymptotic performance of the two strategies. 
Then, we numerically identify more complex synthetic settings where performance gaps, evaluated via fixed classifier complexity, favor weighting strategies.  
Thus, to guide practitioners, we introduce an indicator that aims at signaling whether a data set may benefit from weighting.
\end{abstract}
%emerge between weighting and unweighted strategies.

\section{INTRODUCTION}
%% Related work on decision trees
Decision trees are the elementary brick of more powerful classifiers such as random forests \citep[][]{breiman2001random}, first-order boosted trees \citep[][]{friedman2001greedy}, XGboost\citep[][]{chen2016xgboost}, LightGBM \citep[][]{ke2017lightgbm} and Catboost \citep[][]{prokhorenkova2018catboost}.
When dealing with tabular data, tree-based classifiers are relevant classifiers for solving learning tasks, as suggested by \cite{grinsztajn2022tree, shwartz2022tabular} in their empirical comparison between tree-based and deep learning-based procedures. 
Although the recent emergence of tabular foundation models \citep{hollmann2022tabpfn,grinsztajn2025tabpfn,qu2025tabicl} has led to impressive predictive performance on benchmarks such as TabArena \citep{erickson2026tabarena}, tree-based methods remain widely used, due notably to their minimal computational and hardware requirements.

In machine learning, a data set is considered imbalanced when target classes are unevenly distributed \citep{sun2009classification,spelmen2018review}. 
Class imbalance \citep{sun2009classification,spelmen2018review} has been extensively studied in the literature due to its prevalence across diverse domain applications. For instance, class imbalance frequently arise in banking credit scoring \citep{sakho2025harnessing}, customer churn prediction \citep{nguyen2021comparison}, and medical diagnosis \citep{RF-medical-example}.

Standard learning algorithms typically assume balanced class priors, thus, their performances degrade significantly on imbalanced data as optimization favors the majority class. To mitigate this behavior, several rebalancing strategies have been proposed in the literature.
Among techniques that synthesize new minority samples, the Synthetic Minority Over-sampling Technique (SMOTE) \citep{chawla2002smote} is widely used, generating synthetic minority instances via linear interpolation between neighboring observations. Numerous variants of SMOTE have been proposed in the literature. Techniques such as Borderline-SMOTE and ADASYN focus specifically on samples near class decision boundaries \citep{han2005borderline,he2008adasyn}. Other extensions, including SVM-SMOTE and DBSMOTE, incorporate SVM classification \citep{nguyen2011borderline} or DBSCAN clustering \citep{bunkhumpornpat2012dbsmote} prior to the interpolation phase. Finally, hybrid methods like SMOTEBoost \citep{chawla2003smoteboost} integrate SMOTE directly into the model training process by resampling data within each boosting iteration.

On the other hand, a widespread approach to mitigate this bias is the use of weighting-based rebalancing strategies.
We define weighting-based rebalancing strategies as procedures that do not introduce synthetic data, but instead modify the relative importance of existing observations by assigning sample-level weights. This category include Random Oversampling (ROS), Random Undersampling (RUS), and class weighting. Specifically, ROS effectively assigns integer weights to duplicated samples, whereas RUS assigns binary weights $w_i \in \{0, 1\}$, where $0$ corresponds to discarded instances and $1$ to retained ones. %Finally, although class weighting can be implemented across various loss functions depending on the chosen estimator\es{je ne comprends pas le lien avec les pertes ici}, its underlying mechanism remains identical: applying deterministic weights to each observation based on its class label.
%These methods assign importance weights either to individual samples prior to training or directly to class losses within the objective function. By penalizing minority-class errors more heavily, these schemes enforce the model to learn relevant patterns from underrepresented classes.
%

Under class imbalance, standard evaluation metrics such as accuracy become uninformative, making the Area Under the Precision-Recall Curve (PR AUC) the recommended benchmark \citep[see][]{saito2015precision}. Because the choice of metric dictates whether a rebalancing strategy provides genuine utility, this work aims to identify the precise settings in which weighting-based rebalancing strategies are effective.

Our main contributions are detailed below:
\begin{itemize}
    %\item %We position our analysis within the infinite-sample setting, deriving the optimal tree splits for both unweighted and weighted decision trees independently of the underlying regression function.
    %We derive optimal tree splits for both unweighted and weighted decision trees independently of the underlying regression function.
    \item %We conduct a detailed case study using a linear regression function.  %Within this framework, we derive the exact theoretical predictive performance metrics for both class-weighted and unweighted strategies when paired with a CART decision tree. We demonstrate that as the tree depth $H$ increases, the number of operating points on the Precision-Recall (PR) curves grows for both strategies. While these empirical points do not coincide (as they depend on the specific split selections), they lie on the exact same underlying PR curve. Consequently, the PR AUC for both strategies converges to the same asymptotic quantity as the tree depth increases.
    With a linear univariate regression function, we prove that the points in the PR space for both strategies are distributed along the same curve, thus illustrating the similar non-asymptotic performance of the two strategies. 
    %\es{il me manque une phrase pour savoir comment interpréter ce résultat/pourquoi est-ce intéressant?}
    \item However, under fixed model complexity, we highlight more complex synthetic settings where the presence of an isolated peak induces discrepancies in  performance between both strategies. %We also demonstrate that these differences come from the fact that the peak is overestimated by the weighted model under a fixed constraint on classifier complexity, while under-estimated by unweighted.
    \item To diagnose performance gaps, we propose an indicator that aims at determining whether a given data set will benefit from weighting, with promising results on real-world data sets. %if the regression function is not well estimated.
\end{itemize}

\section{RELATED WORK}

%%Related work on Imbalance
Several studies have focused on oversampling strategies that generate synthetic samples to augment the minority class. For example, \cite{elreedy_theoretical_2023} derived the theoretical density of samples generated by the Synthetic Minority Over-sampling Technique (SMOTE). Subsequently, \cite{sakho2024we} demonstrated that SMOTE yields a distribution that closely replicates original minority instances and tends to vanish near the boundaries of the minority class support. Furthermore, \cite{ahmad2026concentration} established excess risk bounds when SMOTE is used in conjunction with a downstream classifier. 

%% Related work on weighting strategies (théoretical studies).
Complementing these generative approaches, weighting-based strategies have also been extensively investigated from a theoretical standpoint. 
\cite{scott2012calibrated} theoretically studies surrogate losses within asymmetric of imbalance data. The authors demonstrate that employing calibrated losses under these conditions yields higher excess risk bounds.
\cite{koyejo2014consistent} theoretically analyze several performance metrics for binary classification and derive optimal classifiers associated to these metrics. The optimal classifiers turn out to be metric-dependent.
\cite{xu2020class} derive the convergence rates of the expected excess weighted risk when the weighting for the plug-in classifier and the risk is different. They prove that in such a setting there is an irreducible term for the convergence of risk to the weighted Bayes risk. Then they define two robust risk that have the specificity to minimize the risk with respect to a set of weightings.
\cite{aghbalou2023sharp}, following the results from \cite{xu2020class}, propose sharper error bounds in a special setting: when the absolute number of samples within the minority class is low.
\cite{esposito2021ghost} introduce two different automated procedures to adjust the decision threshold of classifiers post-training to improve performance on minority classes, rather than modifying the original data set.
\cite{mayala2024infinite} study infinite random forests (IRF) and undersampling IRF theoretically and establish their asymptotic normality. 

% Long-tailed
%% Long tailed learning
A field closely related to our work is long-tailed learning \citep{zhang2023deep}, which addresses multi-class learning under severe imbalance by aiming to improve predictive performance on infrequent tail classes. This line of research primarily targets computer vision benchmarks such as ImageNet-LT \citep{liu2019large} and iNaturalist \citep{van2018inaturalist}. 
%Although several rebalancing techniques developed in long-tailed learning, including Focal Loss \citep{lin2017focal}, LDAM \citep{cao2019learning}, and L2RW \citep{ren2018learning}, rely on instance weighting, long-tailed settings fall outside the scope of this study. While long-tailed learning focuses on multi-class image classification, our work specifically concentrates on tabular data sets with few target classes.
%
%
%Advanced loss-based and adaptive reweighting techniques have also been developed to address class imbalance and long-tailed data distributions. 
For instance, Focal Loss \citep[][]{lin2017focal} dynamically downweights well-classified easy examples to force the model to focus on difficult minority instances. To improve generalization, the Label-Distribution-Aware Margin (LDAM) loss \citep[][]{cao2019learning} enforces larger classification margins specifically for minority classes. Alternatively, from a meta-learning perspective, Learning to Reweight (L2RW) \citep[][]{ren2018learning} utilizes an unbiased validation set to adaptively optimize instance-wise weights during training rather than relying on static heuristics.
While long-tailed learning techniques often rely on instance weighting and could thus gain some intuition from our study, long-tailed settings fall outside the scope of this study. Indeed, our work specifically focuses on tabular data with few classes.

\section{WEIGHTING STRATEGIES}
%\es{je passerais tous les titres en mincuscule : ça tiendra sur une ligne}
%\es{je déplacerais le paragraphe qui suit plus bas, tu l'as déjà dit en intro}
%\es{tu parles de weighting-based strategies mais il n'y a que CW dans cette partie ?}

%In this paper, we shed some light and compare predictive performances induced by applying no rebalancing strategy (None) and applying weighting-based resampling such as class-weighting (CW). 

In this section, we investigate weighting-based strategies through a focus on class weighting, its most common variant.

\subsection{Framework and Strategies}
We consider a binary classification setting defined by the pair of random variables $(X,Y)$, where the target $Y \in \{0,1\}$ is predicted from covariates $X \in \mathcal{X} \subset \mathbb{R}^{d}$. Let $f_X$ denote the probability density function of $X$. In the following, we use the superscript $^\star$ to denote quantities associated with the weighted distribution. \es{ok?}

%\es{pour moi, les deux équations qui suivent vont dans la définition de X*, Y*; sinon on ne voit pas bien comment tu repondères}
\begin{definition}[\cite{mayala2024infinite}]
    %In particular, 
    We denote by $(X^\star, Y^\star)$ the pair of random variables under the weighted joint distribution, defined as
    %, and $p^\star = \mathds{P}(Y^\star = 1)$ represents the weighted probability of the positive class.
    %In our weighted context, we have the following properties (from \cite{mayala2024infinite}):
\begin{align}
    \mathds{P}(Y^{\star}=1) = p^{\star} &\ge \mathds{P}(Y=1) = p \\
    f_{X^{\star}  | Y^{\star}}(x|0) &= f_{X  | Y}(x|0), \\
    f_{X^{\star}  | Y^{\star}}(x|1) &= f_{X | Y}(x|1).
\end{align}
\end{definition}

The previous equations illustrate that weighting strategies do not modify the conditional distribution of $X|Y$ while altering the joint distribution, such that $(X^{\star},Y^{\star}) \not = (X,Y)$.
%binary classification setting, i.e., when a weight-based strategy is applied during classifier training.

%In our weighted context, we have 
%$\mathds{P}(Y^{\star}=1) = p^{\star} > \mathds{P}(Y=1) = p$.
In general, the objective of weighting is to train a model on $(X^{\star},Y^{\star})$ and obtain better predictive performances. We insist on the fact that the unweighted (None) and weighting methodologies are both evaluated under the original $ (X, Y) $ distribution, which is mandatory as evaluation must always be performed on the original data distribution. In the following, we use class-weighting (CW) as weighting strategy. %We define the regression function associated to the original (resp. weighting) distribution as $\mu(x) =\mathds{P}(Y=1|X=x)$ (resp. $\mu^{\star}(x) = \mathds{P}(Y^{\star}=1|X^{\star}=x)$).
We also denote the regression functions under the original and weighted distributions as $\mu(x) = \mathds{P}(Y=1 \mid X=x)$ and $\mu^\star(x) = \mathds{P}(Y^\star=1 \mid X^\star=x)$, respectively.

\subsection{A Related Result}

In their analysis in their analysis of tree-based methods under class imbalance, \citet{mayala2024infinite} also derive the following result to move from $\mu^{\star}$ to $\mu$.

\begin{lemma}[\cite{mayala2024infinite}]
\label{lemma:main-mustar}
We have
\begin{align}
\mu^{\star}(x)=\frac{\omega_{1}\mu(x)}{\omega_{0}(1-\mu(x))+\omega_{1}\mu(x)},
\end{align}
with $\omega_{0}=p(1-p^{\star})$ and $\omega_{1}=p^{*}(1-p)$.
\end{lemma}

An alternative proof of \Cref{lemma:main-mustar} is provided in \Cref{proof_lemma:main-mustar}. Specifically, we define the mapping $g: [0,1] \to [0,1]$ such that $\mu^{\star} = g(\mu)$ and $\mu = g^{-1}(\mu^{\star})$, where for $x \in [0,1]$: $g(x) = \omega_{1}x / (\omega_{0}(1-x)+\omega_{1}x) \quad \text{and} \quad g^{-1}(x) = \omega_0 x/ (\omega_0 x +\omega_1 (1-x)).$
\Cref{lemma:main-mustar} demonstrates that $\mu^{\star}$ is a strictly increasing function of $\mu$, meaning that $\mu^{\star}$ preserves the rank ordering of the conditional probabilities. Consequently, the unweighted and class-weighted strategies yield identical probability orderings. Therefore, the total Area Under the PR Curve (PR AUC) and ROC Curve (ROC AUC) for both strategies are identical. Furthermore, threshold-dependent metrics such as precision are identical across approaches, up to a shift in the probability decision threshold.
This is a highly notable observation, as practitioners may expect weighting to structurally alter metrics by shifting the training distribution to $(X^{\star},Y^{\star})$. 
%
%However, weighting-based strategies can modify the underlying classifier training and thus produce a model with different parameters when the regression function is not originally well estimated.

All in all, \Cref{lemma:main-mustar} implies that any universally consistent classifier achieves identical asymptotic performance whether trained on $(X,Y)$ or $(X^{\star},Y^{\star})$, across PR AUC and other metrics, up to a shift on the decision threshold. Thus, we aim to investigate how these metrics behave in an infinite-sample regime under non-asymptotic model complexity.
%\es{ici je commenterais peut-être que l'implication de ce résultat est que pour toute procédure universellement consistante, qu'on l'entraîne sur $(X,Y)$ ou sur $(X*, Y*)$ on obtiendra les mêmes performances asymptotiquement pour le PR-AUC et pour toute les métriques à condition d'adadpter le seuil. La question est donc que se passe t-il d'un point de vue non asymptotique}
%We emphasize that a machine learning classifier estimates either $\mu$ or $\mu^*$ depending on the chosen rebalancing procedure. Practitioners working with models should therefore be aware that class weighting does not improve the $\textit{true}$ PR AUC or ROC AUC capabilities of the model over a baseline unweighted approach.

\section{SPLITTING PROCEDURE}
%\es{attention, tu introduis le cadre $d$-dimensionnel mais tu parles beaucoup de dimension 1, il faut que ce positionnement soit clair dans l'intro}
%In the following, we investigate whether decision trees estimating $\mu$ and $\mu^{\star}$ under unweighted (None) and weighted strategies lead to identical probability orderings, and by extension, identical metrics. 
In the following, we examine whether theoretical performance discrepancies emerge when estimating the regression function via a CART decision tree with and without weighting. %This analysis is conducted in the infinite-sample as follows.
\es{l'enchainement de l'hypothèse et de l'explication n'est pas optimale. Il vaudrait peut-être mieux remplacer l'hypothèse par une définition du type L'arbre CART théorique est l'arbre dont les splits optimisent le critère théorique puis dire que c'est la limite de l'arbre empirique} 
\begin{definition}
    Let $n$ denote the number of training samples. The theoretical classifier $\hat{m}_{k}$ is defined as the  limit of finite-sample estimator:
    \begin{equation}
    \hat{m}_{k} = \lim_{n \to \infty} \hat{m}_{k,n}.
    \end{equation}
\end{definition}

\begin{assumption}
    \label{assump-infinite-sample}
    We operate in the infinite-sample regime for both training and testing, employing either the unweighted classifier $\hat{m}_{k}$ or its weighted counterpart $\hat{m}^{\star}_{k}$. %\es{est-ce que tu as besoin du test infini ? Car tu ne regardes que des quantités sur le train au final ? Ou des précisions qui peuvent être définies comme les quantités théoriques} \as{Je dirais que oui.} 
    %Furthermore, the classifier under consideration is a default CART decision tree $\hat{m}_k$ of depth $k$  optimizing the gini criterion. 
\end{assumption}
%
% Thus, under \Cref{assump-infinite-sample}, the theoretical classifier $\hat{m}_{k}$ is defined as the  limit of finite-sample estimator:
% \begin{equation}
% \lim_{n \to \infty} \hat{m}_{k,n} = \hat{m}_{k}.
% \end{equation}
%In this section, we focus on decision trees with a maximum tree depth equal to one. 
%Let $\hat{m}_{k}$ be a CART decision tree, \citep[see][ for details]{breiman1984classification} of maximum depth equal to $1$. 
In this section, we place ourselves in the univariate setting ($d=1$). Thus, let $\Myt$ $\in \mathds{R}$ be the split position of our tree $\hat{m}_{k}$, thus producing two child nodes:
\begin{align*}
    &A_L(\Myt) = A_{L} =  \{ x \in \mathcal{X} \mid x < \Myt \} \\  &A_R(\Myt) = A_{R} = \{ x \in \mathcal{X} \mid x \geq \Myt \}.
\end{align*}
%\es{dimension 1 ici, sinon il faut écrire $x^{(j)} < t$}

%First, we would like to compute the splitting cut of our tree stump $D_t$ when the impurity measure is the \textit{gini index}. This is done on both default stump and the reweighted stump. 
%\es{$L_0$ n'est jamais défini ?} \as{j'ai rajouté une phrase dans the lemme}
\begin{lemma}
\label{lemma:ig-default}
     Under \Cref{assump-infinite-sample}, the impurity gain $L_0$ of the first split in the univariate for unweighted tree setting is
    \begin{align}
    &L_0(\Myt) =I_L I_R\left[ \frac{J_L}{I_L} -\frac{J_R}{I_R}  \right]^2,
    \end{align}
    where 
    \begin{align}
      &I_L = \int_{x \in A_L} f_X(x) \mathrm{d}x,  \; \; \; \; \; \; \; \; \;  I_R = \int_{x \in A_R} f_X(x) \mathrm{d}x,  \nonumber \\
      &J_L =  \int_{x \in A_L} f_X(x) \mu(x) \mathrm{d}x, \, \, \, J_R = \int_{x \in A_R} f_X(x) \mu(x) \mathrm{d}x. \nonumber
    \end{align}
\end{lemma}
%\es{pense aux virgules et aux points dans toutes les équations}
The proof of \Cref{lemma:ig-default} is provided in \Cref{ig-default}. Note that this derivation holds for arbitrary regression functions as soon as $I_L, I_R \neq 0$. Leveraging \Cref{lemma:ig-default}, we establish analogous results for weighted decision trees.
%\es{as soon as $I_L, I_R \neq 0$ et peut-être d'autres conditions}

\begin{theorem}
    \label{lemma:ig-ponderate}
    Under \Cref{assump-infinite-sample}, let $\varepsilon>0$ such that $p^{\star} = p + \varepsilon$. For $\varepsilon$ small enough, the split gain $L(\Myt)$ of a weighted tree  in the univariate setting can be approximated as,
    \begin{align}
    &L(s) \approx L_0(\Myt) + \varepsilon L_1(\Myt) + \mathcal{O}(\varepsilon^2),
\end{align}
where, \begin{align}
    L_1(\Myt)&= 2L_0(\Myt) \left(\frac{G_1(\Myt)}{G_0(\Myt)} - \frac{1}{1-p} \right) \nonumber \\
    & \quad + \frac{G_0^2(\Myt)}{p(1-p)}(I_RJ_L + I_LJ_R),
\end{align} 
with $G_0=\frac{J_L}{I_L} - \frac{J_R}{I_R}$ and $$G_1=-  \left[ \frac{J_L}{I_L}\frac{ J_L - pI_L}{I_L p (1-p)} + \frac{J_R}{I_R}\frac{ J_R - pI_R}{I_R p (1-p)} \right].$$
\end{theorem}
%\es{l'écriture du théorème fonctionne en dimension 1, sinon il faut avoir la direction de coupure qui intervient également}
The proof of \Cref{lemma:ig-ponderate} can be found in \Cref{proof:ig-ponderate}. 
\Cref{lemma:ig-ponderate} establishes an approximation of the impurity gain for a weighted tree under the assumption that $\varepsilon$ is sufficiently small. This simplified expression facilitates the analytical control of the gain terms. Furthermore, this result is general: it holds for arbitrary regression functions (as soon as $I_L, I_R \neq 0$) and provides a foundation for approximating the split values of weighted decision trees. Finally, the validity of this approximation for small $\varepsilon$ can be seen as a constraint on a limited weighting strength. %\es{je comprends qu'on devrait se limiter à de faibles resampling rate pour que $\varepsilon$ soit petit et que l'approximation soit vraie} \as{Oui c'est ce que je  veux dire. Tu trouves que c'est pas terrioble comme résultata ?}
%In \Cref{section:numerical_experiments} several regression functions are studied under weighting. \es{je ne comprends pas le lien}
%\es{peut-être à mettre en perspective avec la section simulation. Sinon on a du mal à voir l'intérêt de cette expression}
%

Below, we conduct an analysis of the decision tree splits with a linear regression function.

%\begin{lemma}[Optimizing gini in weighting setting]
%    \label{lemma:ig-ponderate}
% Under \Cref{assump-1}, a weighted tree stump optimizing the gini expressed as follows :
% \begin{align}
%     L(t) = \mathds{P}(Y^{\star}=1)(1-\mathds{P}(Y^{\star}=1)) - S(t,\alpha,\beta,\omega_0,\omega_1),
% \end{align}
% where there is an algebric expression for $S(t,\alpha,\beta,\omega_0,\omega_1)$.
% \end{lemma}
% The proof of \Cref{lemma:ig-ponderate} and the expression of $S(t,\alpha,\beta,\omega_0,\omega_1)$ can be found in \Cref{proof:ig-ponderate}.

\section{LINEAR APPROACH}

%To describe the characteristics of unweighted and weighted trees with linear regression functions, we introduce the following assumption. 
In this section, to derive explicit expressions for split values, tree models, and performance metrics, we focus our analysis on a univariate linear setting defined as follows.
%Depending on the specific result, our derivations invoke either \Cref{assump-1} or \Cref{ass:linear_gradient}:
\begin{assumption}
\label{assump-1}
Let $d=1$,  $X \sim \mathcal{U}([0,1])$ and assume that $\mu(x) = \mathds{P}(Y=1|X=x)=\alpha x+ \beta $ with $\alpha,\beta \in \mathds{R}^{+}$ such that $\alpha +\beta \leq 1$.
\end{assumption}

\subsection{Impurity Gain}
We denote by $s_p$, the initial split of a tree at depth 1 ($k=1$). We have from \cite{ishwaran2015effect} that $L_0(\Myt)$ is maximized for $\Myt_{p}=0.5$. %\es{$s_p$ non défini}  
% \begin{lemma}{(From \cite{ishwaran2015effect})}
% \label{lemma:IG_UNweighted_optimal_treshold}
%     Under \Cref{assump-infinite-sample} and \Cref{assump-1}, for all $\alpha$, $\beta$ we have $L_0(\Myt)=   \frac{\alpha^2}{4}(\Myt-\Myt^2)  $ and $L_0(\Myt)$ is maximized for $\Myt_{p}=0.5$.
% \end{lemma}
%\es{le résultat ci-dessus est déjà connu (cf Ishwaran ou Biau (2012)}
We remark that this split position does not depend on the parameters $\alpha$ and $\beta$.

\begin{lemma}
\label{lemma:IG_weighted_optimal_treshold}
Under \Cref{assump-infinite-sample} and \Cref{assump-1}, we have for  the weighted impurity gain denoted  $L(\Myt)$:
\begin{align}
L(\Myt) = C \cdot \frac{\Myt(1-\Myt)}{\left[ 1 - \kappa(1-\Myt) \right] \left[ 1 + \kappa \Myt \right]},
\end{align}
where $\kappa = \alpha(p^{\star} - p)/(2p(1-p)) $ and $C = \alpha^2\left[ \frac{\omega_1}{p(1-p)} \left( \frac{1}{2} -\frac{p^* - p}{2(1-p)}\right) \right]^2$.
The optimal split $\Myt_{p}^{\star}$ that maximizes $L(\Myt)$ is given by the unique valid solution of equation:
$
2\kappa \Myt^2 + 2(1-\kappa) \Myt - (1-\kappa) = 0,
$
which is: 
%the exact closed-form expression:
% \begin{equation}\label{eq:t_star_closed}
% \Myt_{p}^{\star} = \frac{\kappa - 1 + \sqrt{1 - \kappa^2}}{2\kappa} = \frac{1}{2} + \frac{\sqrt{1 - \kappa^2} - 1}{2\kappa} = \frac{1}{2} - \frac{\kappa}{2\left(\sqrt{1 - \kappa^2} + 1\right)}.
% \end{equation}
\begin{equation}\label{eq:t_star_closed}
\Myt_{p}^{\star} = \frac{1}{2} - \frac{\kappa}{2\left(\sqrt{1 - \kappa^2} + 1\right)}.
\end{equation}
\end{lemma}
%\es{tu pourrais essayer de voir quand cette expression est la plus éloignée de $1/2$ ? Il faut commenter le fait que $\kappa$ ne peut pas être plus grand que 1 aussi, pour que l'expression soit valide}
The proof of \Cref{lemma:IG_weighted_optimal_treshold} can be found in \Cref{proof:lemma:IG_weighted_optimal_treshold}.
\Cref{lemma:IG_weighted_optimal_treshold} allows us to compare the effect of weighting on tree splitting.  
This result establishes how sample weighting alters tree split selection relative to the unweighted baseline $\Myt_p = 1/2$. Note that $\kappa \in [-1, 1]$ ensures that the expression for $\Myt_p^\star$ remains valid. Specifically, for our primary regime of interest ($p^{\star} > p$, implying $\kappa \in [0,1]$), the optimal split value $\Myt_p^\star$ drops below $1/2$. This downward shift enables the tree to form more homogeneous leaves by allocating greater spatial volume to regions where the regression function takes larger values. \es{c'est marrant car si on remplaçait cette fonction linéaire par un pic, on aurait envie au contraire que la position de coupure se déplace sur la droite}
In the special case where $c_1, c_2 \in \mathbb{R}^+$, $\alpha = \frac{c_1 p}{2}$, $\beta = c_2 p$, assuming $p \ll p^\star$ and $p \ll 1$ yields the approximation $\kappa \approx \frac{c_1 p^\star}{2}$. Consequently, $\kappa$ increases monotonically with the target prevalence $p^\star$. Because $\Myt_p^\star$ is a strictly decreasing function of $\kappa$, higher values of $p^\star$ induce a more pronounced downward shift in the optimal split value $\Myt_p^\star$ relative to $1/2$. 
%\es{il y a peut-être moyen de faire plus de commentaires ici, en fonction des valeurs des coefficients}

Hereafter, we denote $\Myt_p$ and $\Myt_{p}^{\star}$ as the split values for weighted and unweighted depth-1 trees.

\subsection{Maximum Depth Equal to \texorpdfstring{$k$}{k}}

Let $t \in [0, 1]$ denote the decision threshold used to predict $\hat{Y} = 1$ within each leaf cell of the CART tree. To avoid ambiguity, we consistently use the term threshold for classification rules on predicted probabilities, whereas split value refers to feature separations. Formally, for a given leaf $A$, the decision rule is
\begin{align}
    \mathds{E}[Y \mid X \in A] \ge t &\Longrightarrow \hat{Y} = 1.
\end{align}

\begin{figure*}[h]
    \centering
    \includegraphics[width=0.4\linewidth]{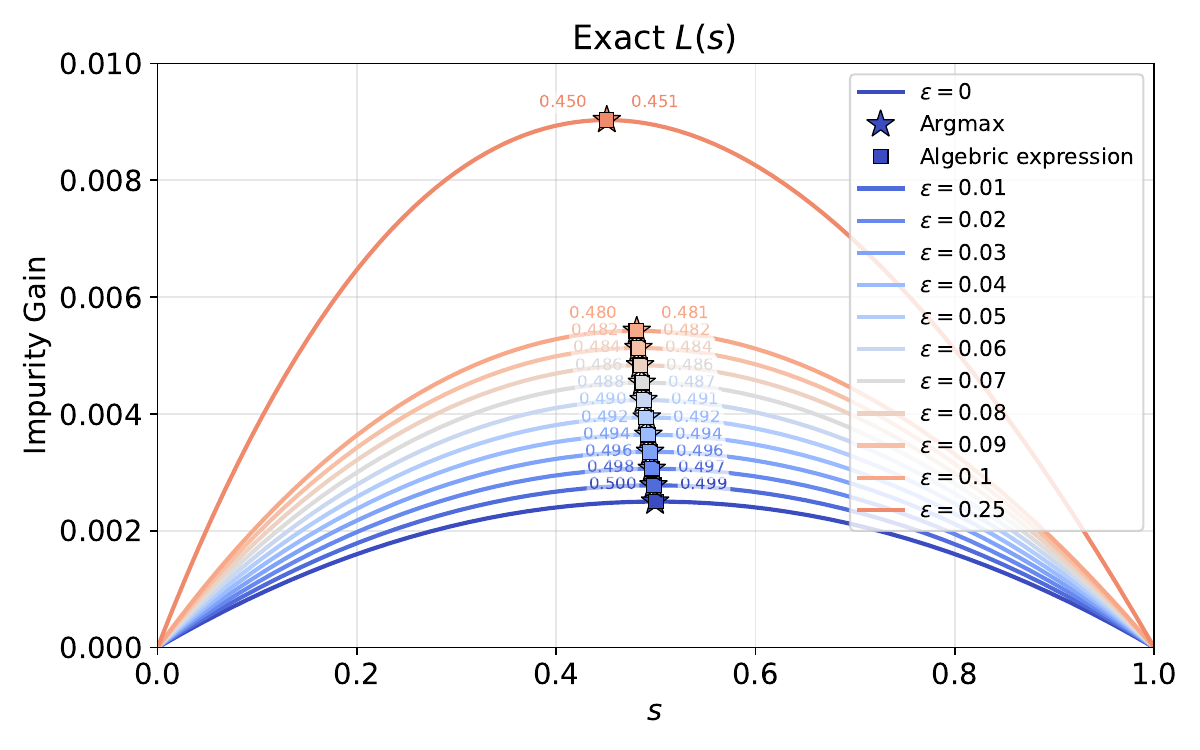}
    \includegraphics[width=0.4\linewidth]{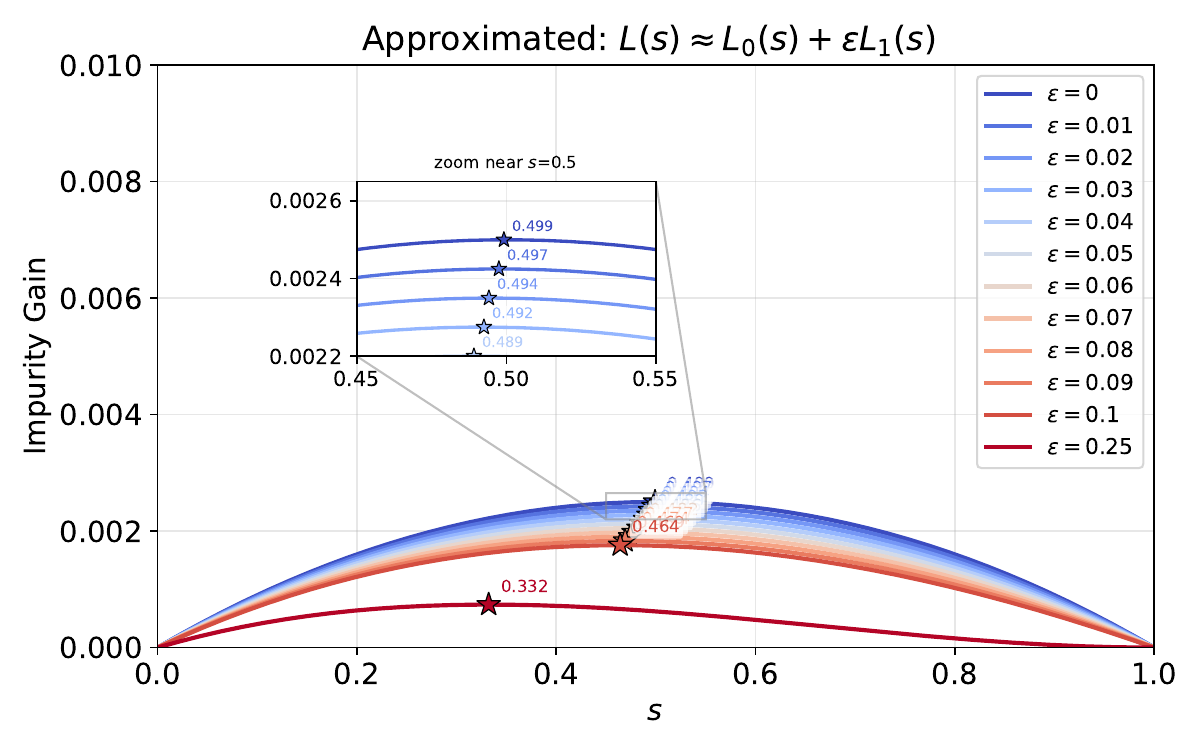}
    \caption{Impurity gain function of first split for linear regression function with $\alpha=0.2$ and $\beta=0.05$.}
    \label{fig:IG-linear}
\end{figure*}
In the subsequent analysis, we focus on Precision-Recall curves and their Area Under the Curve (AUC), as they provide robust, threshold-free metrics for class imbalance. %Furthermore, achieving identical ranking metrics implies that other performance metrics will also align following an appropriate adjustment of the probability threshold.

%We assume that the trees are grown to a maximum depth $k \geq 1$, containing $2^k - 1$ internal split nodes. Let $\{\Myt_p^{h}\}_{h=1}^{2^k - 1}$ and $\{\Myt_p^{\star h}\}_{h=1}^{2^k - 1}$ denote the collections of split values for the standard and weighted trees, indexed level-by-level from left to right. In this section, we reorder these values monotonically using permutation  $\sigma$ and $\sigma^{\star}$ respectively such that $\Myt_p^{(\sigma(1))} \le \Myt_p^{\sigma(2)} \le \dots \le \Myt_p^{\sigma(2^k - 1)} \in [0,1]$ for the standard tree, and $\Myt_p^{ \sigma^\star(1)} \le \dots \le t_p^{\sigma^\star(2^k - 1)}$ for the weighted tree.
Let $\{\Myt_p^{h}\}_{h=1}^{2^k - 1}$ and $\{\Myt_p^{\star h}\}_{h=1}^{2^k - 1}$ denote the internal split values for a depth-$k$ unweighted and weighted tree, respectively. We sort these splits monotonically via permutations $\sigma$ and $\sigma^\star$, yielding $\Myt_p^{\sigma(1)} \le \dots \le \Myt_p^{\sigma(2^k - 1)}$ and $\Myt_p^{\sigma^\star(1)} \le \dots \le \Myt_p^{\sigma^\star(2^k - 1)}$.

In \Cref{appendix:supp_theoretical_results}, we provide simplified algebric expressions of the trained decision trees, along with their corresponding performance metrics. %with respect to split value and the optimal metrics values.
%\es{en lisant cela, je ne sais pas trop ce que je m'attends à trouver dans cet appendix} \as{essai}

%\es{tu porurais juste mettre $Precision_{\hat{m}_{k}, s}$ non ?}
\begin{theorem}
\label{th:depthk-prec-rec-correspondance}
Under \Cref{assump-infinite-sample} and \Cref{assump-1}, both the unweighted decision tree $\hat{m}_{k}$ and weigthing tree $\hat{m}^{\star}_{k}$ satisfy the exact same Precision–Recall relationship across all output probabilities. Expressed for $\hat{m}_{k}$, this relationship is given by:
\begin{align}
&Precision_{\hat{m}_{k}, t}\left(Recall_{\hat{m}_{k}, t}\right)= \frac{\alpha}{2} + \beta   \\
& + \frac{\alpha}{2}\frac{-\beta+\sqrt{\beta^{2}+2\alpha\left(\frac{\alpha}{2}+\beta\right)\left(1-\text{Recall}_{\hat{m}_{k}, t}\right)}}{\alpha}, \nonumber
\end{align}
%where $g^{+}(\hat{m}_{k},s)$ yields the smallest split point $\Myt_p^{(\sigma(i))}$ such that $\mathds{E}[Y \mid \Myt_p^{\sigma(i-1)} \leq X \leq \Myt_p^{\sigma(i)}] \leq s \leq \mathds{E}[Y \mid \Myt_p^{\sigma(i)} \leq X \leq \Myt_p^{\sigma(i+1)}]$.
where $Precision_{\hat{m}_{k}, t}$ denotes the precision of decision tree $\hat{m}_{k}$ at probability threshold $s$. 
%Furthermore, this functional relationship holds analogously for the weighted tree $\hat{m}_{k}^{\star}$.
\end{theorem}
The proof of \Cref{th:depthk-prec-rec-correspondance} can be found in \Cref{proof_th:depthk-prec-rec-correspondance}.
\Cref{th:depthk-prec-rec-correspondance} provides an explicit expression of Precision as a function of Recall that is fundamentally independent of the chosen split sequence. Thus, in the Precision-Recall space, the operating points are strictly constrained to a curve that does not depend on the specific partition values. In other words, the underlying Precision-Recall curves of the unweighted and weighted strategy are identical. 
However, the discrete positioning of their empirical operating points differs, leading to distinct Riemann sum configurations when approximating the AUC. 

\section{PR CURVES: GENERALIZATION}

In this section, we extend our analysis of performance metrics from linear models to broader classes of regression functions.

%\subsection{Univariate Case With Increasing Regression Function}

%The unweighted tree $\hat{m}_{k}$ and weighted tree $\hat{m}_{k}^{\star}$ are evaluated on a test set drawn from the joint distribution $(X,Y)$. Consequently, we have:

\begin{theorem}
    \label{th:general-prec-rec-correspondance}
    Under \Cref{assump-infinite-sample}, both the unweighted decision tree $\hat{m}_{k}$ and weighting tree $\hat{m}^{\star}_{k}$  operating at an output probability threshold $s$, satisfy the PR relationship. Expressed for $\hat{m}_{k}$, this relationship is given by:
    \begin{align}
    Precision_{\hat{m}_{k}, t} =  \frac{Recall_{\hat{m}_{k}, t} \times \mathds{P}\left(Y=1\right)}{\mathds{P}\left(\hat{m}_{k}(X)=1\right)}.
\end{align}
%This functional relationship holds analogously for the weighted tree $\hat{m}_{k}^{\star}$.
\end{theorem}

The proof of \Cref{th:general-prec-rec-correspondance} can be found in \Cref{proof_th:general-prec-rec-correspondance}.  \Cref{th:general-prec-rec-correspondance} is available for all regression functions, univariate or not, and without assumptions on the distribution of $X$. The key insight of this theorem is that the predicted positive probability $\mathds{P}\left(\hat{m}_{k}(X)=1\right)$ is the sole term carrying explicit dependence on tree split locations. Consequently, whenever this term can be reparametrized purely through the underlying regression parameters, precision decouples from split values, yielding analytical PR curves for both weighted and unweighted decision trees.
Next, we subsequently establish assumptions that extend this decoupling property to a broader class of regression functions than considered in \Cref{th:depthk-prec-rec-correspondance}. %\es{ici, si on considère un weighted CART, il faut changer la probabilité du dénominateur et le numérateur $p$ en $p*$. Donc on n'a plus le même lien fonctionnel. L'intérêt de ce résultat n'est pas très clair.} \as{le numérateur ne change pas que l'on soit dans le cas pondéré ou non selon moi.}

\begin{assumption}
\label{assump:increasing-regfunc}
    $X \in [a,b]$ with $a,b \in \mathds{R}$ and $b>a$.
    $t \rightarrow \mu(t)f_X(t)$ has a primitive that we denote $F_{Y}$. Furthermore $F_{Y}^{-1}$ exists and $\mu$ is increasing.
\end{assumption}

\begin{theorem}
\label{th:increasing-prec-rec-correspondance}
    Under \Cref{assump-infinite-sample} and  \Cref{assump:increasing-regfunc}, both the unweighted decision tree $\hat{m}_{k}$ and weigthing tree $\hat{m}^{\star}_{k}$ satisfy the exact same Precision–Recall relationship across all output probabilities. Expressed for $\hat{m}_{k}$, this relationship is given by:
    \begin{align}
        &Precision_{\hat{m}_{k}, t} \\
        &= \frac{Recall_{\hat{m}_{k}, t} \mathds{P}(Y=1)}{ \mathds{P}(X\geq F_{Y}^{-1}\left[ F_{Y}(b)- Recall_{\hat{m}_{k}, t} \mathds{P}(Y=1) \right])}. \nonumber
\end{align}
%\es{on a des $Y$ et des $1$ bizarre ici}
%where $g^{+}(\hat{m}_{k},s)$ yields the smallest split point $\Myt_p^{(\sigma(i))}$ such that $\mathds{E}[Y \mid \Myt_p^{\sigma(i-1)} \leq X \leq \Myt_p^{\sigma(i)}] \leq s \leq \mathds{E}[Y \mid \Myt_p^{\sigma(i)} \leq X \leq \Myt_p^{\sigma(i+1)}]$ and $\{\Myt_p^{\sigma(i)}\}_{h=1}^{2^H-1}$ is the set of ordered split positions.
%where $Precision_{\hat{m}_{k}, t}$ denotes the precision of decision tree $\hat{m}_{k}$ at probability threshold $s$.
%Furthermore, this functional relationship holds analogously for the weighted tree $\hat{m}_{k}^{\star}$.
\end{theorem}
The proof of \Cref{th:increasing-prec-rec-correspondance} can be found in \Cref{proof_th:increasing-prec-rec-correspondance}.
%\as{Preuve similaire en 1d AVEC depth=1 sans l'hypothèse de croissance de $mu$ et en ne s'intéressant qu'aux cas non pathologiques.}
\Cref{th:increasing-prec-rec-correspondance} generalizes  \Cref{th:depthk-prec-rec-correspondance}. 
%
%\Cref{th:increasing-prec-rec-correspondance} serves as a general framework for \Cref{th:depthk-prec-rec-correspondance}. 
This result demonstrates that regardless of the weighting strategy, the PR curves are the same for both settings and depend solely on the true regression parameters. Derived as a special case of \Cref{th:general-prec-rec-correspondance}, this formulation imposes specific assumptions that allow $\mathds{P}\left(\hat{m}_{k}(X)=1\right)$ to be reparametrized purely through the underlying regression parameters, yielding identical analytical PR curves for both weighted and unweighted decision trees.
%Consequently, as the tree depth $k \to \infty$ (model complexity increases), the performance metrics for both strategies converge toward the same quantity.%\es{il faudrait le détailler, je ne suis pas sûr de voir ça sur l'expression}

\section{NUMERICAL ILLUSTRATIONS}
\label{section:numerical_experiments}
%In this section, we illustrate our findings using both simulated and real-world data sets.

%\subsection{Linear Setting}
Within the linear setting, we conduct numerical experiments to validate our theoretical derivations.

%\es{je trouve que la description des résultats manque un peu de précision. J'aurais bien mis des choses comme : The curve profiles of the exact criterion and its approximation (established for small $\varepsilon$) are similar for $\varepsilon$ ranging from 0 to 0.25. For these values, while the exact shape differ, the position of maxima are aligned. Si tu ne commentes pas les grandes valeurs de epsilon tu peux les enlever du graphique (je trouve que c'est sympa de les garder mais elle amène des questions)}
\Cref{fig:IG-linear} illustrates the impurity gain function, comparing the original formulation with the approximation derived in \Cref{lemma:ig-ponderate} under the assumption that $\varepsilon$ is sufficiently small (where $p^{\star} = p + \varepsilon$). 
%Although the impurity gain curves exhibit distinct profiles, the positions of their respective maxima remain closely aligned for small values of $\varepsilon$. This suggests that our approximation provides a robust surrogate for the original function in the limit of small $\varepsilon$.
The curve profiles of the exact criterion and its approximation (established for small $\varepsilon$) are similar for $\varepsilon$ ranging from 0 to 0.1. For these values, while the exact shape differs, the positions of maxima are aligned, thus vlidating our \Cref{lemma:ig-ponderate}.
Furthermore, \Cref{fig:IG-linear} demonstrates that our weighted tree split formulation from \Cref{lemma:IG_weighted_optimal_treshold} closely aligns with the empirically computed argmax, thereby validating our derivation.

\begin{figure}[h]
    \centering
    \includegraphics[width=0.7\linewidth]{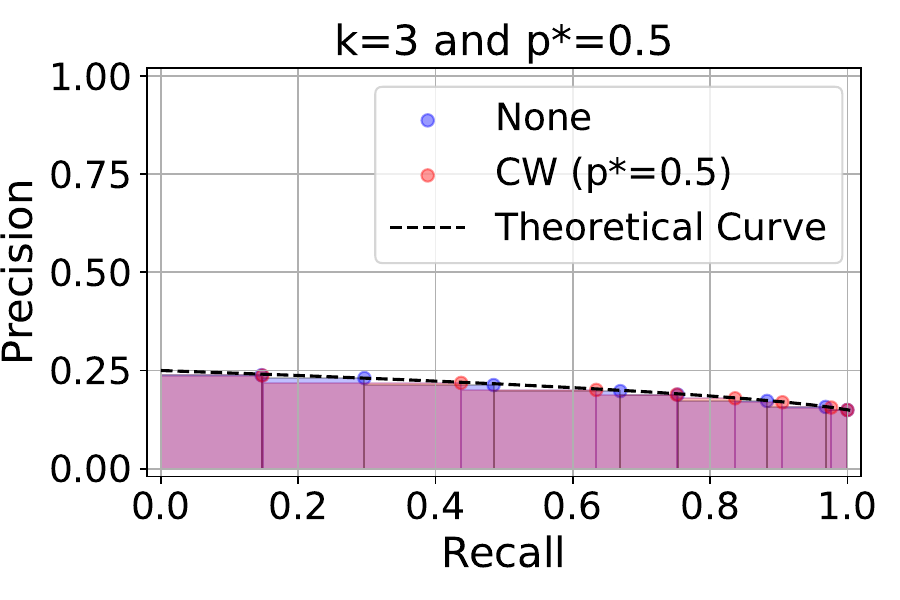}
    \caption{PR curves of  None and CW for a linear regression function with $\alpha=0.2$ and $\beta=0.05$.}
    \label{fig:prcurves-linear}
\end{figure}

\Cref{fig:prcurves-linear} illustrates the PR curves for both the standard decision tree, unweighted shown in blue ($p=0.15$) and  class-weighted in red, for $p^{\star}=0.5$ and maximum depth equal to $3$. Each classifier is trained on $500.000$ samples and evaluated on $500.000$. 
We also display the theoretical curve derived from \Cref{th:depthk-prec-rec-correspondance}. 
We observe that although the empirical points for the unweighted and weighted models are distinct, they all align precisely with our derived theoretical curve.
Furthermore, in \Cref{appendix:supp_num_exp}, we illustrate on \Cref{fig:prcurves-linear-all} settings with distinct value of $p^{\star} \in \{0.2, 0.5, 0.7\}$, and maximum depth within $\{2, 3, 4, 5\}$.

\section{EMPIRICAL COUNTERPART}
 
Although practitioners often observe empirical performance gains from weighting, our theoretical analysis demonstrates that under well-estimated regression functions or suitable structural assumptions, we can show that both methodologies yield identical metrics such as PR AUC.
%under infinite-sample conditions, fixed model complexity, and a linear regression function, weighted and unweighted strategies yield identical PR curves. 
%In this section, we investigate the underlying mechanisms that drive the predictive advantages of weighting in practice when the regression is not well-estimated.
This section investigates the source of the empirical advantages of weighting when the regression function is not well-estimated.

\subsection{\textit{ExpSetting}}
We consider a regression function, denoted as \textit{ExpSetting}, defined by the following expression:
$$\mu(x) = \exp(-35x) + \exp(-35(1-x)).$$
For our experimental analysis, we augment this base function with an isolated Gaussian peak $\nu$, given by: $$\nu(x) = a\exp\left( -\frac{1}{2} \left(\frac{x - b}{ c}\right)^2 \right),$$ where $a,b,c \in \mathds{R}$ represent respectively the peak height, localisation and width.
\Cref{fig:expsetting-regcurves} illustrates the \textit{ExpSetting} function with several peak parameters.
%\es{le pic n'apparaît pas dans l'expression ci-dessus. Il faut l'ajouter sinon on ne comprend pas d'où il vient}
%\es{Ce serait bien d'avoir un graphique pour visualiser la fonction, avec et sans pic} 
\begin{figure}[h]
    \centering
    \includegraphics[width=1.0\linewidth]{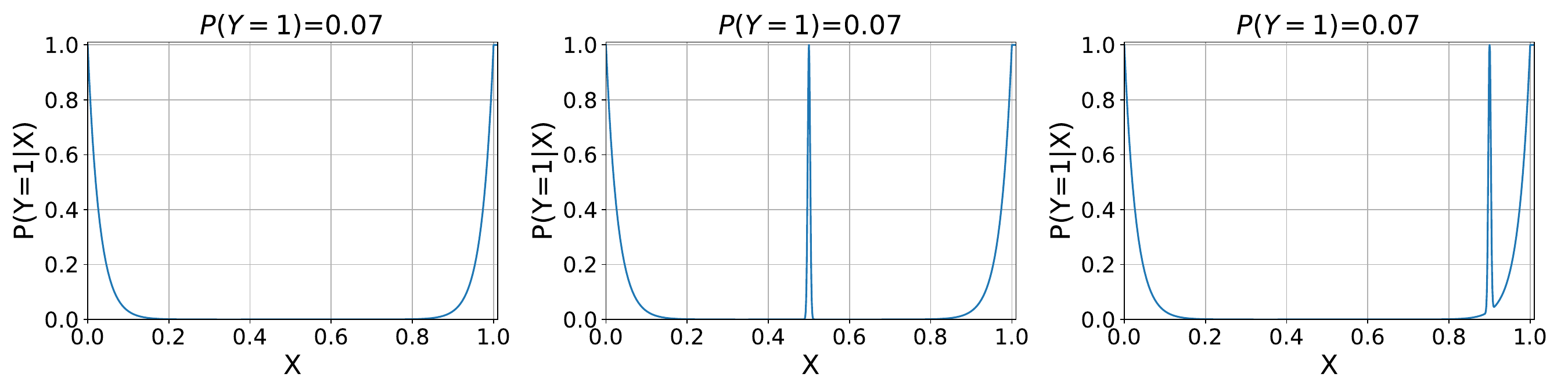}
    \caption{\textit{ExpSetting} with a peak across localisations.}
    \label{fig:expsetting-regcurves}
\end{figure}

As a first illustration,  \Cref{fig:boxplot_main_alone} presents the PR-AUC performance gaps for the \textit{ExpSetting} configuration, comparing scenarios with and without a localized peak at $b=0.5$ (evaluated at $p^{\star}=0.5$ and depth $k=4$). We observe that the weighting strategy yields better predictive performance in the presence of the peak. \Cref{fig:pr_reg_curves_no_peak} and \Cref{fig:pr_reg_curves_with_peak_0-5} in \Cref{appendix:supp_num_exp} extend these results across diverse maximum depths and $p^{\star}$ values, alongside the estimated regression functions and PR curves. These detailed appendix evaluations show that the performance gain arises because weighting enables the tree to successfully capture the localized peak, whereas unweighted models fail to detect it. 
%

%We observe that the weighting strategy yields an improved PR AUC at least in the presence of the localized peak at $x=0.5$.%\es{je trouve la formulation un peu ambigüe: je comprends que si le pic se trouve  à $x=0.4$ on ne voit pas d'amélioration.}
%\es{j'aurais peut-être remonté dans le corps du texte une sous figure de 15 pour montrer visuellement l'intérêt de la repondération}%\es{je ne suis pas fan de la sous figure de droite, je mets personnellement du temps à la comprendre. Et c'est une figure qui ne véhicule pas beaucoup d'information (information binaire pour chaque position du pic). Pour ma part, je garderais seulement celle de gauche}
\begin{figure}[h]
    \centering
    \begin{subfigure}[b]{0.23\textwidth}
        \centering
        \includegraphics[width=1.0\linewidth]{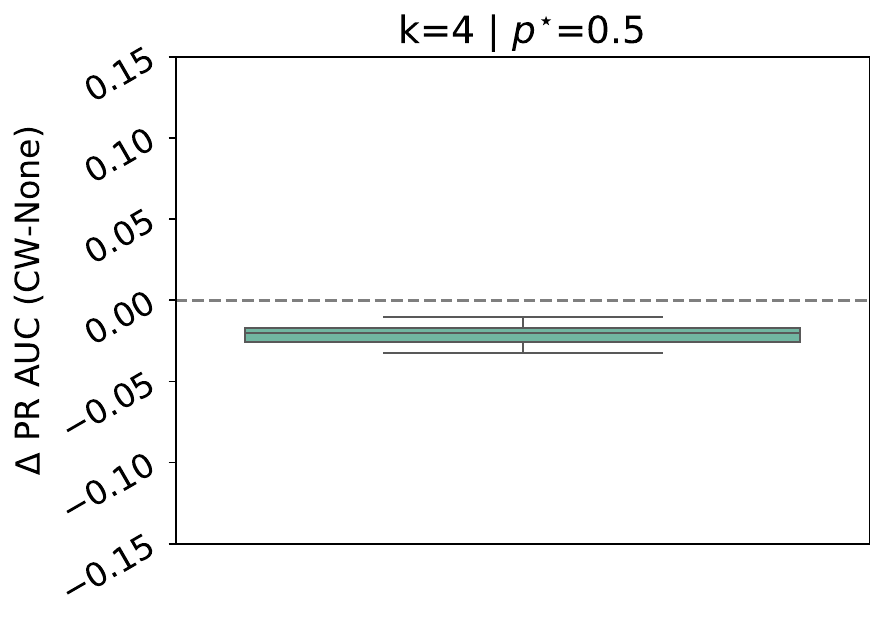}
        \caption{Without peak.}
        \label{fig:boxplot_alone_without_peak}
    \end{subfigure}\hfill
    \begin{subfigure}[b]{0.23\textwidth}
        \centering
        \includegraphics[width=1.0\linewidth]{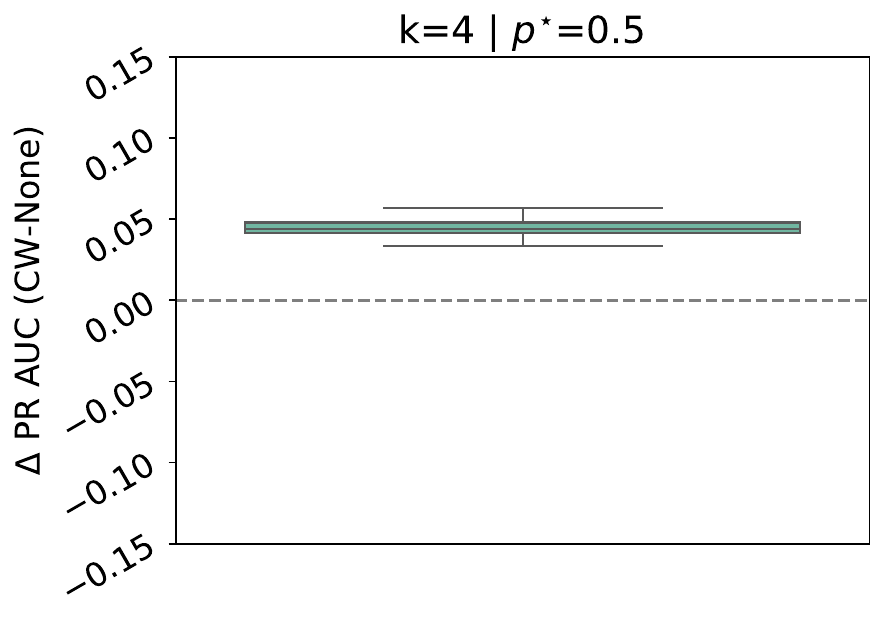}
        \caption{ With peak. }
        \label{fig:boxplot_alone_with_peak}
    \end{subfigure}
    \caption{PR-AUC performance gap between class-weighted (CW) and unweighted models under the \textit{ExpSetting} data distribution. %Detailed results are provided in \Cref{fig:pr_reg_curves_no_peak} and \Cref{fig:pr_reg_curves_no_peak} in \Cref{appendix:supp_num_exp}.
    }
    \label{fig:boxplot_main_alone}
\end{figure}

\begin{figure*}[h]
    \centering
    \includegraphics[width=0.75\linewidth]{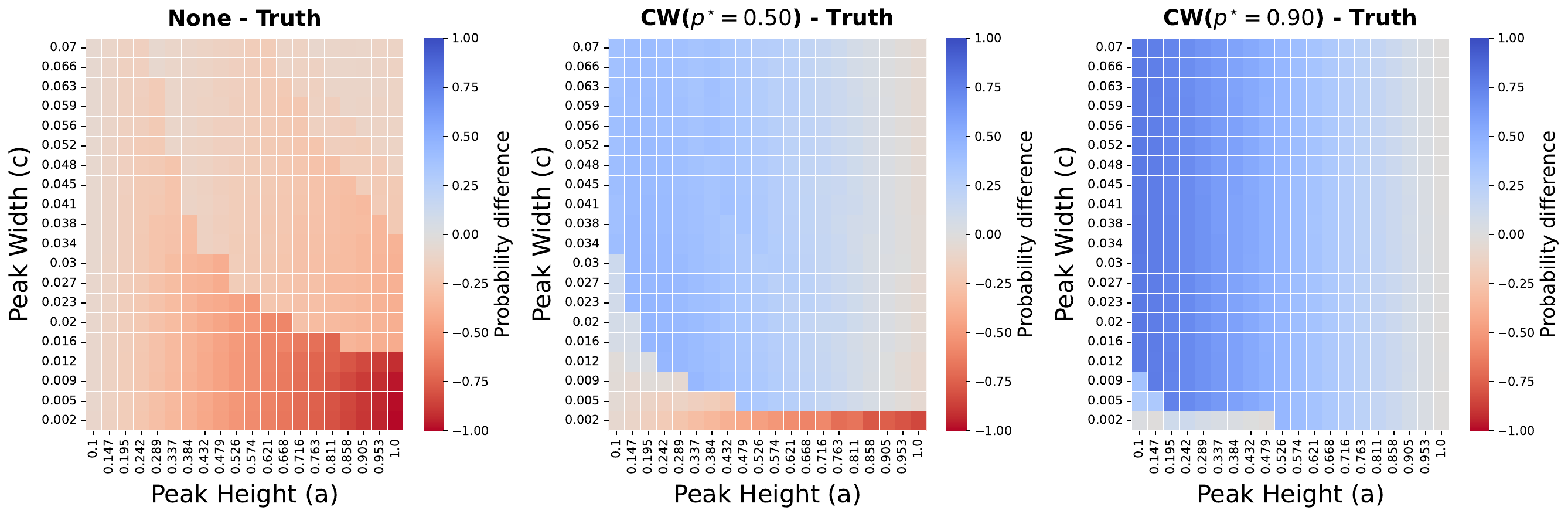}
    \caption{\textit{ExpSetting}: CART tree difference between predicted and ground-truth peak probability across peak widths and heights ($b=0.5$), for the unweighted (None) alongside weighting with $p^{\star} = 0.5$ and $p^{\star} =0.9$ ($k= 4$).}
    \label{fig:sim-heatmap-expsetting}
\end{figure*}

%\Cref{fig:sim-heatmap-expsetting} presents a scenario where both unweighted and weighted decision trees are evaluated on the \textit{ExpSetting} regression function\es{comme pour les autres expériences, je ne pense pas que tu aies besoin de le repréciser}. 
In what follows, we conduct a numerical analysis to investigate the mechanisms governing peak detection.
In \Cref{fig:sim-heatmap-expsetting}, the left panel displays the difference between the unweighted classifier's probability predictions and the true peak probability. The middle and the right panels illustrate the same difference for the class-weighted classifier for respectively $p^{\star}=0.5$ and $p^{\star}=0.9$. We fix the peak location at $b=0.5$ and vary its height $a$ and width $c$.
%\es{d'où l'intérêt d'introduire un modèle général en début de section pour ensuite dire seulement que tu fixes tels paramètres ou que tu fais varier tel autre dans chaque sous section. Cela permet de renforcer l'idée d'un cadre global} 
%
%
We observe that the unweighted model fails to detect the peak under most configurations; its heatmap is predominantly red, indicating poor detection except in the regime of large widths and high heights. More precisely, the baseline strategy introduces a negative bias, underestimating the true peak probability. 
In contrast, the weighted approach captures nearly all peaks when $p^{\star}=0.5$. When shifting to $p^{\star}=0.9$, the model's sensitivity improves further, allowing it to detect even subtle peaks with low height and small width. %However, while weighting introduces an overestimation of the peak amplitude, this behavior is particularly valuable in practical scenarios for flagging rare events. Moreover, any structural overestimation can be easily regulated after training by leveraging the inverse relationship that maps $\mu^*$ to $\mu$ (see \Cref{lemma:main-mustar}).Consequently, by altering the split selection criteria, class weighting facilitates easier peak detection compared to the unweighted baseline, thereby yielding superior predictive performance for a fixed tree depth.
%Although weighting overestimates peak amplitude, this bias may detect localized peak. Through the inverse mapping $\mu^{\star} \to \mu$ (see \Cref{lemma:main-mustar}), . %\es{cela me paraît dangereux de dire cela : 1) est-ce vrai ? Le mapping va corriger une partie du biais mais il n'y aucune garantie qu'il corrige tout le biais ? 2) Si c'était vrai, cela voudrait dire qu'en appliquant un seuil bien choisi à l'estimateur None on arriverait à obtenir de bons résultats aussi (ce n'est pas tout à fait équivalent mais il faudrait le discuter plus en détail pour qu'on comprenne la subtilité} \as{je suis pas sûr de vouloir en parler finalement}.
Consequently, by altering split selection, class weighting may enhance peak detection and boost performance for a fixed tree depth. 
%\es{on pourrait aussi dire qu'il y a biais de sous-estimation}
%\es{mais dans ce cas il y a un biais de surestimation. Tu peux dire que cela peut être intéerssant pour détecter des observations rares}
%
We observe a similar phenomenon when increasing the dimensionality of the input features to $d=5$ in the presence of an isotropic peak (see \Cref{fig:sim-heatmap-expsetting-dim5} in \Cref{appendix:supp_num_exp}). %In this higher-dimensional setting, although a larger peak width is required to facilitate detection, the majority of these localized peaks are still successfully captured exclusively by the weighting strategy.
% \begin{figure}[t]
%     \centering
%     \includegraphics[width=1.0\linewidth]{images/heatmap_peak_area_pstar_0.5_depth_4.pdf}
%     \includegraphics[width=1.0\linewidth]{images/heatmap_peak_area_pstar_0.9_depth_4.pdf}
%     \caption{\textit{ExpSetting} regression function: peak probability minus Truth probability with respect to peak width and height.}
%     \label{fig:sim-heatmap}
% \end{figure}

%\es{il faudrait mettre probability difference en légende car une probability négative ça chose (même si c'est expliqué dans la légende)}
%\es{peut-être plutôt "Peak height bias" plutôt que Probability au-dessus des graphiques}
%\es{on voit donc que None sous estime tout le temps}
%\as{Avec la correspondance entre mu et mustar, on peut corriger la surestimation de CW si besoin. Pertinent d'en parler ? Et ce serait sympa d'avoir les fonctions de regression avec les courbes PR en annexe.}
%\clearpage

\subsection{\textit{SinSetting}}

To assess the generality of our previous findings, we extend our analysis to a sinusoidal regression function, denoted \textit{SinSetting}, for which complete experimental details are provided in \Cref{appendix:sinsetting}. The \textit{SinSetting} regression function is defined as:
$$\mu(x) = 0.15 sin\left(2\pi x+1.2\right) +0.16.$$ As before, this base function is augmented with the same peak $\nu$ (illustrated in \Cref{fig:sinsetting-regcurves}) 
%
%In the following, we look at another specific regression function named \textit{SinSetting}  and defined as follows: $$\mu(x) = 0.15sin\left(2\pi x+1.2\right) +0.16.$$ Once again, we augment this base function with the same isolated Gaussian peak $\nu$. 
%
%\Cref{fig:sinsetting-regcurves} illustrates the \textit{SinSetting} function. %, demonstrating the presence and absence of a peak across various localisations.
Replicating the experimental setup used for \textit{ExpSetting} yields consistent findings: when model capacity cannot fully resolve the underlying regression function, sharp peaks can induce performance gaps in favor of weighting. % between weighted and unweighted strategies.
\begin{figure}[h]
    \centering
    \includegraphics[width=1.0\linewidth]{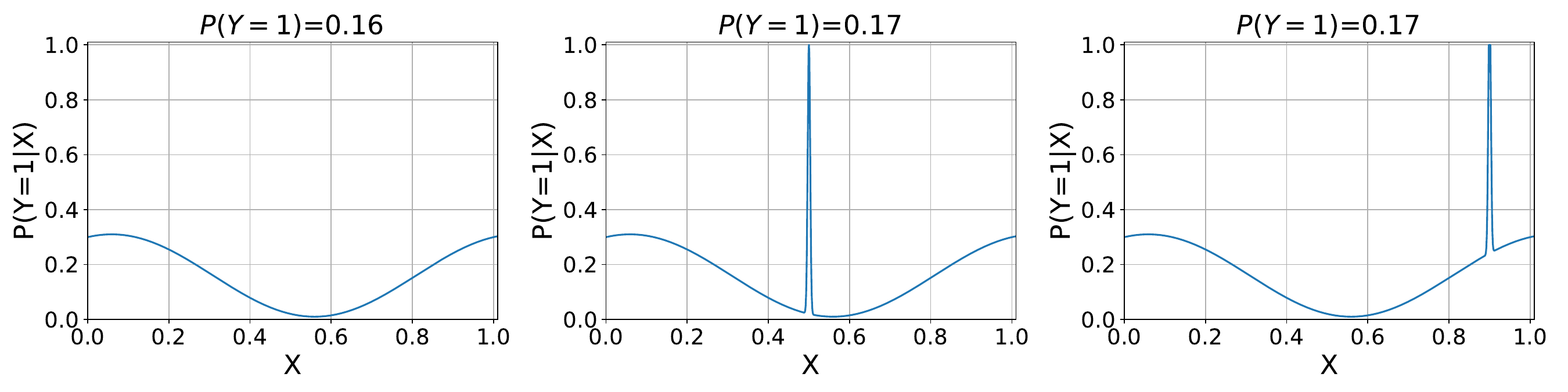}
    \caption{\textit{SinSetting} with a peak across localisations.}
    \label{fig:sinsetting-regcurves}
\end{figure}

\begin{figure*}[h]
    \centering
    \includegraphics[width=0.3\linewidth]{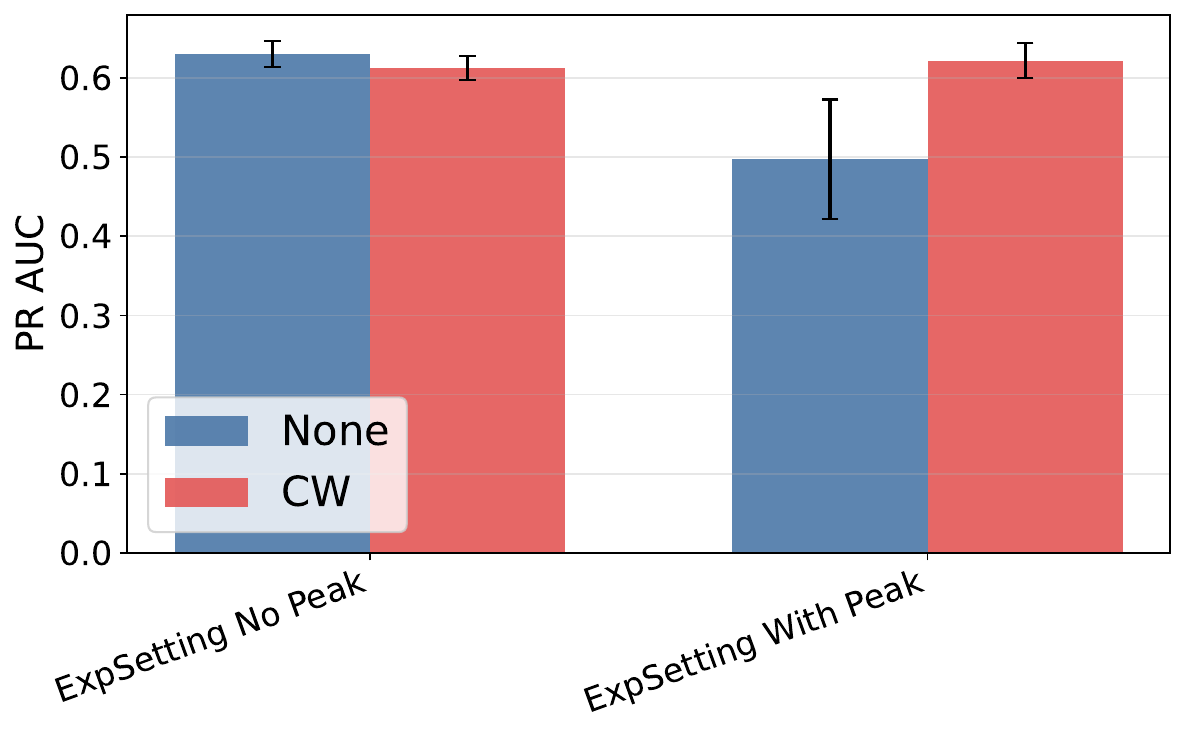}
    \includegraphics[width=0.3\linewidth]{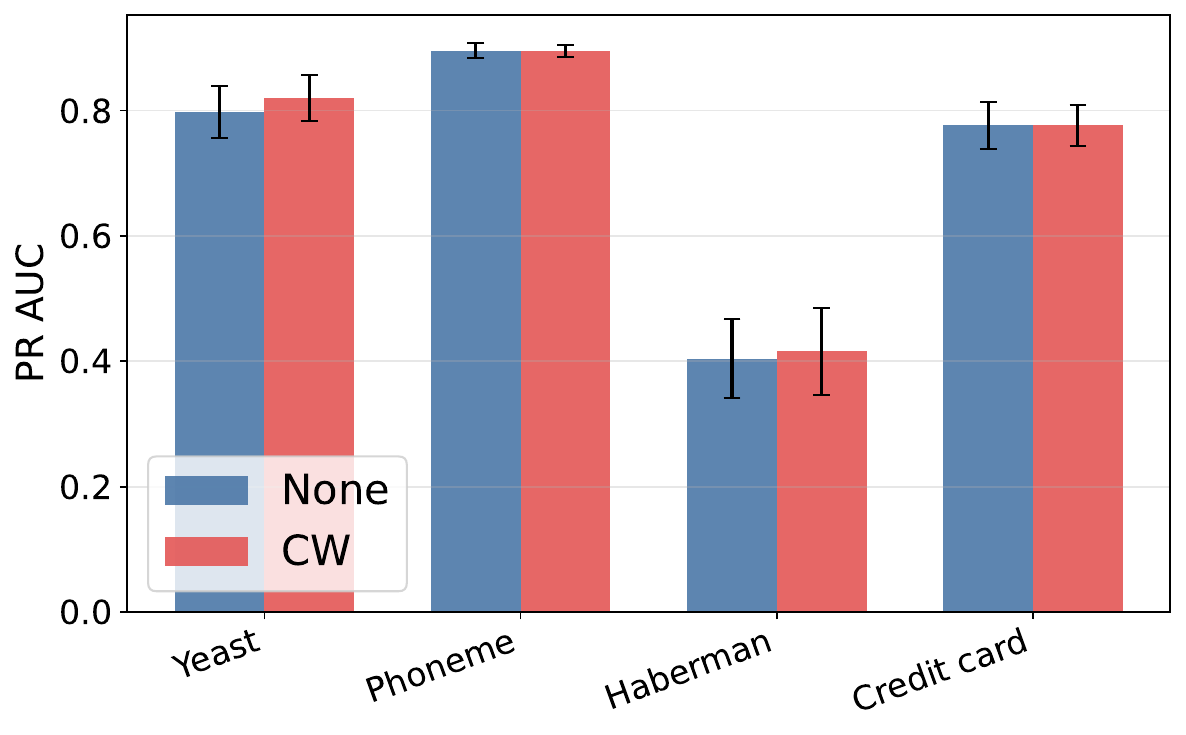}
    \includegraphics[width=0.3\linewidth]{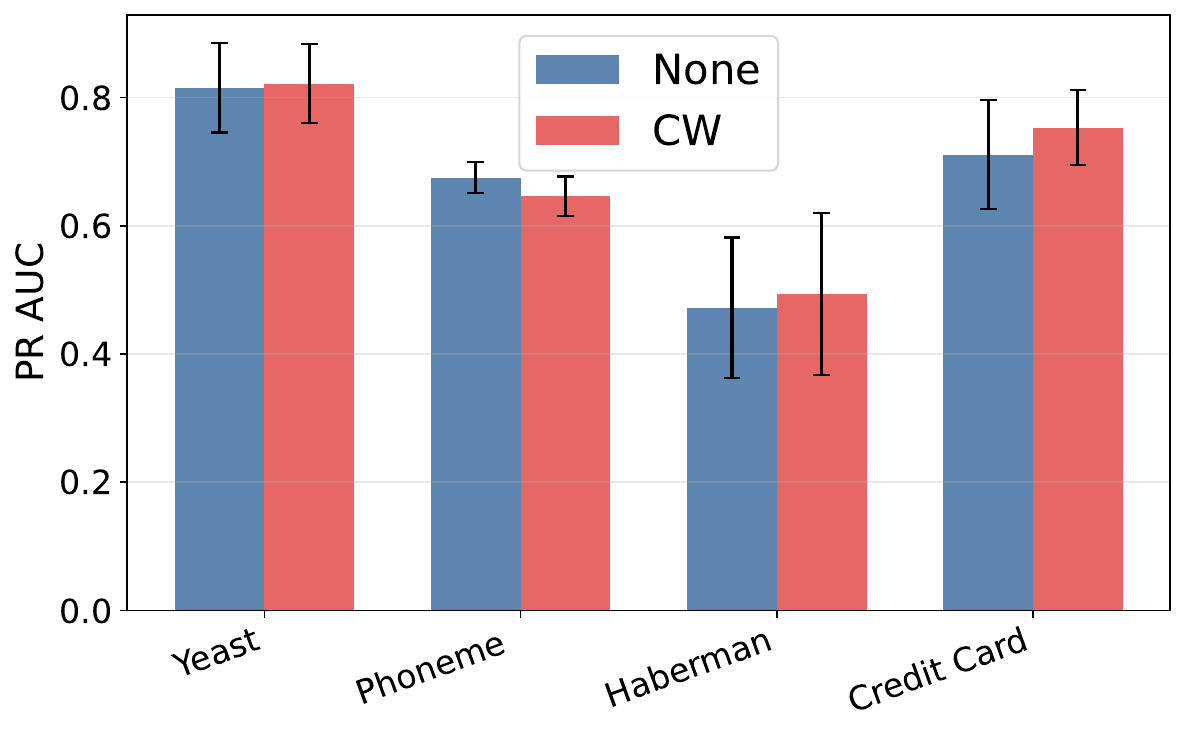}
    \caption{PR AUC comparison across model configurations with $p^\star = 0.5$: (left) decision tree with maximum depth $4$ evaluated on \textit{ExpSetting}, (center) unconstrained random forest evaluated on real-world data sets, and (right) $2$-depth-constrained random forest evaluated on real-world data sets.}
    \label{fig:res-models-expsetting}
\end{figure*}
%\es{tu évalues bien les performances sur le jeu de test ?} \as{Oui avec une $5-$CV}

%Indeed, \Cref{fig:sim-heatmap-sinsetting} in \Cref{appendix:sinsetting}  presents a heatmap analogous to \Cref{fig:sim-heatmap-expsetting}, evaluated on the \textit{SinSetting} regression function which features a peak at $b=0.6$. %%%Consistent with our earlier findings, the baseline None strategy fails to detect the peak when it is characterized by a large height and a narrow width. However, this peak is accurately estimated across a notably broader range of parameter configurations than in the \textit{ExpSetting}. %%%Thus, the results confirm that the weighting strategy successfully recovers the peak in geometric regimes where it otherwise goes undetected.
%Complete experiments on \textit{SinSetting} are provided in \Cref{appendix:sinsetting}.
%%%Consistent with the findings in the previous section, the \textit{SinSetting} configuration yields an improved PR AUC under the weighting strategy for specific peak locations.

%Consequently, the presence of sharp peaks within low-signal zones, coupled with a limited sample size or capacity to fully estimate the underlying regression function, can drive significant practical performance discrepancies between weighted and unweighted strategies. 
To assist practitioners in determining whether to incorporate weighting into their pipelines, we introduce a diagnostic indicator. %isolated, high-purity local peaks within a data set.
\begin{table}[h]
\center
\caption{Data sets.
%, where $N$ is the total number of samples, $n$ is the number of minority samples, and $d$ is the number of features.
}
\label{table:Chap2:data-sets}
    \begin{tabular}{lccr} 
         \toprule
          &  $\#$Samples & $\frac{\#Positive}{\#Samples}$ & $d$   \\ 
         \midrule
         Haberman & $306$ & $26\%$ & $3$\\

         %Pima & $768$ & $35\% $ & $8$   \\
         Yeast %(ME3vsRest) 
         & $1\,462$ & $11\%$ & $8$ \\
         %Abalone %(18vsRest)  &$4\,177$ & $1\% $ &$8$   \\
         %\shortstack{Wine %Quality \\ (vsRest)}  & $4\,974$ & $4\% $ & $11$ \\  
         Phoneme & $5\,404$ & $29\%$ & $5 $    \\
         Credit Card & $284\,315$ &$0.2\%$ & 29 \\
 \bottomrule
\end{tabular}
\end{table}
%\es{c'est voulu d'avoir un intervalle de confiance plutôt que de "colorier" la zone sous la courbe? La seconde solution n'est pas plus habituelle?} \as{Je ne m'y connais absolumebnt. J'ai fait la première car ça me semblait plus simple à implémenter. je vais me renseigne rpour la seconde option. Mais j'ai déjà vu la première option dans des articles.}

\subsection{Peak Indicator}
To diagnose the hidden structures introduced in the previous sections, we propose an indicator.
For each positive sample ($Y=1$), we have:
\begin{equation}
    \mathbf{\tilde{C}} = \frac{D^{-} - D^{+}}{D^{-} + D^{+}},
\end{equation}
where $D^{+}$ represents the mean distance from the positive instance to its $B$ nearest neighbors within the same class ($Y=1$), and $D^{-}$ represents its mean distance to its nearest neighbors within the opposite class ($Y=0$). %\es{j'aurais besoin que tu me dises ici si cette distance ressemble à celles utilisées dans les différentes version de NearMiss et si elle diffère, pourquoi est-ce qu'on pense qu'elle est meilleure} \as{Ajouter comparaisons NM1 et NM2.} 
While instance complexity metrics such as the Silhouette coefficient \citep{rousseeuw1987silhouettes} or the overlap measure \citep{lorena2019complex} rely on $D^{+}$ and $D^{-}$, $\mathbf{\tilde{C}}$ introduces a different ratio that restricts distance evaluations to a local neighborhood. %This provides a smooth, instance-level indicator of localized density relative to class boundaries.
Furthermore, our formulation differs from established distance-based metrics in the literature, such as those used in NearMiss undersampling strategies \citep{mani2003knn}. Specifically, NearMiss-1 measures the mean distance from majority class samples to their nearest minority class neighbors, whereas NearMiss-2 evaluates the distance to the farthest minority class neighbors. Neither metric accounts for the ratio of intra-class to inter-class distances like our indicator.

The distribution of $\mathbf{\tilde{C}}$ across all positive samples is subsequently visualized using a histogram.
A high concentration of samples near $1$ in the peak indicator histogram signifies that a large portion of the minority class instances are highly clustered together yet spatially isolated from the majority class, indicating the presence of localized peaks. Thus, such data sets are highly likely to require a weighting procedure to achieve better performance if the downstream model cannot perfectly estimate the regression function.

\begin{figure}[h]
    \centering
    \begin{subfigure}[b]{0.23\textwidth}
        \centering
        \includegraphics[width=1.0\linewidth]{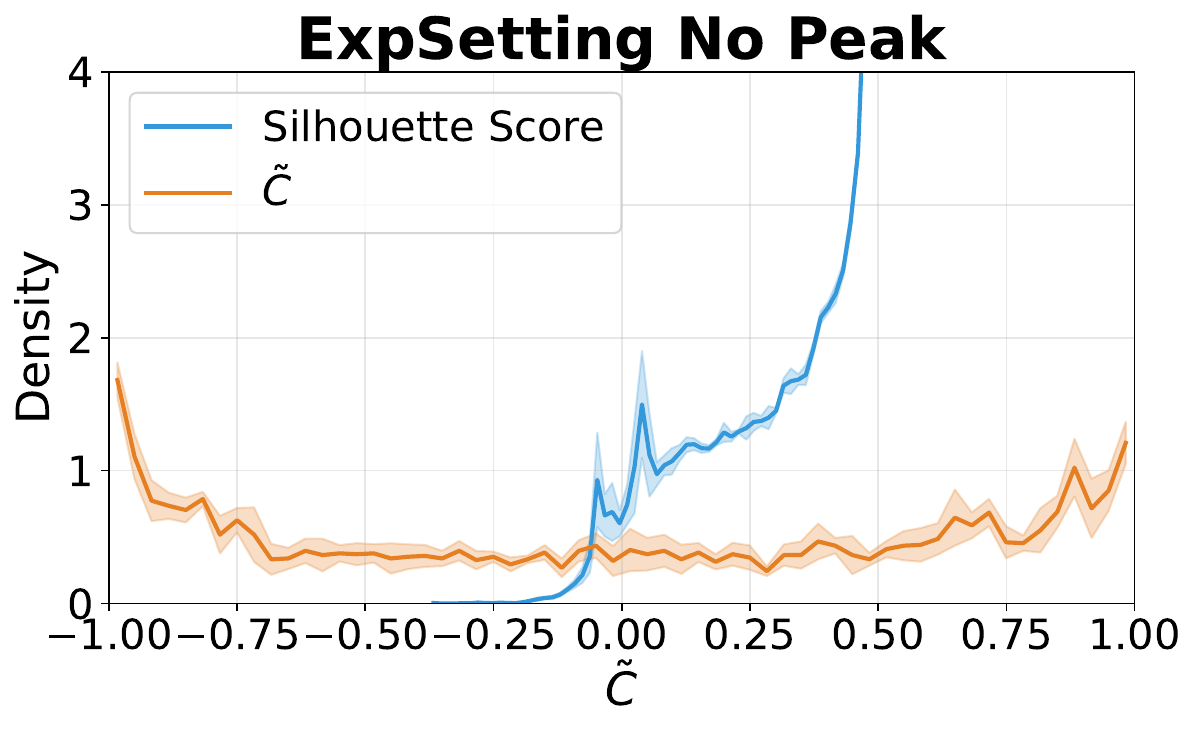}
        \caption{\textit{ExpSetting}}
        \label{fig:peak-k1-nopeak}
    \end{subfigure}\hfill
    \begin{subfigure}[b]{0.23\textwidth}
        \centering
        \includegraphics[width=1.0\linewidth]{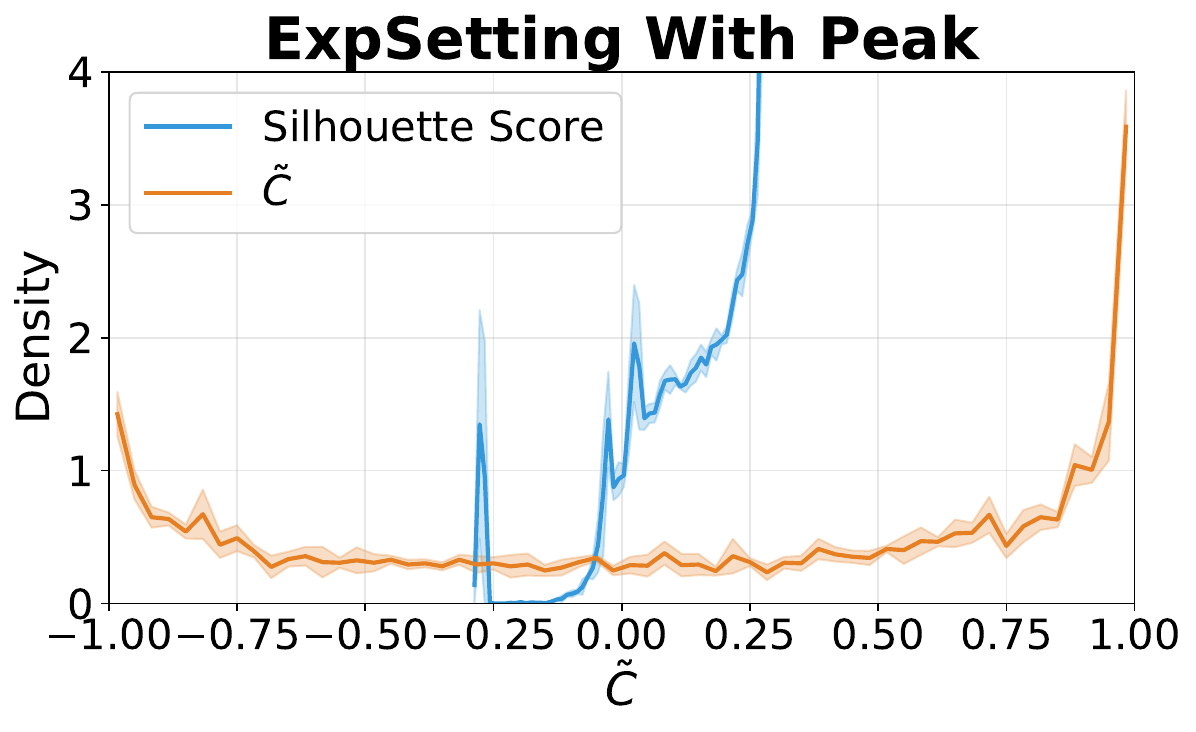}
        \caption{\textit{ExpSetting} with peak }
        \label{fig:peak-k1-withpeak}
    \end{subfigure}\hfill
    \begin{subfigure}[b]{0.23\textwidth}
        \centering
        \includegraphics[width=1.0\linewidth]{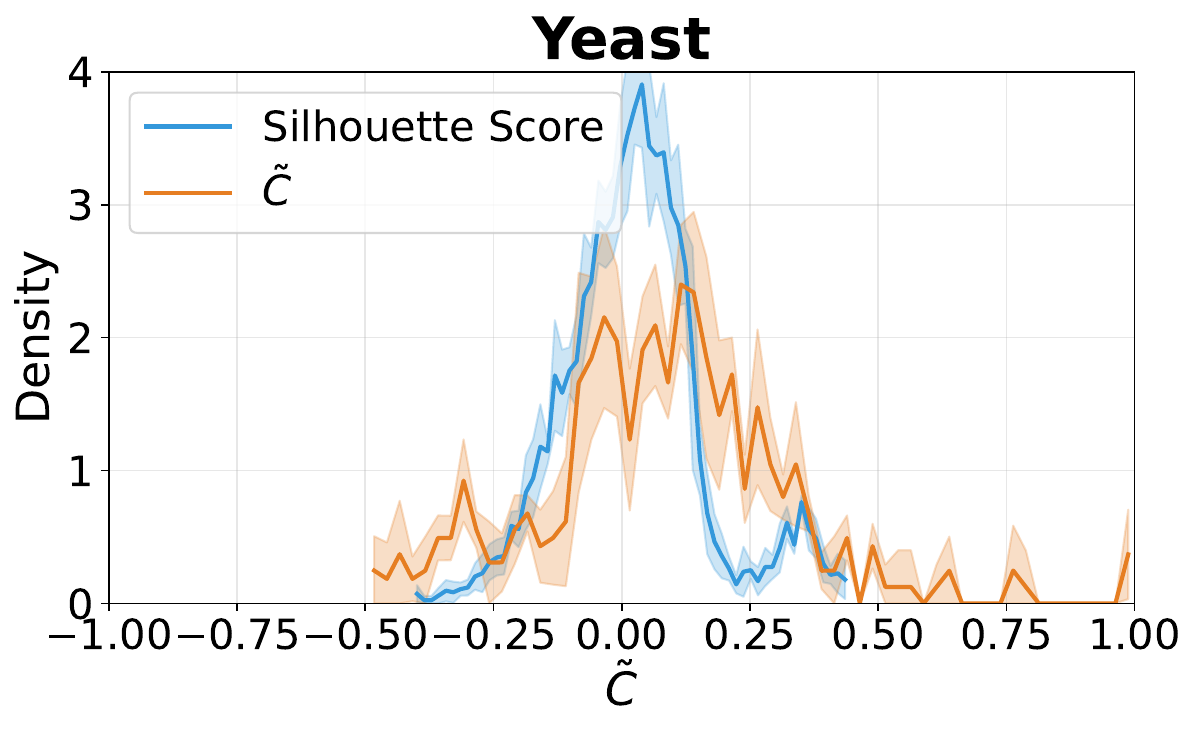}
        \caption{Yeast}
        \label{fig:peak-k10-nopeak}
    \end{subfigure}\hfill
    \begin{subfigure}[b]{0.23\textwidth}
        \centering
        \includegraphics[width=1.0\linewidth]{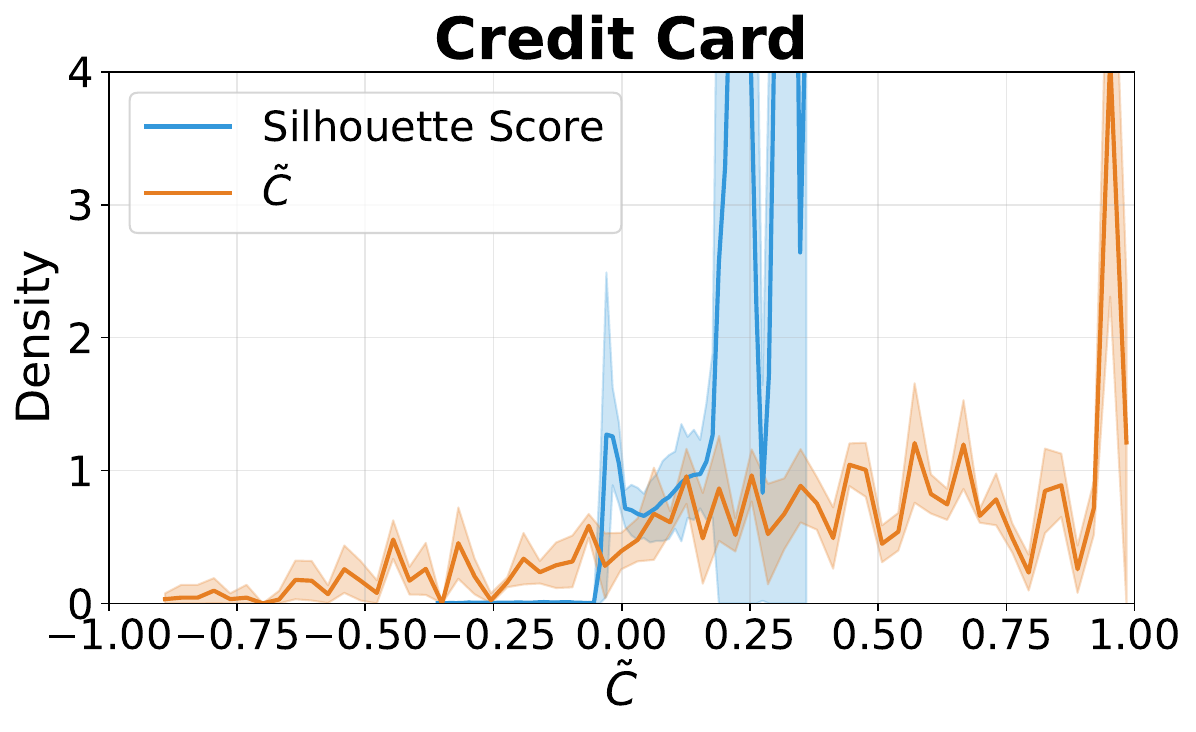}
        \caption{CreditCard}
        \label{fig:peak-k10-withpeak}
    \end{subfigure}

    %\caption{Peak indicator for the \textit{ExpSetting} configuration, Yeast and CreditCard for $B=1$.}
    \caption{Peak indicator ($B=1$) and Silouhette samples scores for the \textit{ExpSetting}, Yeast and CreditCard.}
    \label{fig:peak-indicator-expsetting}
\end{figure}

\Cref{fig:peak-indicator-expsetting} (alongside \Cref{appendix:fig:peak-indicator-expsetting} in the Appendix) illustrates the peak indicator distributions for the \textit{ExpSetting} data set under both peaked ($b=0.5$) and non-peaked configurations, while \Cref{fig:res-models-expsetting} details the corresponding predictive performances. We observe that introducing a sharp peak alters the distribution profile, shifting probability mass toward $\mathbf{\tilde{C}} \approx 1$ across both $B=1$ and $B=10$. For the peaked \textit{ExpSetting} regime, the weighting strategy yields predictive gains. This distributional shift underscores the capacity of $\mathbf{\tilde{C}}$ to effectively identify problem regimes that stand to benefit from a weighting-based strategy. In contrast, standard Silhouette sample scores fail to detect the presence or absence of a localized peak.

Then, we employ $4$ tabular data sets from \cite{Dua:2019},    \cite{alinat1993periodic} (Phoneme) and  \cite{dal2015calibrating} (Credit Card).
The characteristics of the data sets are summarized in \Cref{table:Chap2:data-sets}. Their corresponding peak indicator histograms are displayed in \Cref{fig:peak-indicator-expsetting} and \Cref{fig:peak-indicator-real-dataset} (see \Cref{appendix:supp_num_exp}), while their downstream predictive performances are detailed in \Cref{fig:res-models-expsetting}. Consistent with our previous findings, we observe a strong empirical alignment: the data sets where class-weighting successfully improves the PR AUC are precisely those whose peak indicator histograms exhibit a concentration of samples near $1$, such as Credit Card data set.
Notably, for unconstrained (fully developed) random forests (\Cref{fig:res-models-expsetting} center), predictive performance with or without weighting becomes negligible. This behavior aligns with \Cref{lemma:main-mustar}, as fully grown trees may accurately approximate the true regression function. Conversely, for Credit Card data set, when maximum depth is constrained, we observe that the weighting-based strategy yields better predictive performance, while our peak indicator indicates that the data set may contain localized peaks.
Besides, increasing to $B=10$ effectively filters out noise. Indeed, the Haberman data set displays a minor peak near $\tilde{C}=1$ under $B=1$ that disappears at $B=10$, explaining the absence of gains from weighting.

% \begin{table}[h]
%     \centering
%     \begin{tabular}{lcccr}
%     \toprule
%          Data set & unweighted & Weighting &Weighting\\
%           &  &  ($p^{\star}=0.5$) &  ($p^{\star}=0.9$)\\
%           \midrule
%          \textit{ExpSetting} No peak & \textbf{0.569} &0.554 & 0.550\\
%          %ExpSetting With peak & 0.552 & \textbf{0.645}\\
%          ExpSetting With peak & 0.588 & \textbf{0.607} &0.549\\
%          \midrule
%          Haberman & 0.388 & 0.389 &0.389\\
%          Yeast & \textbf{0.844} & 0.821 &  0.799\\
%          Phoneme & 0.868& \textbf{0.871} &0.841\\
%     \bottomrule    
%     \end{tabular}
%     \caption{CART tree PR AUC on \textit{ExpSetting} and random forest PR AUC for real-world data sets.}
%     \label{tab:res-models}
% \end{table}

%\clearpage
\section{CONCLUSION}

In this paper, we provide insights for weighting-based rebalancing strategies. First, we theoretically study the estimation of a linear regression function through weighted and unweighted CART trees with limited depths. In this setting, we prove that the points in the PR space for both strategies are distributed along the same curve, thus illustrating the similar non-asymptotic performance of the two strategies.  
%s
Beyond the linear case, under fixed model complexity constraints, standard unweighted classifiers suffer from a negative behavior that may underestimate sharp minority peaks. Conversely, weighting strategies capture these structures, improving their predictive performance.
%
%To transition these insights into actionable engineering workflows, we introduced an indicator, $\tilde{C}$. By quantifying instance-wise spatial isolation via nearest-neighbor, this tool serves as a lightweight diagnostic filter. It accurately signals whether a given data set requires a weighting strategy. 
%before a practitioner spends computational budget on extensive grid-search optimization pipelines. 
%
Then, we introduce a peak indicator $\tilde{C}$ that aims to determine whether a given data set benefits from a weighting strategy if the regression function is not well estimated.

%Thus, our findings redefine how class weighting should be viewed: class weighting does not change what a model can inherently learn, it simply provides a way to tune a capacity-limited model.
%
%A future work is to generalize our exact PR curve derivations to arbitrary regression functions and broader classifiers. 

\clearpage
\subsection*{AI use statement}

%(This section is \textbf{required} and does not count toward the page limit.)

In this work, we used generative AI tools for rephrasing and text formatting.
Generative AI tools were used solely to assist with the algebraic factorization of $f_{X^{\star}}$ in \Cref{lemma:IG_weighted_optimal_treshold}, and were not employed for any other theoretical derivations.
Additionally, we used generative AI tools for the preparation of supporting materials through Python code.
We have reviewed line by line all AI-assisted work, including the formatting code.
We take responsibility for the final content of this work, including text, claims, or artifacts produced with the aid of generative AI.

%See the AISTATS 2027 AI Policy for Authors for more details. This statement should not be more than 1 page.

% \subsubsection*{Acknowledgements}
% All acknowledgements go at the end of the paper, including thanks to reviewers who gave useful comments, to colleagues who contributed to the ideas, and to funding agencies and corporate sponsors that provided financial support.
% To preserve the anonymity, please include acknowledgements \emph{only} in the camera-ready papers. The acknowledgements do not count against the 9-page page limit in the camera-ready.

%\clearpage
\bibliography{refs}
\bibliographystyle{plainnat}

%%%%%%%%%%%%%%%%%%%%%%%%%%%%%%%%%%%%%%%%%%%%%%%%%%%%%%%%%%%%
%\clearpage

% %%% BEGIN INSTRUCTIONS %%%
%The checklist follows the references. For each question, choose your answer from the three possible options: Yes, No, Not Applicable.  You are encouraged to include a justification to your answer, either by referencing the appropriate section of your paper or providing a brief inline description (1-2 sentences).
%Please do not modify the questions.  Note that the Checklist section does not count towards the page limit. Not including the checklist in the first submission won't result in desk rejection, although in such case we will ask you to upload it during the author response period and include it in camera ready (if accepted).

%\textbf{In your paper, please delete this instructions block and only keep the Checklist section heading above along with the questions/answers below.}
% %%% END INSTRUCTIONS %%%

\clearpage

\thispagestyle{empty}

% Supplementary material: To improve readability, you must use a single-column format for the supplementary material.
\onecolumn
\doparttoc
\faketableofcontents
\part{}
\appendix

\aistatstitle{Supplementary Materials}

\vspace{1cm}
\parttoc % Insert the appendix TOC
\newpage

%\addcontentsline{toc}{chapter}{Appendix}

\section{Supplementary numerical experiments}
\label{appendix:supp_num_exp}

\subsection{General additional numerical experiments}
For all experiments, we use an AMD
Ryzen Threadripper PRO 5955WX (16 cores, 4GHz) with 256GB RAM.

\begin{figure}[h]
    \centering
    \includegraphics[width=0.8\linewidth]{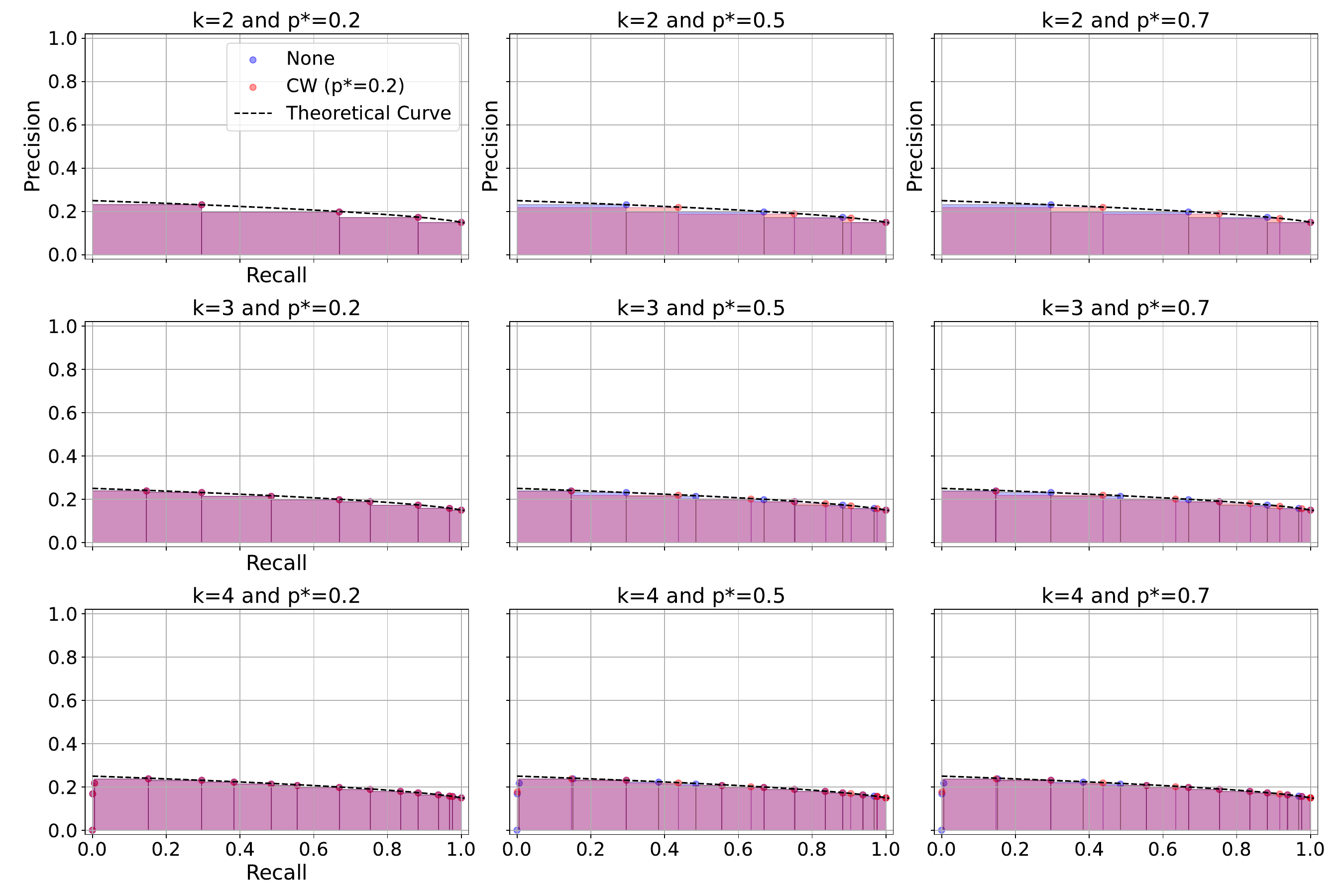}
    \caption{PR curves of  None and CW for a linear regression function with $\alpha=0.2$ and $\beta=0.05$.}
    \label{fig:prcurves-linear-all}
\end{figure}

%Here, the peak indicators results of various real-world data sets are shown with $B=1$ and $B=10$.

\begin{figure}[h]
    \centering
    \begin{subfigure}[b]{0.23\textwidth}
        \centering
        \includegraphics[width=1.0\linewidth]{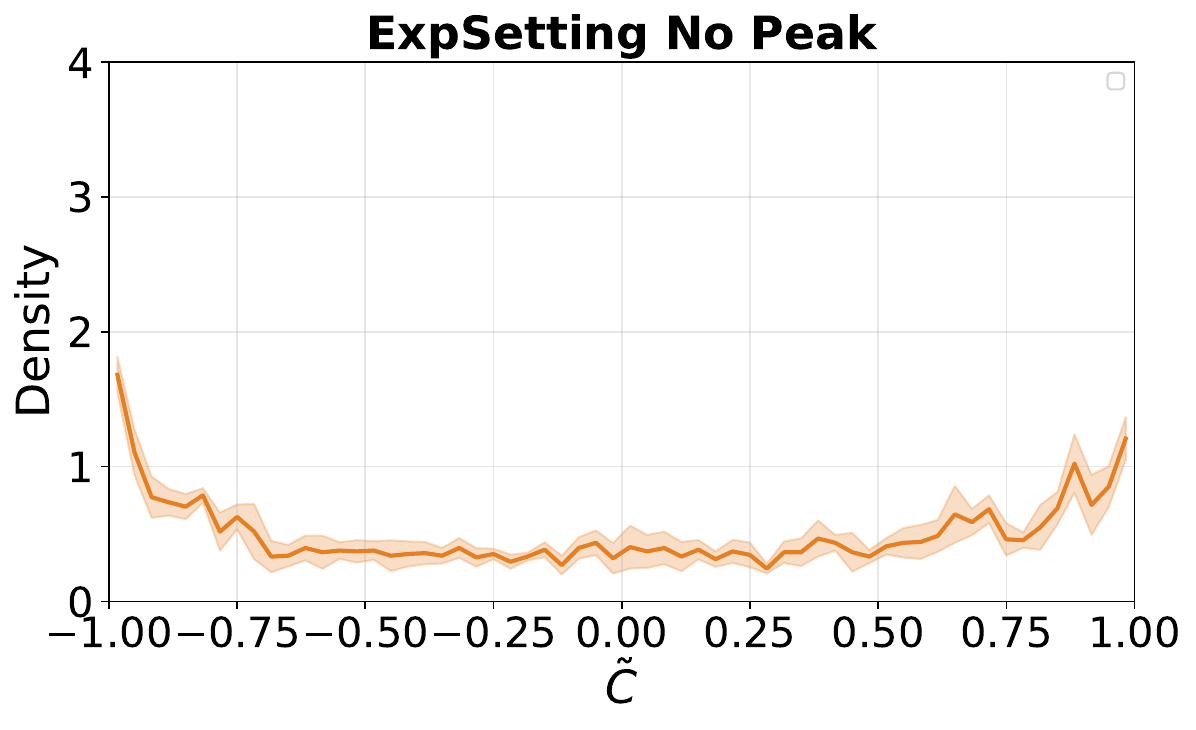}
        \caption{No peak ($B=1$)}
        \label{appendix:fig:peak-k1-nopeak}
    \end{subfigure}\hfill
    \begin{subfigure}[b]{0.23\textwidth}
        \centering
        \includegraphics[width=1.0\linewidth]{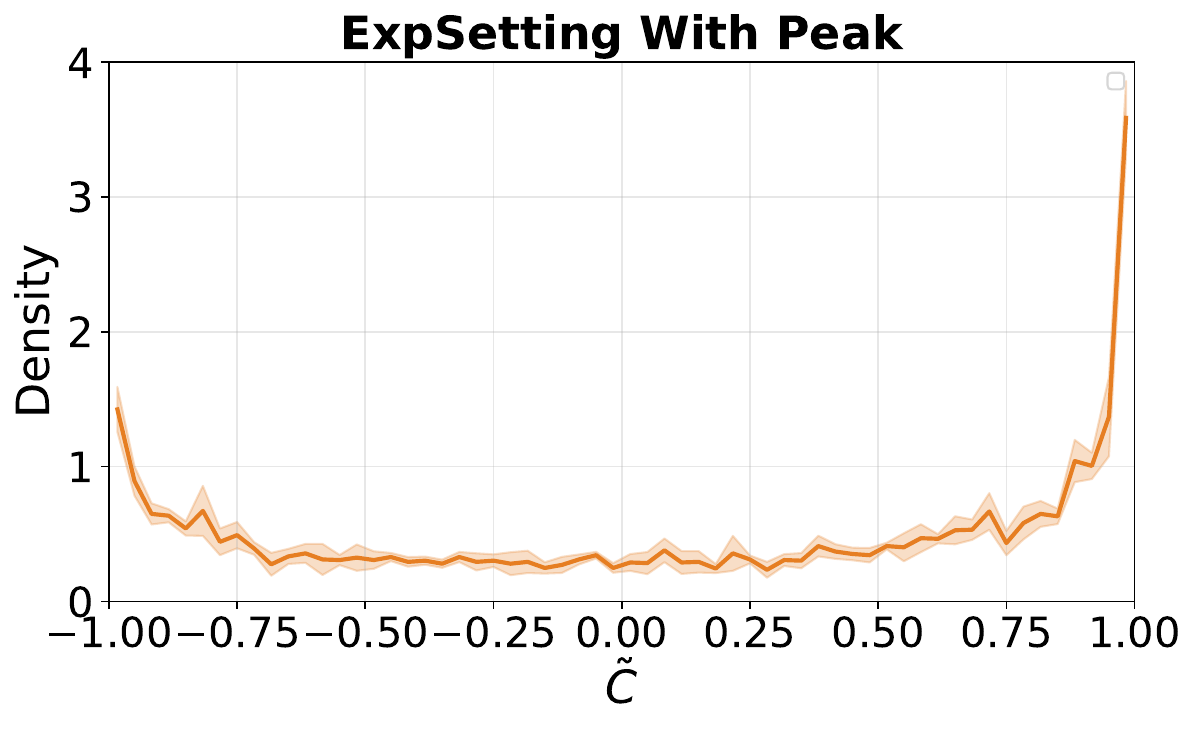}
        \caption{With peak ($B=1$)}
        \label{appendix:fig:peak-k1-withpeak}
    \end{subfigure}\hfill
    \begin{subfigure}[b]{0.23\textwidth}
        \centering
        \includegraphics[width=1.0\linewidth]{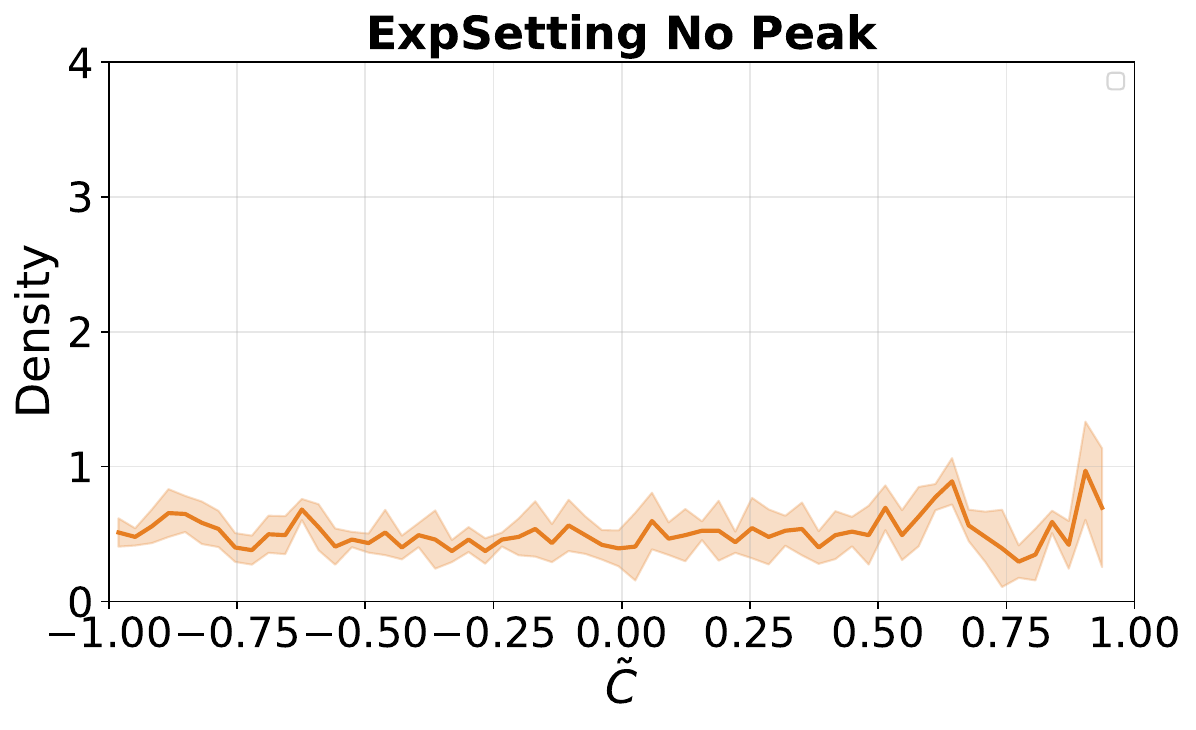}
        \caption{No peak ($B=10$)}
        \label{appendix:fig:peak-k10-nopeak}
    \end{subfigure}\hfill
    \begin{subfigure}[b]{0.23\textwidth}
        \centering
        \includegraphics[width=1.0\linewidth]{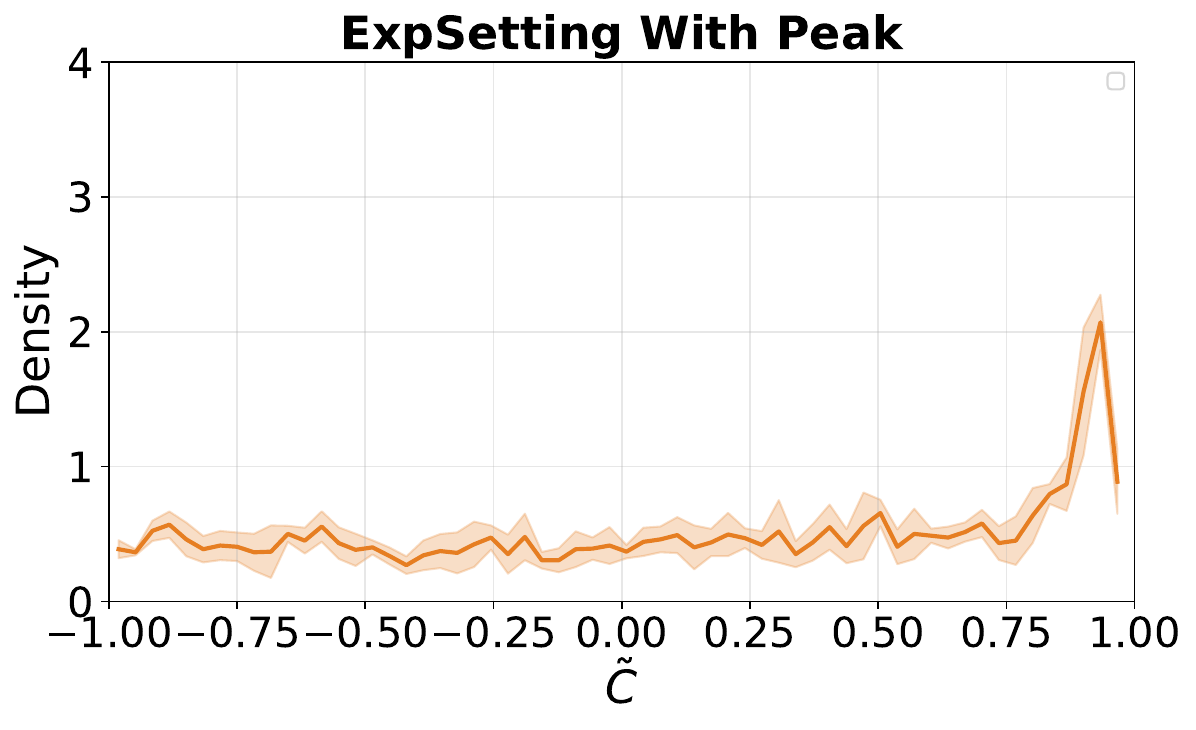}
        \caption{With peak ($B=10$)}
        \label{appendix:fig:peak-k10-withpeak}
    \end{subfigure}

    \caption{Peak indicator for the \textit{ExpSetting} configuration under varying peak settings and $K$ parameters.}
    \label{appendix:fig:peak-indicator-expsetting}
\end{figure}

\begin{figure}[h]
    \centering
    % Top line: K = 1
    \begin{subfigure}[b]{\textwidth}
        \centering
        \includegraphics[width=0.23\linewidth]{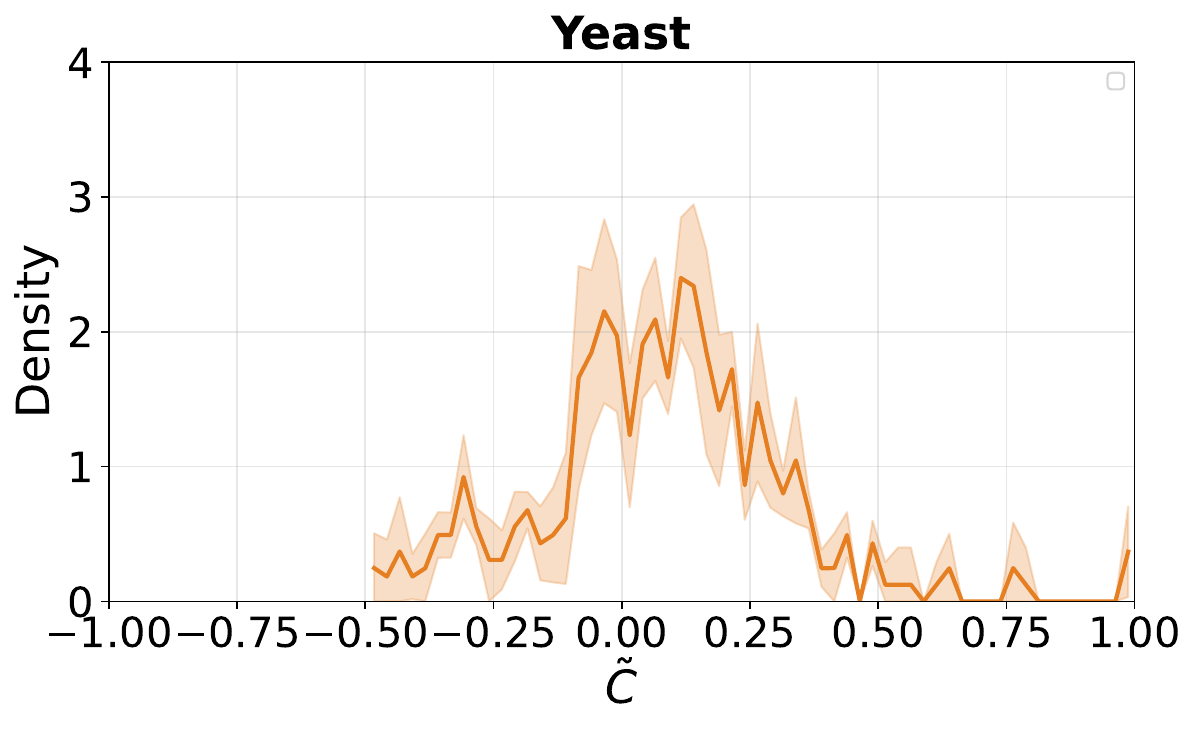}\hfill
        \includegraphics[width=0.23\linewidth]{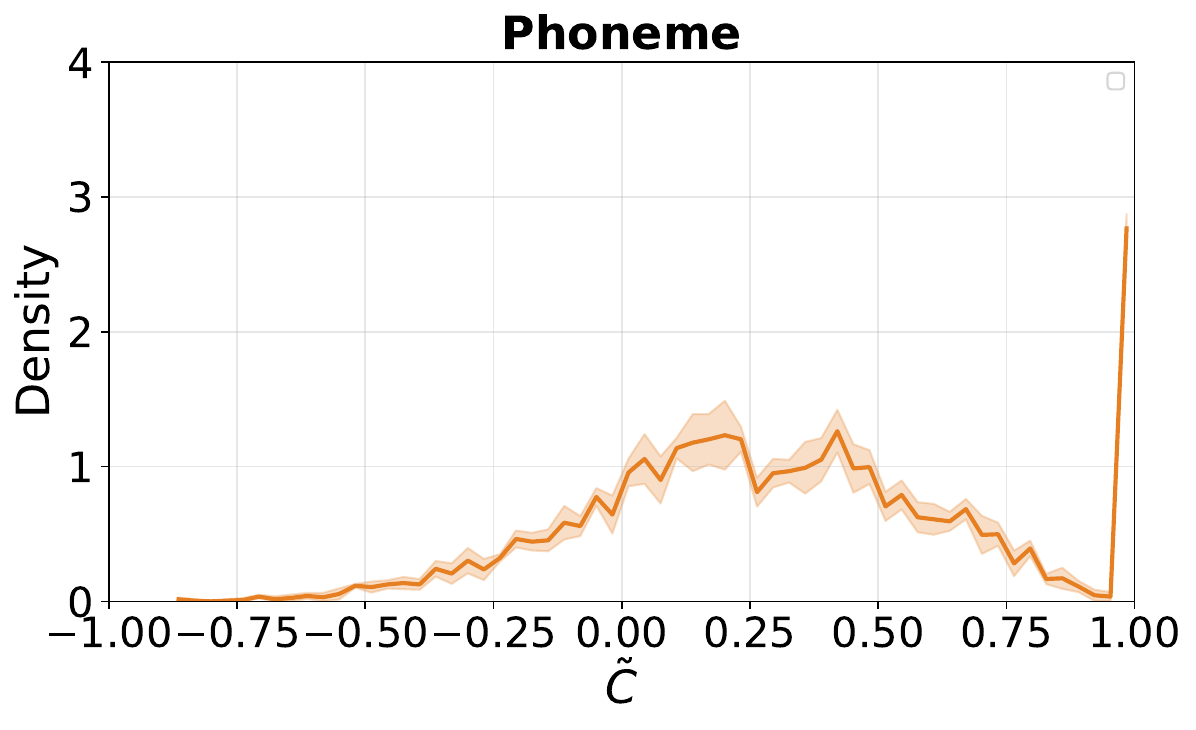}\hfill
        \includegraphics[width=0.23\linewidth]{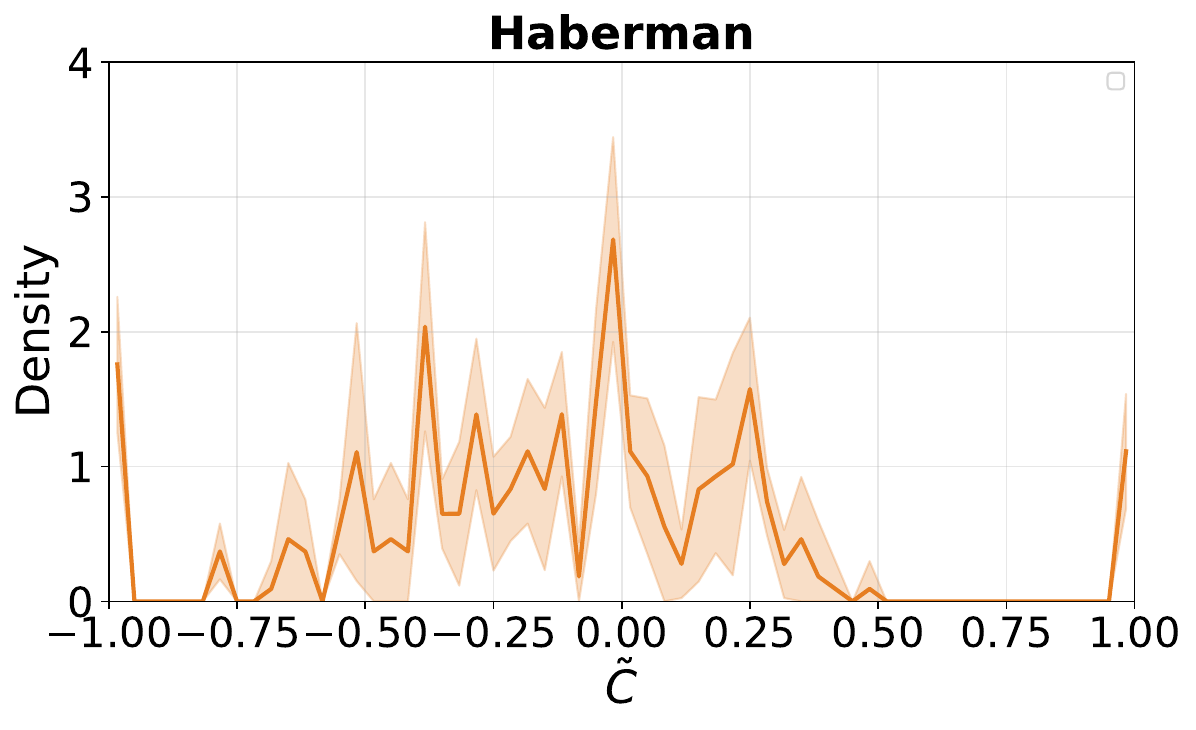}\hfill
        \includegraphics[width=0.23\linewidth]{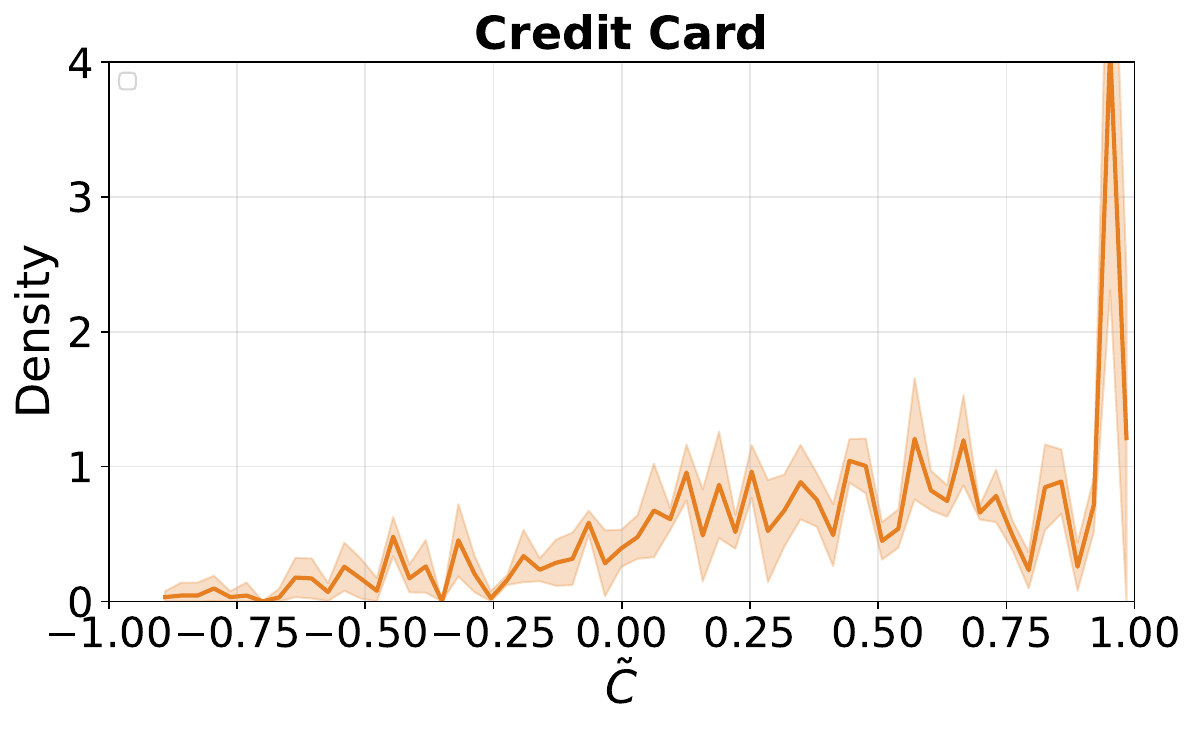}
        \caption{$B=1$.}
        \label{fig:peak-indicator-k1}
    \end{subfigure}

    \vspace{0.8em} % Vertical spacing between top and bottom lines

    % Bottom line: K = 10
    \begin{subfigure}[b]{\textwidth}
        \centering
        \includegraphics[width=0.23\linewidth]{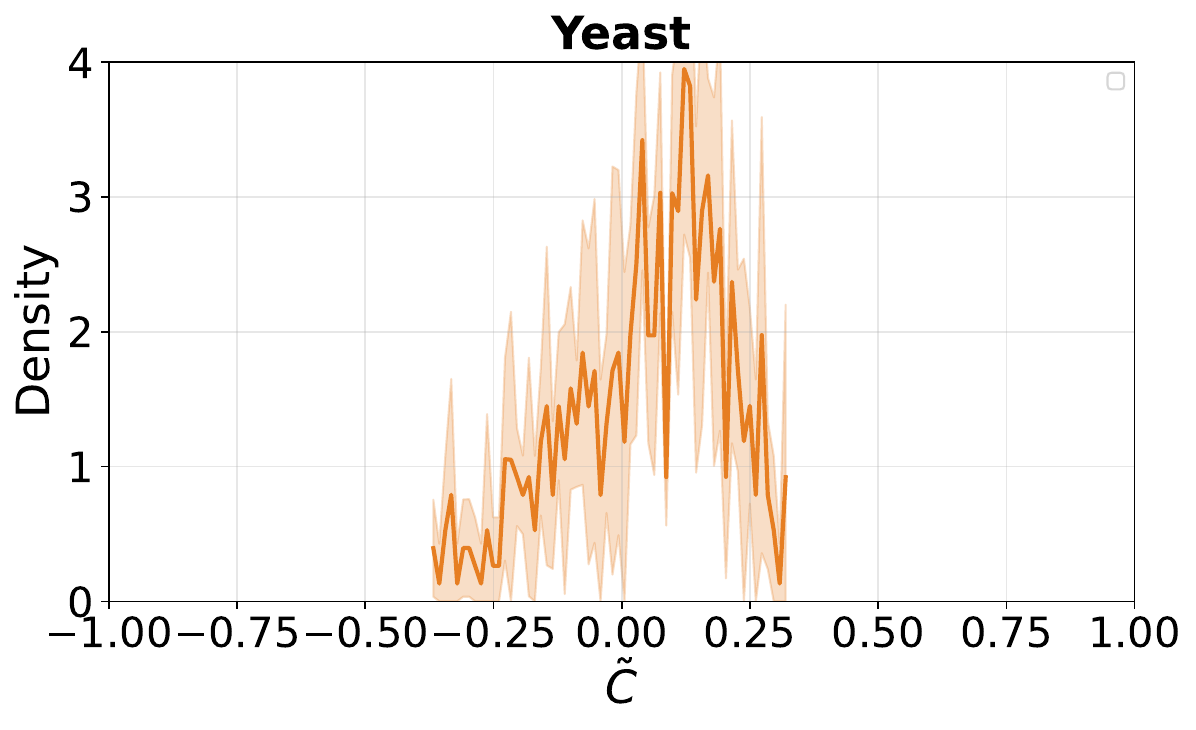}\hfill
        \includegraphics[width=0.23\linewidth]{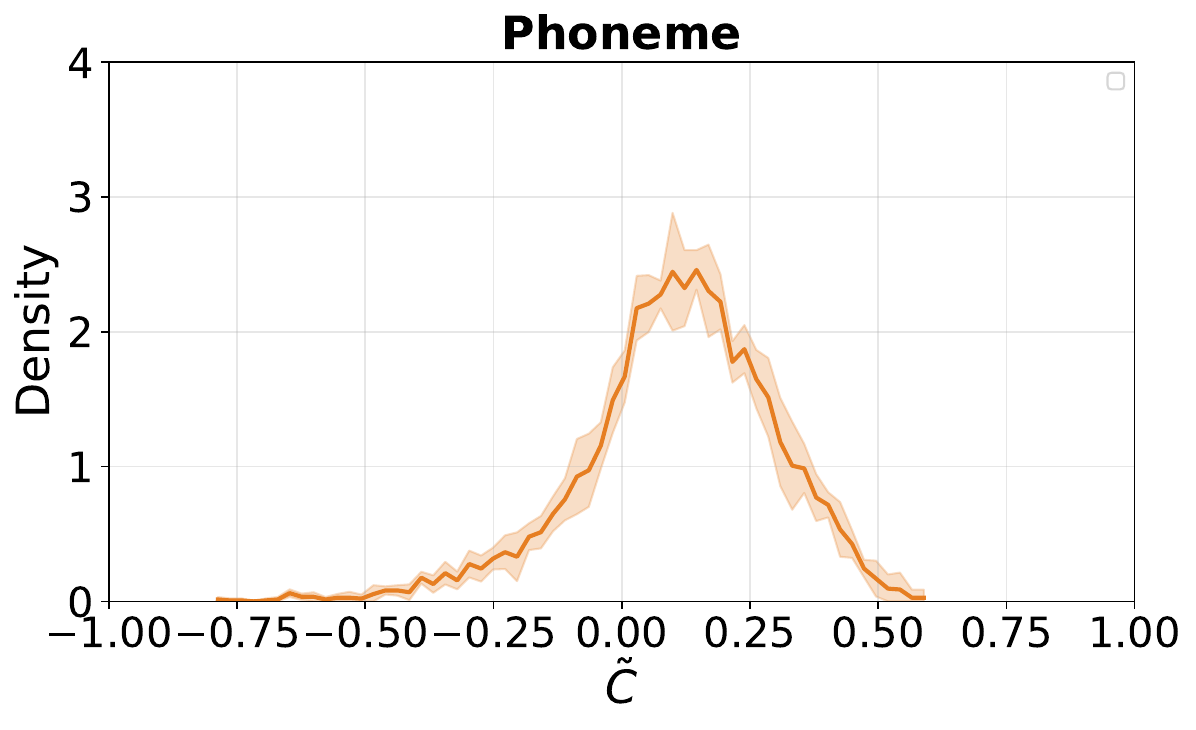}\hfill
        \includegraphics[width=0.23\linewidth]{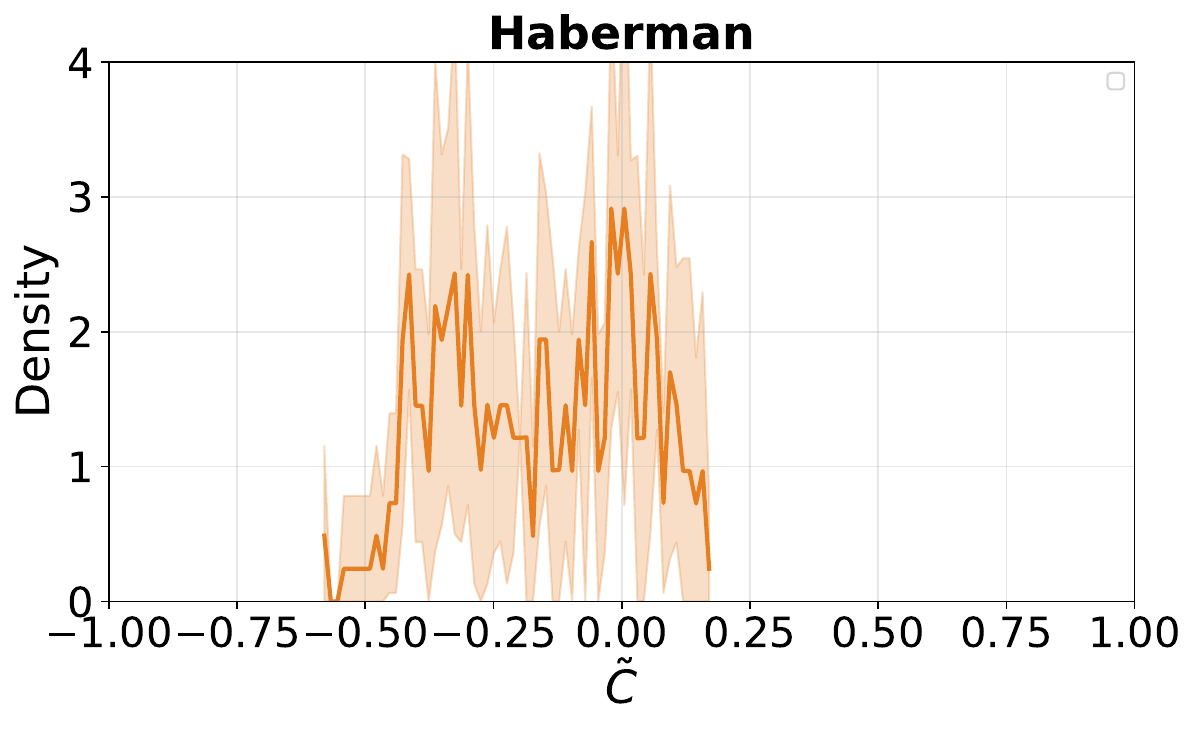}\hfill
        \includegraphics[width=0.23\linewidth]{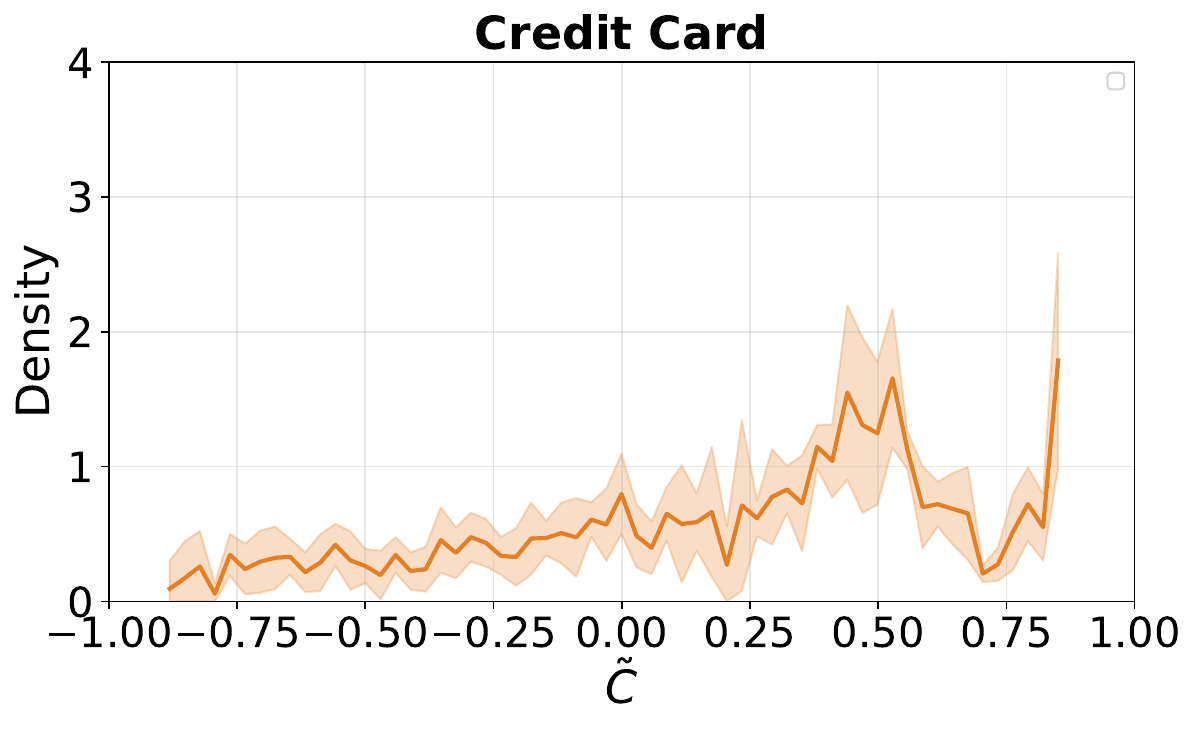}
        \caption{$B=10$.}
        \label{fig:peak-indicator-k10}
    \end{subfigure}

    \caption{Peak indicator on real-world data sets for $B=1$ (top) and $B=10$ (bottom).}
    \label{fig:peak-indicator-real-dataset}
\end{figure}

%\clearpage
\subsection{\textit{ExpSetting}: additional results}

\begin{figure}[h]
    \centering
    \includegraphics[width=1.0\linewidth]{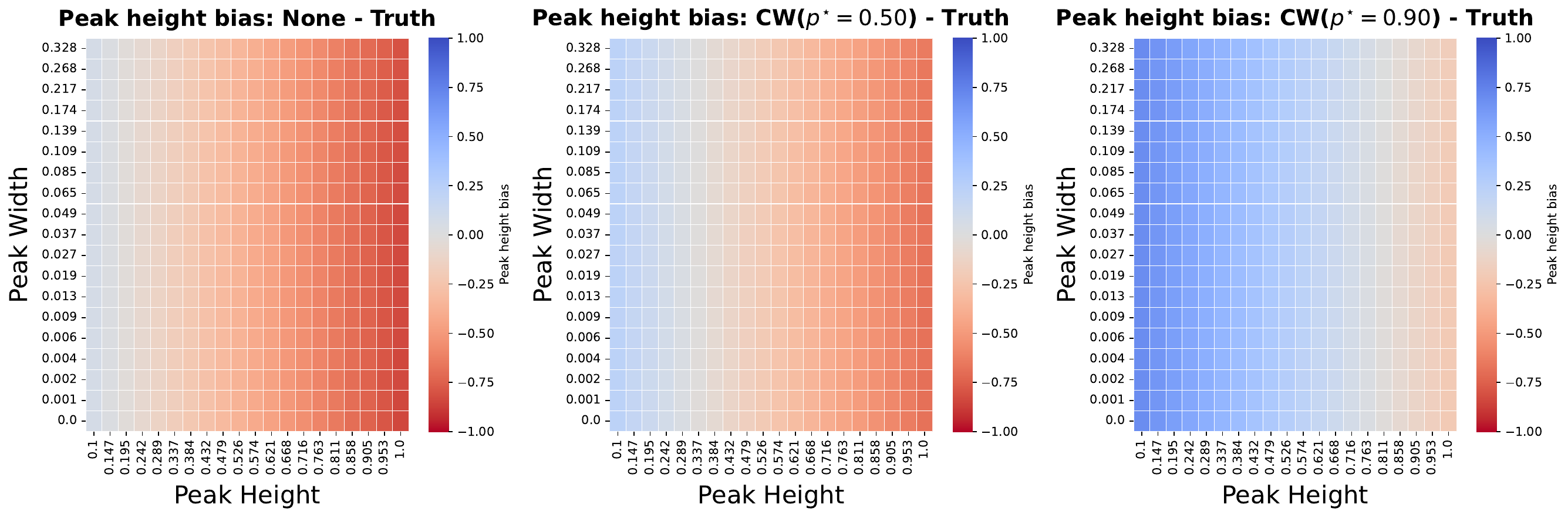}
    \caption{\textit{ExpSetting}: CART tree difference between predicted and ground-truth peak probability for input \textbf{dimension of 5} ($k= 4$). %Panels show the None alongside weighting with $p^{\star} = 0.5$ and $p^{\star} =0.9$.
    }
    \label{fig:sim-heatmap-expsetting-dim5}
\end{figure}

\begin{figure}[h]
    \centering
    \includegraphics[width=0.6\linewidth]{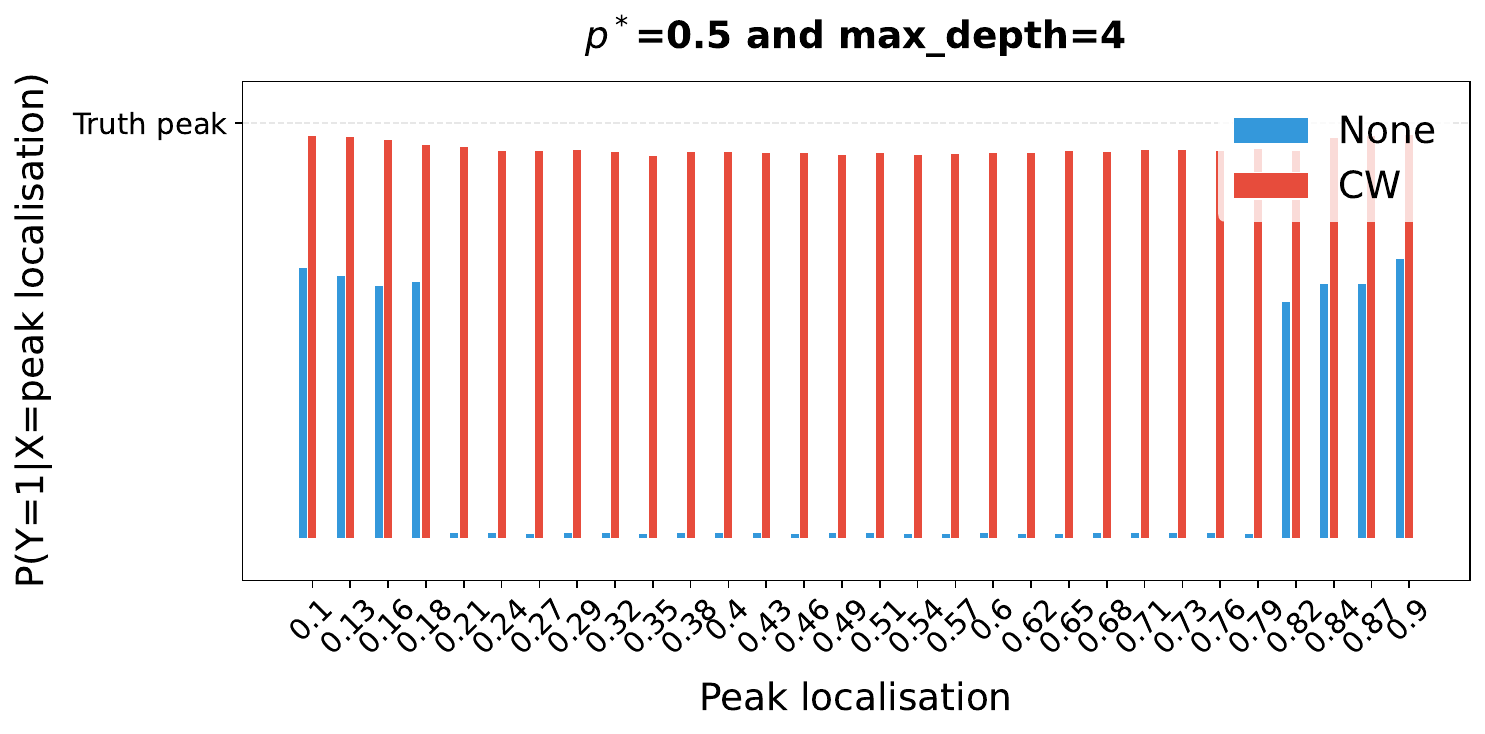}
    \caption{CART tree probability of a \textit{ExpSetting} regression function with respect to peak localisation.}
    \label{fig:sim-barplot-scatter-peak-pos-expsetting}
\end{figure}
\Cref{fig:sim-barplot-scatter-peak-pos-expsetting} evaluates a regression setup where a sharp peak (height $a=1.0$ and width $c=0.003$) %\es{il faudrait aussi donner son expression mathématiques, éventuellement avec sa hauteur et largeur en annexe} 
is introduced at various positions for both settings. The panel illustrates the output probabilities of both models as a function of the peak position. We observe that the weighted model consistently detects the peak regardless of its location within $[0,1]$. Conversely, the unweighted model only manages to detect two peaks located at the boundaries of the support.

In the following, we focus on the \texttt{ExpSetting} regression function, evaluating both predictive performance and the estimated probabilities produced by a CART decision tree. Each row corresponds to a maximum tree depth from $1$ to $4$, while each column represents a target class prevalence $p^{\star} \in \{0.1, 0.2, \dots, 0.9\}$.  

\begin{figure}[h]
    \centering
    \includegraphics[width=0.8\linewidth]{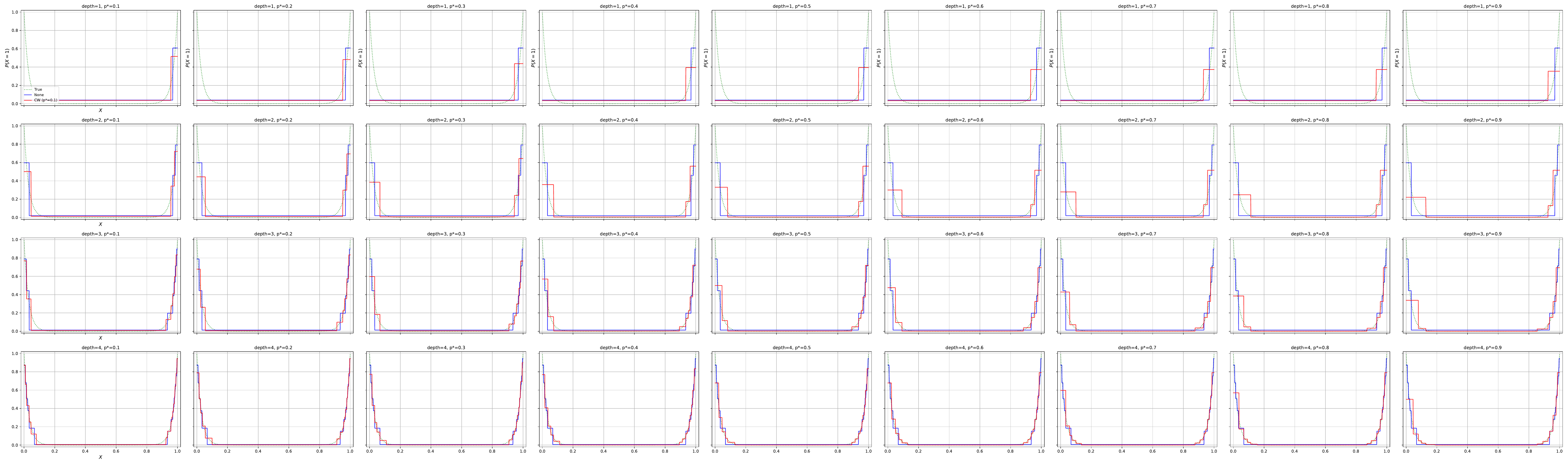}
    \includegraphics[width=0.8\linewidth]{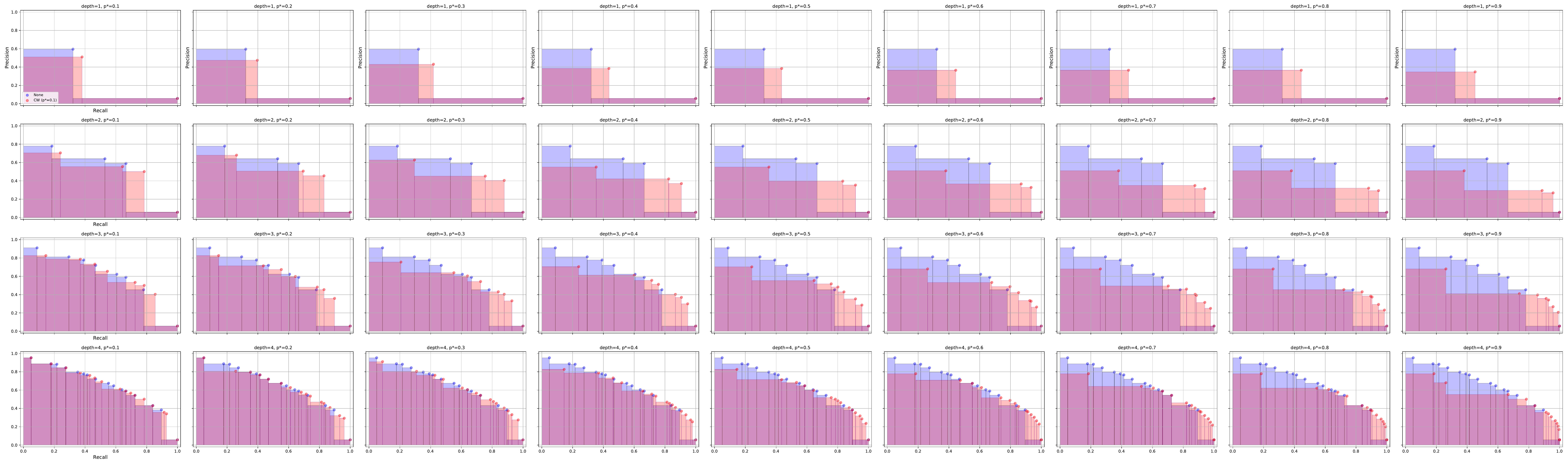}
    \includegraphics[width=0.8\linewidth]{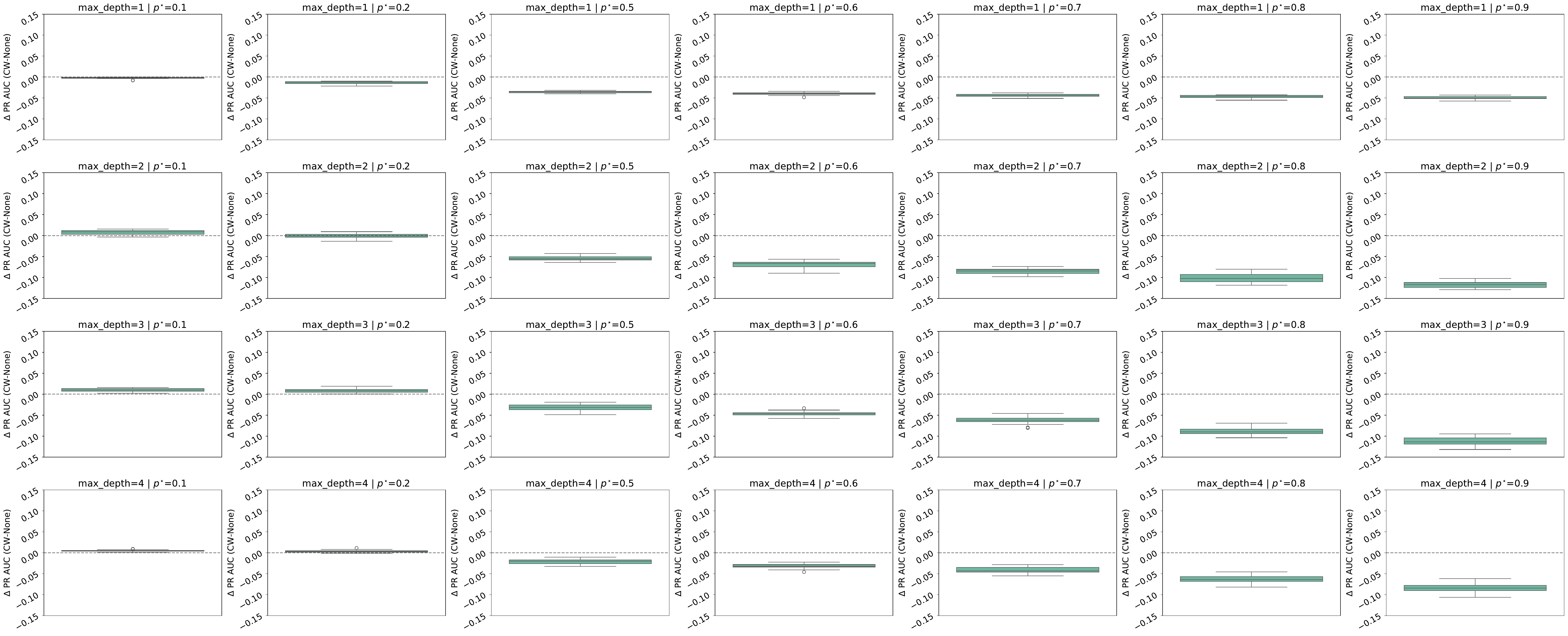}
    \caption{Estimated regression curves, PR curves, and PR AUC boxplots for \textit{ExpSetting} in the absence of a peak. For the regression plots, the inverse transformation from \Cref{lemma:main-mustar} is applied to the class-weighted strategy's predictions to ensure direct comparability with the unweighted baseline.}
    \label{fig:pr_reg_curves_no_peak}
\end{figure}

\begin{figure}[h]
    \centering
    \includegraphics[width=0.8\linewidth]{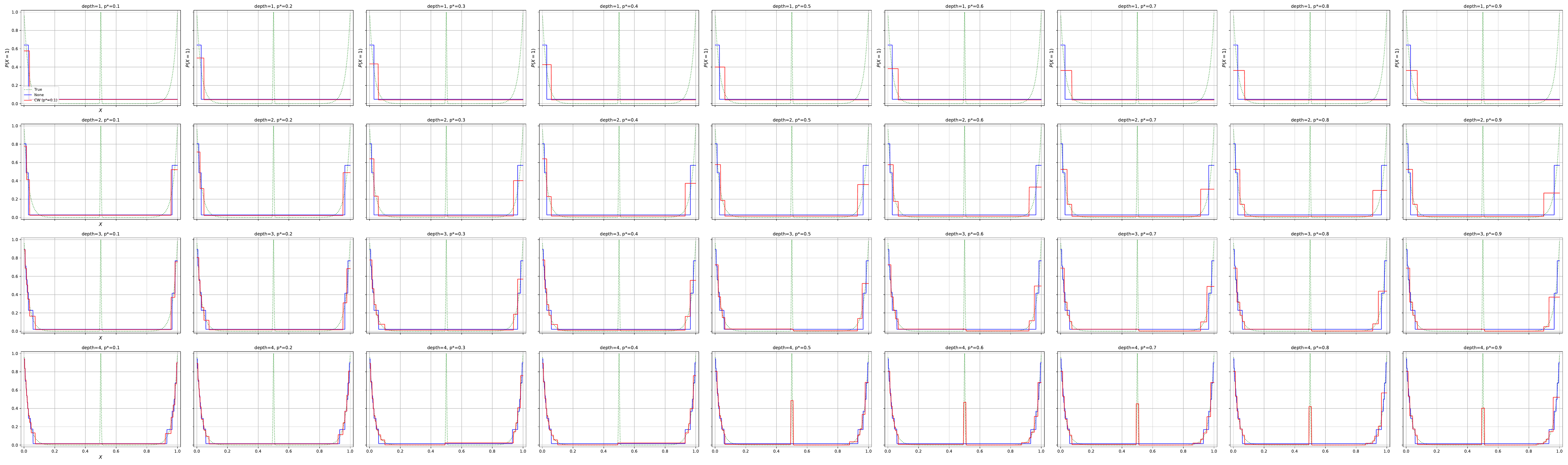}
    \includegraphics[width=0.8\linewidth]{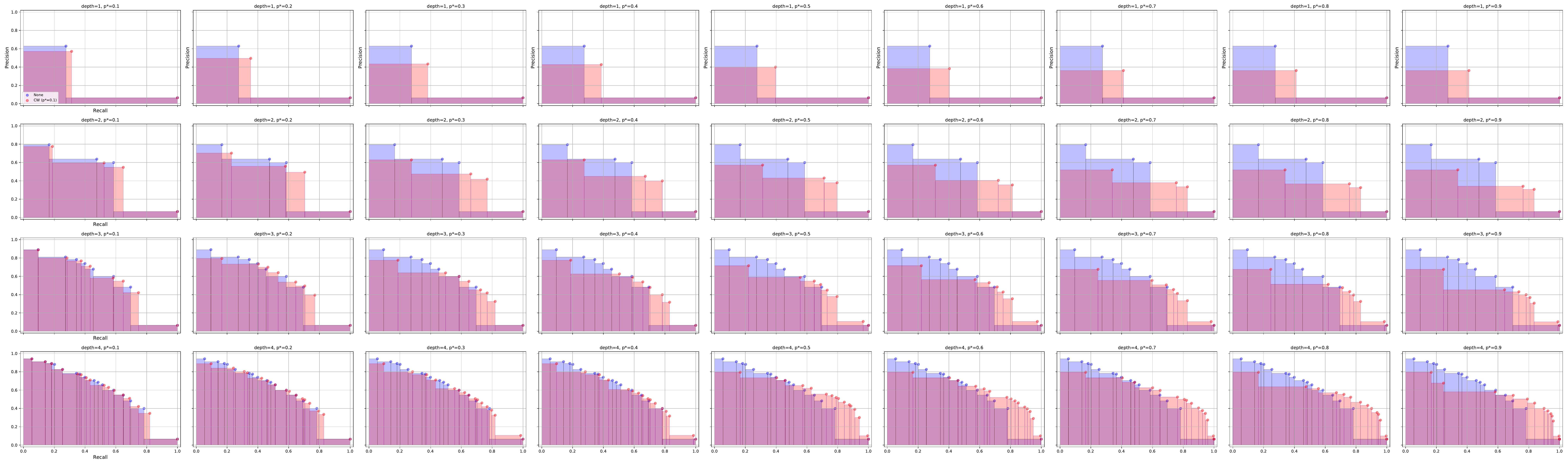}
    \includegraphics[width=0.8\linewidth]{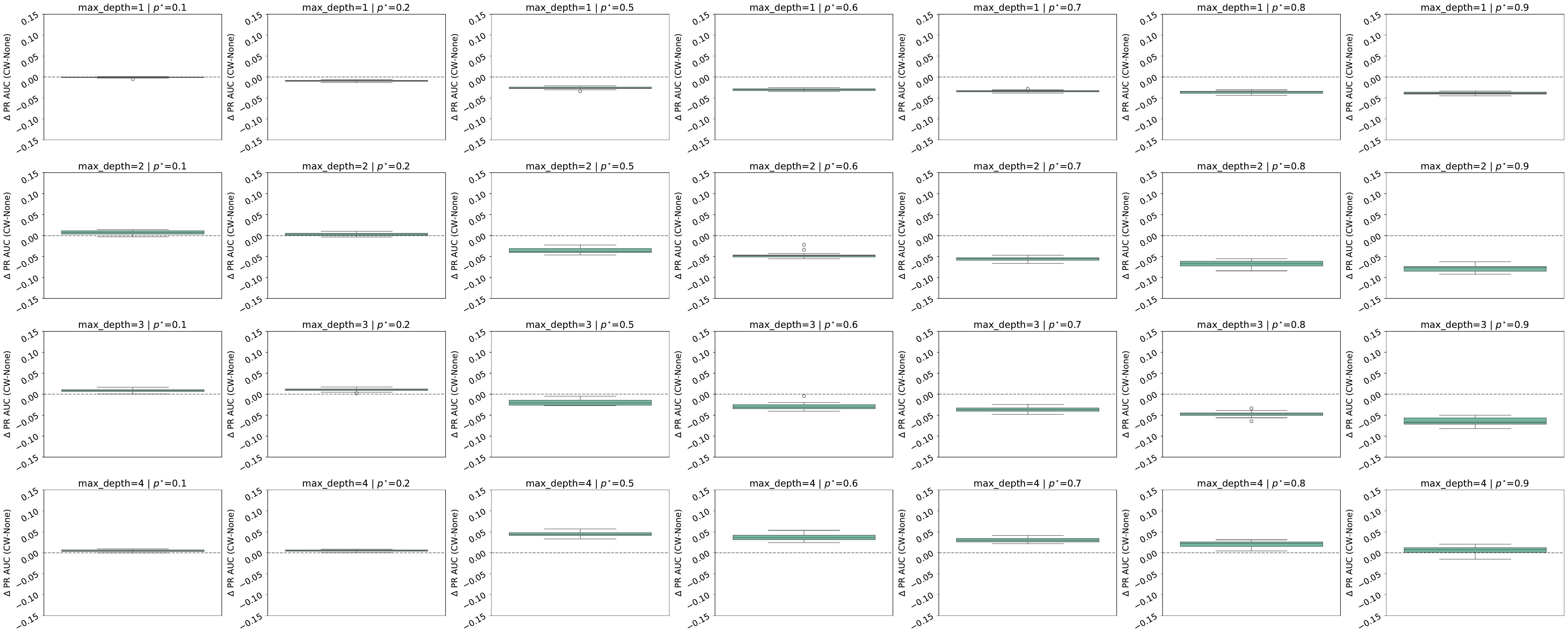}
    \caption{Estimated regression curves, PR curves, and PR AUC boxplots for \textit{ExpSetting} with peak at $x=0.5$. For the regression plots, the inverse transformation from \Cref{lemma:main-mustar} is applied to the class-weighted strategy's predictions to ensure direct comparability with the unweighted baseline.}
    \label{fig:pr_reg_curves_with_peak_0-5}
\end{figure}

\clearpage

\subsection{\textit{SinSetting}}

\label{appendix:sinsetting}
In the following, we look at another specific regression function named \textit{SinSetting}  and defined as follows: $$\mu(x) = 0.15sin\left(2\pi x+1.2\right) +0.16.$$
\Cref{fig:sinsetting-regcurves} illustrates the \textit{SinSetting} function, demonstrating the presence and absence of a peak across various localisations.

Consistent with the findings in the previous section, the \textit{SinSetting} configuration yields an improved PR AUC under the weighting strategy for specific peak locations.

\Cref{fig:sim-barplot-scatter-peak-pos-sinsetting} illustrates the empirical probability of peak detection for both the unweighted and weighted (with $p^{\star}=0.9$) models as a function of the peak localization. We observe that both methodologies successfully detect the peak across most locations, except when it resides in a zero-signal region for  $b \in [0.51,0.65]$. In this specific scenario, only the weighted classifier is capable of estimating the peak, thereby reinforcing our previous observations from \Cref{fig:sim-barplot-scatter-peak-pos-expsetting} under the \textit{ExpSetting}.
\begin{figure}[h]
    \centering
    \includegraphics[width=0.5\linewidth]{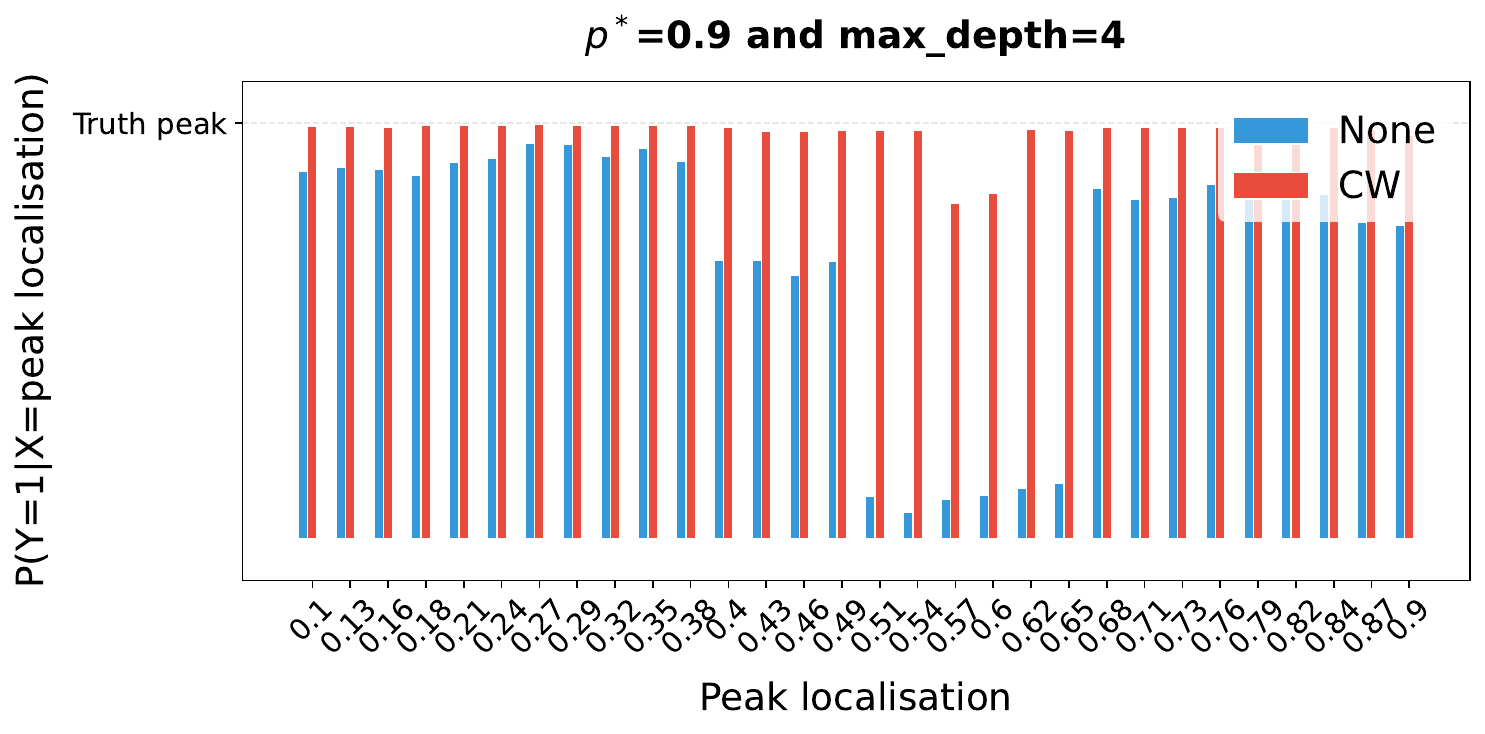}
    \caption{CART tree probability of a \textit{SinSetting} regression function with respect to peak localisation.}
    \label{fig:sim-barplot-scatter-peak-pos-sinsetting}
\end{figure}

%\Cref{fig:sim-heatmap-sinsetting} presents a heatmap analogous to \Cref{fig:sim-heatmap-expsetting}, evaluated on the \textit{SinSetting} regression function which features a peak at $x=0.6$. Consistent with our earlier findings, the baseline None strategy fails to detect the peak when it is characterized by a large height and a narrow width. However, this peak is accurately estimated across a notably broader range of parameter configurations (height and width) than in the \textit{ExpSetting}. Ultimately, the results confirm that the weighting strategy successfully recovers the peak in geometric regimes where it otherwise goes undetected.
\begin{figure}[h]
    \centering
    \includegraphics[width=1.0\linewidth]{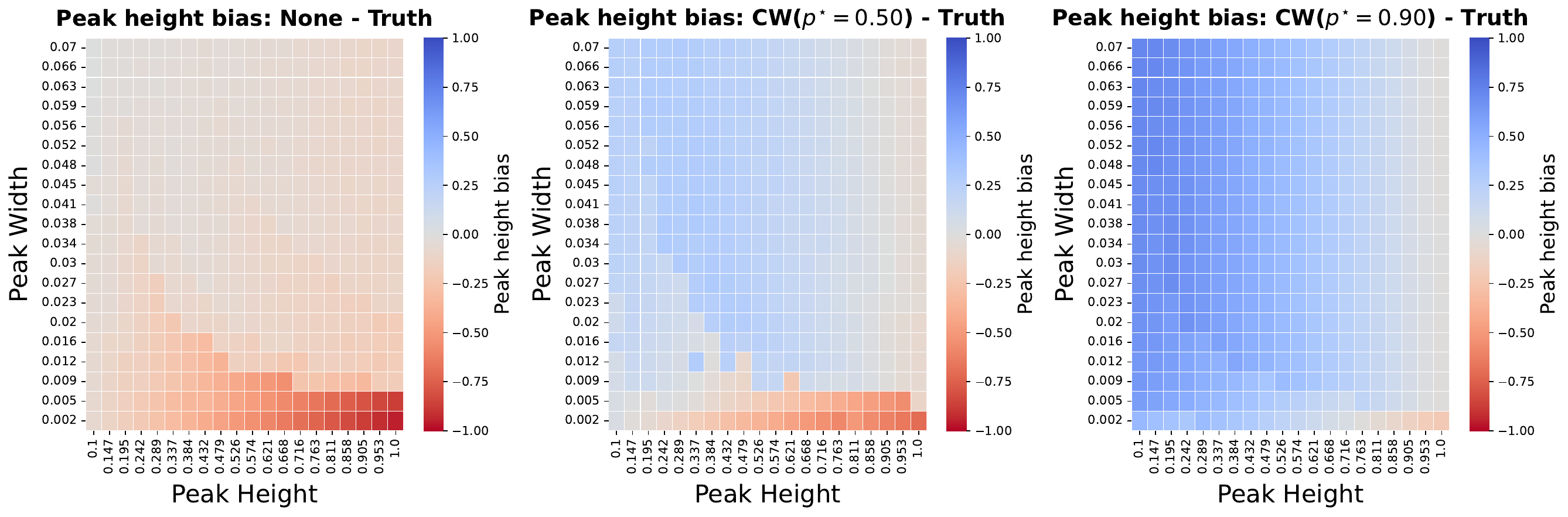}
    \caption{\textit{SinSetting} regression function: CART tree peak probability minus Truth probability with respect to peak width and height (maximum depth $=4$).}
    \label{fig:sim-heatmap-sinsetting}
\end{figure}

%\clearpage
%\subsection{Peak indicator: additional results}
%\label{appendix:add_results_peak_indicator}

%%%%%
\clearpage
%\section{Supplementary numerical experiments}
%\subsection{\texttt{ExpSetting} With or Without Peak.}

\clearpage
\section{Supplementary Theoretical Results}
\label{appendix:supp_theoretical_results}

\subsection{Tree Expression}
Let $t \in [0, 1]$ denote the decision threshold used by a CART tree to predict $\hat{Y} = 1$ within each leaf cell. Formally, for a given leaf $A$, the decision rule is
\begin{align}
    \mathds{E}[Y \mid X \in A] \ge t &\Longrightarrow \hat{Y} = 1.
\end{align}
Below, we derive explicit analytical expressions for both unweighted and weighted decision stumps.

\begin{lemma}[Default stump]
\label{lemma:default-stump}
Under \Cref{assump-infinite-sample} and \Cref{assump-1}, we have for a default stump (tree of depth $1$),
    \begin{align}
    \hat{m}_{k}(x) = \left\{ 
        \begin{array}{ll}
            1 & \text{if} \; t \in [0, \frac{\alpha \Myt_p}{2} + \beta] \\
            \mathds{1}_{\{x \geq \Myt_p\}} & \text{if} \; t \in [ \frac{\alpha \Myt_p}{2} + \beta, \frac{\alpha (\Myt_p+1)}{2} + \beta] \\
            0 & \text{if} \; t \in [ \frac{\alpha (\Myt_p+1)}{2} + \beta, 1].
        \end{array}
    \right.
\end{align}
\end{lemma}
The proof of \Cref{lemma:default-stump} can be found in \Cref{proof:default-stump}. We remark that our tree stump is constant piecewise with respect to $s$.
\begin{lemma}[Weighted stump]
\label{lemma:ponderate-stump}
Under \Cref{assump-infinite-sample} and \Cref{assump-1}, we have for a reweighted stump (tree of depth $1$),
{\small
\begin{align}
    \hat{m}_{k}^{\star}(x) = \left\{ 
        \begin{array}{ll}
            1 & \text{if} \; t \in [0, \frac{\omega_{1}(\frac{\alpha}{2}\Myt_{p}^{\star \, 2} + \beta \Myt_{p}^{\star} )}{  \omega_{0}(\Myt_{p}^{\star}- \frac{\alpha}{2} \Myt_{p}^{\star \, 2} - \beta \Myt_{p}^{\star} ) + \omega_{1}(\frac{\alpha}{2} \Myt_{p}^{\star \, 2} + \beta \Myt_{p}^{\star}) }] \\
            \mathds{1}_{\{x \geq \Myt_{p}^{\star}\}} & \text{if} \; t \in [ \frac{\omega_{1}(\frac{\alpha}{2}\Myt_{p}^{\star \, 2} + \beta \Myt_{p}^{\star} )}{  \omega_{0}(\Myt_{p}^{\star}- \frac{\alpha}{2} \Myt_{p}^{\star \, 2} - \beta \Myt_{p}^{\star} ) + \omega_{1}(\frac{\alpha}{2} \Myt_{p}^{\star \, 2} + \beta \Myt_{p}^{\star}) }, \frac{\omega_{1}(\frac{\alpha}{2}(1-\Myt_{p}^{\star \, 2}) + \beta (1-\Myt_{p}^{\star}) )}{p(1-p) -  \omega_{0}(\Myt_{p}^{\star}- \frac{\alpha}{2} \Myt_{p}^{\star \, 2} - \beta \Myt_{p}^{\star} ) - \omega_{1}(\frac{\alpha}{2} \Myt_{p}^{\star \, 2} + \beta \Myt_{p}^{\star}) }] \\
            0 & \text{if} \; t \in [\frac{\omega_{1}(\frac{\alpha}{2}(1-\Myt_{p}^{\star \, 2}) + \beta (1-\Myt_{p}^{\star}) )}{p(1-p) -  \omega_{0}(\Myt_{p}^{\star}- \frac{\alpha}{2} \Myt_{p}^{\star \, 2} - \beta \Myt_{p}^{\star} ) - \omega_{1}(\frac{\alpha}{2} \Myt_{p}^{\star \, 2} + \beta \Myt_{p}^{\star}) }, 1].
        \end{array}
    \right.
\end{align}
}
\end{lemma}
The proof of \Cref{lemma:ponderate-stump} is provided in \Cref{proof:ponderate-stump}. While previous results demonstrate that weighted trees assign greater volume to regions with higher regression function values, \Cref{lemma:ponderate-stump} establishes that class reweighting systematically shifts the decision stump's leaf-level probability estimates.

Below, we analyze the performance metrics of both trees to evaluate how these structural differences influence predictive performance.

\subsection{Metrics}

Based on \Cref{lemma:main-mustar}, one might theoretically expect the predictive performance of a decision stump to remain identical with or without sample weighting.

\begin{proposition}
\label{prop:main1-metrics}
We denote by $(\tilde{X},\tilde{Y})$ the distribution of data used for the training of a stump.
Under \Cref{assump-infinite-sample} and \Cref{assump-1}, we have for any stump (tree of depth $1$) cutting at $\Myt$,
{\footnotesize
\begin{align}
    Precision = \left\{ 
        \begin{array}{ll}
             \alpha/2 + \beta& \text{if} \; t \in \left[0, B_{A_{L}}\right] \\
            \frac{\alpha}{2} (1+\Myt) + \beta & \text{if} \; t \in \left[ B_{A_{L}}, B_{A_{R}}\right] \\
            Undefined  & \text{if} \; t \in \left[B_{A_{R}}, 1\right],
        \end{array}
    \right.
    Recall = \left\{ 
        \begin{array}{ll}
            1 & \text{if} \; t \in \left[0, B_{A_{L}}\right] \\
            \frac{\frac{\alpha}{2}(1-\Myt^{2}) + \beta (1-\Myt)}{ \frac{\alpha}{2} + \beta} & \text{if} \; t \in \left[B_{A_{L}},B_{A_{R}} \right] \\
            0 & \text{if} \; t \in \left[B_{A_{R}},1\right].
        \end{array}
    \right.
\end{align}

\begin{align}
    FPR = \left\{ 
        \begin{array}{ll}
            1 & \text{if} \; t \in \left[0, B_{A_{L}}\right] \\
            \frac{-\frac{\alpha}{2} \Myt^{2} + \Myt (1 -\beta)}{  1 -\beta - \frac{\alpha}{2}} & \text{if} \; t \in \left[B_{A_{L}},B_{A_{R}} \right] \\
            0 & \text{if} \; t \in \left[B_{A_{R}},1\right].
        \end{array}
    \right.    
    \text{F1} = \left\{ 
        \begin{array}{ll}
             \frac{\alpha/2 + \beta}{1+ \alpha/2 + \beta}& \text{if} \; t \in \left[0, B_{A_{L}}\right] \\
            \frac{ \left(\frac{\alpha}{2}(1+\Myt) + \beta\right)(1-t) }{ (\frac{\alpha}{2} + \beta)  + (1-\Myt)} & \text{if} \; t \in \left[B_{A_{L}}, B_{A_{R}}\right] \\
            Undefined  & \text{if} \; t \in \left[B_{A_{R}},1\right].
        \end{array}
    \right.
\end{align}
}
{\small
\begin{align}
    &\text{PR AUC} = \frac{\left(\frac{\alpha}{2}(1-\Myt^{2}) + \beta (1-\Myt)\right) \left(\frac{\alpha}{2}\Myt \right) }{\frac{\alpha}{2} + \beta} + \frac{\alpha}{2}+ \beta, &&
    ROC AUC = 1 - \frac{ \Myt^2 \left(1 -\frac{\alpha}{2} \Myt -\beta \right) \left(\frac{\alpha}{2}\Myt + \beta  \right)}{  \left(1 -\beta - \frac{\alpha}{2}\right) \left(\frac{\alpha}{2} + \beta\right)}.
\end{align}
}

When $(\tilde{X},\tilde{Y}) \sim (X,Y)$, we have,
\begin{align}
    B_{A_{L}} = \mathds{E}[{Y}|{X} \in A_{L}] &= \frac{\alpha t}{2} + \beta, &&
    B_{A_{R}} = \mathds{E}[{Y}|{X} \in A_{R}] = \frac{\alpha (t+1)}{2} + \beta.
\end{align}
On the other hand, when $(\tilde{X},\tilde{Y}) \sim (X^\star,Y^\star)$, we have,
\begin{align}
    B_{A_{L}} = \mathds{E}[{Y^\star}|{X^\star} \in A_{L}] &= \frac{\omega_{1}(\frac{\alpha}{2}\Myt^{2} + \beta t )}{  \omega_{0}(\Myt- \frac{\alpha}{2} \Myt^2 - \beta \Myt ) + \omega_{1}(\frac{\alpha}{2} \Myt^2 + \beta \Myt) }   \\
    B_{A_{R}} = \mathds{E}[{Y^\star}|{X^\star} \in A_{R}] &= \frac{\omega_{1}(\frac{\alpha}{2}(1-\Myt^{2}) + \beta (1-t) )}{p(1-p) -  \omega_{0}(\Myt- \frac{\alpha}{2} \Myt^2 - \beta \Myt ) - \omega_{1}(\frac{\alpha}{2} \Myt^2 + \beta \Myt) }.
\end{align}
\end{proposition}
The proof of \Cref{prop:main1-metrics} can be found in \Cref{subsec-proofs-metrics-default-stump} and \Cref{subsec-proofs-metrics-ponderate-stump}.
%The proof of \Cref{prop:main1-metrics} can be found in \Cref{proof:main1-metrics}.
\Cref{prop:main1-metrics} demonstrates that while performance metrics for both unweighted and weighted  models depend on the initial split value $\Myt$, their mathematical expressions are identical with respect to $\Myt$. By applying \Cref{lemma:IG_weighted_optimal_treshold} and the result from  \cite{ishwaran2015effect} to determine the respective split points, we observe that the weighted tree achieves lower precision but higher recall than the unweighted tree under adapted probability thresholds $s$. While practitioners widely acknowledge that class weighting improves recall at the expense of precision for a fixed probability threshold $s$, we establish that for depth-1 trees, this trade-off persists even when both probability thresholds are adapted (due to different tree split values).
Furthermore, for weighted trees, the predicted probabilities ($\mathds{E}[{Y^\star}\vert{}{X^\star} \in A_{L}]$ and $\mathds{E}[{Y^\star}\vert{}{X^\star} \in A_{R}]$) depend on both the split threshold $\Myt$ and the degree of weighting governed by $\omega_0$ and $\omega_1$. As demonstrated in \Cref{lemma:ponderate-stump}, these estimates undergo a systematic shift.

\begin{lemma}[Optimization of metrics under \Myt,$s$]
\label{lemma:opt-metrics-ts}
Let $t \in [0, 1]$ denote the decision threshold used by a CART classifier to predict $\hat{Y} = 1$ within each leaf cell. Under \Cref{assump-infinite-sample} and \Cref{assump-1}, for a tree of maximum depth $1$, we have:
\begin{itemize}
    \item  The $F_1$-score is maximized for decision thresholds $t \in \left[ \mathds{E}[\tilde{Y}\mid\tilde{X} \in A_{L}], \mathds{E}[\tilde{Y}\mid\tilde{X} \in A_{R}]\right]$, with optimal split point $\Myt$:
    \[
    \Myt = \begin{cases} 
    \frac{\alpha\left(1 + \frac{\alpha}{2} + \beta\right) - \sqrt{\alpha^2\left(1 + \frac{\alpha}{2} + \beta\right)^2 + \alpha \left[ \alpha(\beta-1) + 2 \beta^2 \right]}}{\alpha}, & \text{if } \alpha \ge \frac{2 \beta^2}{1-\beta}, \\[1.2ex]
    0, & \text{otherwise.}
    \end{cases}
    \]

    \item The Area Under the Precision-Recall Curve ($\mathrm{PR\,AUC}$) is maximized by:
    \[
    \Myt = \begin{cases} 
    \frac{2\sqrt{\beta^2 + \frac{3}{4}\alpha(\alpha + 2 \beta)} - 2 \beta}{3 \alpha}, & \text{if } \alpha > 0, \\[1.2ex]
    \text{any } \Myt \in [0,1], & \text{otherwise.}
    \end{cases}
    \]

    \item The Area Under the Receiver Operating Characteristic Curve ($\mathrm{ROC\,AUC}$) is maximized by:
    \[
    \Myt = \begin{cases} 
    \frac{3 - 6 \beta \pm \sqrt{9 - 4\beta + 4\beta^2}}{4 \alpha}, & \text{if } \alpha \left(24 - 16 \alpha\right) > \beta\left(48 \alpha - 32 + 32 \beta\right), \\[1.2ex]
    0, & \text{otherwise.}
    \end{cases}
    \]
\end{itemize}
\end{lemma}
The proof of \Cref{lemma:opt-metrics-ts} is postponed in \Cref{proof:opt-metrics-ts}.

\paragraph{Remark} Precision, Recall and FPR are optimized for pathologic cases.

\begin{theorem}
\label{th:depthk-metrics}
We assume that trees are grown to a maximum depth $k \in \mathbb{N}$, containing $2^k - 1$ internal split nodes. Let $\{\Myt_p^{h}\}_{h=1}^{2^k - 1}$ and $\{\Myt_p^{\star h}\}_{h=1}^{2^k - 1}$ denote the collections of split thresholds for the standard and weighted trees, respectively, indexed level-by-level from left to right. We reorder these thresholds monotonically via permutations $\sigma$ and $\sigma^\star$ such that $\Myt_p^{\sigma(1)} \le \Myt_p^{\sigma(2)} \le \dots \le \Myt_p^{\sigma(2^k - 1)} \in [0,1]$ for the standard tree, and $\Myt_p^{\sigma^\star(1)} \le \dots \le \Myt_p^{\sigma^\star(2^k - 1)} \in [0,1]$ for the weighted tree. Let $t \in [0, 1]$ denote the decision threshold used by a CART tree to predict $\hat{Y} = 1$ within each leaf cell.

For simplicity, let $\Myt^1 \le \Myt^2 \le \dots \le \Myt^{2^k - 1}$ denote the ordered split thresholds of a given tree (weighted or unweighted). Under \Cref{assump-infinite-sample} and \Cref{assump-1}, for any tree of maximum depth $H \in \mathbb{N}$, we have:
\begin{align}
    Precision = \left\{ 
        \begin{array}{ll}
            \alpha/2 + \beta & \text{if} \; t \in \left[0, \mathds{E}[\tilde{Y}|\tilde{X} \in A_{1}]\right] \\
            \frac{\alpha}{2} (1+\Myt^1) + \beta & \text{if} \; t \in \left[\mathds{E}[\tilde{Y}|\tilde{X} \in A_{1}], \mathds{E}[\tilde{Y}|\tilde{X} \in A_{2}]\right] \\
             \frac{\alpha}{2} (1+\Myt^2) + \beta & \text{if} \; t \in \left[\mathds{E}[\tilde{Y}|\tilde{X} \in A_{2}], \mathds{E}[\tilde{Y}|\tilde{X} \in A_{3}]\right] \\
            \vdots &\vdots\\
            \frac{\alpha}{2} (1+\Myt^{h}) + \beta  & \text{if} \; t \in \left[\mathds{E}[\tilde{Y}|\tilde{X} \in A_{h}], \mathds{E}[\tilde{Y}|\tilde{X} \in A_{h+1}]\right] \\
            \vdots &\vdots\\
            \frac{\alpha}{2} (1+\Myt^{2^{k}-1}) + \beta  & \text{if} \; t \in \left[\mathds{E}[\tilde{Y}|\tilde{X} \in A_{2^{k}-1}], \mathds{E}[\tilde{Y}|\tilde{X} \in A_{2^{k}}]\right] \\
            NA & \text{if} \; t \in \left[\mathds{E}[\tilde{Y}|\tilde{X} \in A_{2^{k}}], 1\right].
        \end{array}
    \right.
\end{align}

\begin{align}
    Recall = \left\{ 
        \begin{array}{ll}
            1 & \text{if} \; t \in \left[0, \mathds{E}[\tilde{Y}|\tilde{X} \in A_{1}]\right] \\
            \frac{\frac{\alpha}{2}(1-t_1^{2}) + \beta (1-\Myt^1)}{ \frac{\alpha}{2} + \beta} & \text{if} \; t \in \left[\mathds{E}[\tilde{Y}|\tilde{X} \in A_{1}], \mathds{E}[\tilde{Y}|\tilde{X} \in A_{2}]\right] \\
            \frac{\frac{\alpha}{2}(1-\left(\Myt^2\right)^{2}) + \beta (1-\Myt^2)}{ \frac{\alpha}{2} + \beta} & \text{if} \; t \in \left[\mathds{E}[\tilde{Y}|\tilde{X} \in A_{2}], \mathds{E}[\tilde{Y}|\tilde{X} \in A_{3}]\right] \\
            \vdots &\vdots\\
            \frac{\frac{\alpha}{2}(1-(\Myt^h)^{2}) + \beta (1-t_h)}{ \frac{\alpha}{2} + \beta}  & \text{if} \; t \in \left[\mathds{E}[\tilde{Y}|\tilde{X} \in A_{h}], \mathds{E}[\tilde{Y}|\tilde{X} \in A_{h+1}]\right] \\
            \vdots &\vdots\\
            \frac{\frac{\alpha}{2}(1-(\Myt^{2^{k}-1})^{2}) + \beta (1-\Myt^{2^{k}-1})}{ \frac{\alpha}{2} + \beta}  & \text{if} \; t \in \left[\mathds{E}[\tilde{Y}|\tilde{X} \in A_{2^{k}-1}], \mathds{E}[\tilde{Y}|\tilde{X} \in A_{2^{k}}]\right] \\
            0 & \text{if} \; t \in \left[\mathds{E}[\tilde{Y}|\tilde{X} \in A_{2^{k}}], 1\right].
        \end{array}
    \right.
\end{align}
\end{theorem}
The proof of \Cref{th:depthk-metrics} can be found in \Cref{prrof_th:depthk-metrics}.

%\clearpage
\section{Technical Lemmas}

\subsection{Related results \texorpdfstring{\Cref{lemma:main-mustar}}{lemma:main-mustar}}

\begin{lemma}
\label{tech-lemma:f_XStar}
We have for all $(X,Y)$,
\begin{align}
    f_{X^{\star}}(x) = \frac{ f_{X}(x)\left[\omega_0(1-\mu(x)) + \omega_1\mu(x) \right]}{p(1-p)}
\end{align}
    
\end{lemma}

\begin{proof}[Proof of \Cref{tech-lemma:f_XStar}]
We have,
\begin{align}
    &f_{X^{\star}}(x)\\ &= f_{X^{\star}, Y^{\star}}(x,0)  + f_{X^{\star}, Y^{\star}}(x,1) = f_{X^{\star}| Y^{\star}}(x|0) f_{Y^{\star}}(0) +f_{X^{\star}| Y^{\star}}(x|1) f_{Y^{\star}}(1)  \\
    & = f_{X| Y}(x|0) (1-p^{\star}) +f_{X| Y}(x|1) p^{\star}  \; \; \left( f_{X^{\star}|Y^{\star}}(x|y) = f_{X|Y}(x|y) \right) \\ 
    &= \left[\frac{f_{X, Y}(x,0)}{f_{X}(0)}\right] (1-p^{\star}) + \left[\frac{f_{X, Y}(x,1)}{f_{X}(1)}\right] p^{\star} \\
    &= \left[\frac{f_{Y|X}(0|x)f_{X}()}{f_{Y}(0)}\right] (1-p^{\star}) + \left[\frac{f_{Y|X}(1|x)f_{X}(x)}{f_{Y}(1)}{f_{Y}(1)}\right] p^{\star} \\
    &= \left[\frac{\left(1-f_{Y|X}(1|x)\right)f_{X}(x)}{f_{Y}(0)}\right] (1-p^{\star}) + \left[\frac{f_{Y|X}(1|x)f_{X}(x)}{f_{Y}(1)}{f_{Y}(1)}\right] p^{\star} \\
    &= \left[\frac{(1-\mu(x)) f_{X}(x)}{(1-p)}\right] (1-p^{\star}) + \left[\frac{\mu(x) f_{X}(x)}{p}\right] p^{\star}\\
    &= \frac{p(1-p^{\star})(1-\mu(x))f_{X}(x) + p^{\star}(1-p)\mu(x)f_{X}(x) }{p(1-p)}\\
    &= \frac{ f_{X}(x)\left[p(1-p^{\star})(1-\mu(x)) + p^{\star}(1-p)\mu(x) \right]}{p(1-p)}\\
    &= \frac{ f_{X}(x)\left[\omega_0(1-\mu(x)) + \omega_1\mu(x) \right]}{p(1-p)}.
\end{align}
%\as{J'ai vérfié si $\mathds{P}(X^{\star} = x)$ sommait à 1 et c'est bien le cas.}
\end{proof}

\begin{proof}[Proof of \Cref{lemma:main-mustar}]
\label{proof_lemma:main-mustar}
Then we have,
\begin{align}
    \mathds{P}(Y^{\star}=1|X^{\star}=x) &= \frac{\mathds{P}(Y^{\star} =1 \cap X^{\star})}{f_{X^{\star}}(x)}= \frac{\mathds{P}( X^{\star} | Y^{\star} =1) \mathds{P}(Y^{\star}=1)}{f_{X^{\star}}(x)} \\
    &= \frac{\mathds{P}( X=x | Y =1) p^{\star}}{f_{X^{\star}}(x)} = \frac{\frac{\mathds{P}(Y=1|X=x)f_{X}(x)}{\mathds{P}(Y=1)} p^{\star}}{f_{X^{\star}}(x)} \\
    &= \frac{\frac{\mu(x)f_{X}(x)}{p} p^{\star}}{f_{X^{\star}}(x)} = \frac{\mu(x)p^{\star}f_{X}(x)}{pf_{X^{\star}}(x)} \\
    &= \frac{\mu(x)p^{\star}f_{X}(x)}{p} *  \frac{p(1-p)}{ f_{X}(x)\left[p(1-p^{\star})(1-\mu(x)) + p^{\star}(1-p)\mu(x) \right]}\\
    &=  \frac{\mu(x)p^{\star}(1-p)}{ p(1-p^{\star})(1-\mu(x)) + p^{\star}(1-p)\mu(x) }\\
    &=  \frac{\omega_{1}\mu(x)}{ \omega_{0}(1-\mu(x)) + \omega_{1}\mu(x) } = \mu^{\star}(x),
\end{align}
with,
\begin{align}
    \omega_{0}=p(1-p^{\star}), \; \omega_{1}=p^{\star}(1-p).
\end{align}
%\as{On retrouve le même résultat que dans \url{https://arxiv.org/pdf/2408.01777} ;)}
\end{proof}

\subsection{Weighted and unweighted distributions under \texorpdfstring{\Cref{assump-1}}{assump-1} and \texorpdfstring{\Cref{assump-infinite-sample}}{assump-infinite-sample}}
In this subsection, under \Cref{assump-infinite-sample} and \Cref{assump-1}, several results are derived and used to prove the main results.

\begin{lemma}
\label{tech-lemma:distrib_XY}
    Under \Cref{assump-infinite-sample} and \Cref{assump-1}, we have
    \begin{align}
        \mathds{P}(Y=1) &= \frac{\alpha}{2} + \beta \\
        \mathds{P}_{A_L}(Y=1) &= \frac{\alpha}{2}\Myt + \beta \\
        \mathds{P}_{A_R}(Y=1) &= \frac{\alpha}{2} \left(1+\Myt\right) + \beta.
    \end{align}
\end{lemma}
The proof of \Cref{tech-lemma:distrib_XY} can be found in \Cref{proof:tech-lemma:distrib_XY}.

\begin{lemma}
\label{tech-lemma:distrib_XStarYStar}
Under \Cref{assump-infinite-sample} and \Cref{assump-1}, we have
    \begin{align}
        %\mathds{P}(Y^{\star}=1) &= \\
        \mathds{P}_{A_L}(Y^{\star}=1) &= \frac{\omega_{1} (\alpha \Myt^{2}/2 + \beta \Myt)}{\omega_{0}(\Myt-\alpha \Myt^2 /2 -\beta \Myt) + \omega_{1} (\alpha \Myt^2 /2 +\beta \Myt)}\\
        \mathds{P}_{A_R}(Y^{\star}=1) &= \frac{\omega_{1} \left(\alpha/2 + \beta- \alpha \Myt^{2}/2 - \beta \Myt \right)}{p(1-p) - \omega_{0}(\Myt-\alpha \Myt^2 /2 -\beta \Myt) - \omega_{1} (\alpha \Myt^2 /2 +\beta \Myt)}..
    \end{align}
\end{lemma}
The proof of \Cref{tech-lemma:distrib_XStarYStar} can be found in \Cref{proof:tech-lemma:distrib_XStarYStar}.

\begin{proof}[Proof of \Cref{tech-lemma:distrib_XY}]
\label{proof:tech-lemma:distrib_XY}
\textbf{Derivation of } $\mathds{P}(Y=1)$,
We have,
\begin{align}
    \mathds{P}(Y=1) &= \int_{0}^{1} f_{Y|X}(1|x) \mathrm{d}x =\int_{0}^{1} (\alpha x + \beta) * 1 \mathrm{d}x= \int_{0}^{1} (\alpha x + \beta) \mathrm{d}x = \left[\frac{\alpha}{2}x^2 + \beta x\right]_{0}^{1} \\
    &= \frac{\alpha}{2} + \beta.
\end{align}

On hand we have,
\begin{align}
    \mathds{P}_{A_L}(Y=1) &= \int_{\mathds{R}} f_{Y|X}(1|x)f_{X}^{A_L}(x) \mathrm{d}x = \int_{0}^{\Myt}(\alpha x + \beta) \frac{1}{\Myt} \; \; \left( X \text{truncated on} [0,\Myt] \sim \mathcal{U}([0,t])\right) \\
    &=\frac{1}{\Myt} \left[\frac{\alpha }{2}x^2 + \beta x \right]_0^\Myt = \frac{1}{\Myt} \left[\frac{\alpha}{2}t^2+ \beta \Myt\right] = \frac{\alpha}{2}\Myt + \beta.
\end{align}

On the other hand, we have,
\begin{align}
    \mathds{P}_{A_R}(Y=1) &= \int_{\mathds{R}} f_{Y|X}(1|x)f_{X}^{A_R}(x) \mathrm{d}x = \int_{\Myt}^{1}(\alpha x + \beta) \frac{1}{1-\Myt} \; \; \left( X \text{truncated on} [\Myt,1] \sim \mathcal{U}([\Myt,1])\right) \\
    &= \frac{1}{1-\Myt} \left[\frac{\alpha}{2}+\beta -\frac{\alpha}{2}\Myt^2- \beta \Myt\right]= \frac{1}{1-\Myt} \left[\frac{\alpha}{2} \left(1-\Myt^2\right)+\beta(1-\Myt) \right] \\
    & = \frac{\frac{\alpha}{2} \left(1-\Myt^2\right)}{1-\Myt} + \beta = \frac{\alpha}{2} \left(1+\Myt\right) + \beta.
\end{align}

\end{proof}

\begin{proof}[Proof of \Cref{tech-lemma:distrib_XStarYStar}]
\label{proof:tech-lemma:distrib_XStarYStar}

We have for $\Myt \in [0,1]$,
\begin{align}
    \mathds{P}(X^{\star}\leq \Myt) &= \int_{0}^{\Myt} f_{X^{\star}}\mathrm{d}x = \int_{0}^{\Myt} f_{X} \frac{\omega_{0}(1-\mu(x)) + \omega_{1}\mu(x)}{p(1-p)}\mathrm{d}x \; \;  (\Cref{tech-lemma:f_XStar})\\
    &=  \frac{1}{p(1-p)} \int_{0}^{\Myt} 1* \omega_{0}(1-\mu(x)) + \omega_{1}\mu(x) \mathrm{d}x \\
    &=  \frac{1}{p(1-p)} \left( \omega_{0}[x-\alpha x^2 /2 -\beta x]_{0}^{\Myt} + \omega_{1} [\alpha x^2 /2 +\beta x]_{0}^{\Myt} \right) \\
    &=  \frac{1}{p(1-p)} \left( \omega_{0}(\Myt-\alpha \Myt^2 /2 -\beta \Myt) + \omega_{1} (\alpha \Myt^2 /2 +\beta \Myt) \right).
\end{align}

\begin{align}
    &\mathds{P}(0<X^{\star}\leq \Myt) \\&= \frac{f_{X^\star}(x) \mathds{1}_{\{0 <  x \leq \Myt \}}}{\mathds{P}(X^{\star}\leq t) - \mathds{P}(X^{\star}\leq 0)} \\
    &=\frac{f_{X}(x) \frac{\omega_{0}(1-\mu(x)) + \omega_{1}\mu(x)}{p(1-p)} \mathds{1}_{\{0 <  x \leq \Myt \}} }{\frac{1}{p(1-p)} \left( \omega_{0}(\Myt-\alpha \Myt^2 /2 -\beta t) + \omega_{1} (\alpha \Myt^2 /2 +\beta t) \right) - 0}  \; \;  (\Cref{tech-lemma:f_XStar})\\
    &=\frac{\omega_{0}(1-\mu(x)) + \omega_{1}\mu(x)}{p(1-p)}  * \frac{p(1-p)}{ \omega_{0}(\Myt-\alpha \Myt^2 /2 -\beta \Myt) + \omega_{1} (\alpha t^2 /2 +\beta t) } * \mathds{1}_{\{0 <  x \leq \Myt \}} \\
    &= \frac{\left[\omega_{0}(1-\mu(x)) + \omega_{1}\mu(x) \right]\mathds{1}_{\{0 <  x \leq \Myt \}}}{ \omega_{0}(\Myt-\alpha \Myt^2 /2 -\beta \Myt) + \omega_{1} (\alpha \Myt^2 /2 +\beta \Myt) }.
\end{align}
%\as{see \url{https://en.wikipedia.org/wiki/Truncated_distribution}} \as{Il faudra introduire la notation $f^{\star}(x)$}

Finally,
\begin{align}
    &\mathds{P}_{A_L}(Y^{\star}=1)  = \int_{0}^{1} \mathds{P}f_{Y^{\star}|X^{\star}}(1|x)f_{X^{\star}}^{A_L}(x)\mathrm{d}x \\
    &= \int_{0}^{1} \frac{\omega_{1}\mu(x)}{ \omega_{0}(1-\mu(x)) + \omega_{1}\mu(x) } * \frac{\left[\omega_{0}(1-\mu(x)) + \omega_{1}\mu(x) \right]\mathds{1}_{\{0 <  x \leq \Myt \}}}{ \omega_{0}(t-\alpha t^2 /2 -\beta \Myt) + \omega_{1} (\alpha t^2 /2 +\beta \Myt) } \mathrm{d}x  \; \; (\Cref{lemma:main-mustar})\\
    &= \frac{\omega_{1}}{\omega_{0}(t-\alpha t^2 /2 -\beta \Myt) + \omega_{1} (\alpha t^2 /2 +\beta t)} \int_{0}^{1}\mu(x) \mathds{1}_{\{0 <  x \leq \Myt \}} \mathrm{d}x \\
    &= \frac{\omega_{1} (\alpha \Myt^{2}/2 + \beta \Myt)}{\omega_{0}(\Myt-\alpha \Myt^2 /2 -\beta \Myt) + \omega_{1} (\alpha \Myt^2 /2 +\beta \Myt)}.
\end{align}

\paragraph{Derivation of} $\mathds{P}_{A_R}(Y^{\star}=1)$

We have for $\Myt \in [0,1]$,
\begin{align}
    &\mathds{P}(\Myt<X^{\star}\leq 1) = \frac{f_{X^\star}(x) \mathds{1}_{\{\Myt <  x \leq 1 \}}}{\mathds{P}(X^{\star}\leq 1) - \mathds{P}(X^{\star}\leq \Myt)} \\
    &=\frac{f_{X}(x) \frac{\omega_{0}(1-\mu(x)) + \omega_{1}\mu(x)}{p(1-p)} \mathds{1}_{\{\Myt <  x \leq 1 \}}}{ 1 - \frac{1}{p(1-p)} \left( \omega_{0}(\Myt-\alpha \Myt^2 /2 -\beta \Myt) + \omega_{1} (\alpha \Myt^2 /2 +\beta \Myt) \right)} \; \;  (\Cref{tech-lemma:f_XStar})\\
    &=\frac{\omega_{0}(1-\mu(x)) + \omega_{1}\mu(x)}{p(1-p)}   \frac{p(1-p)}{ p(1-p) - \omega_{0}(\Myt-\alpha t^2 /2 -\beta \Myt) - \omega_{1} (\alpha t^2 /2 +\beta \Myt) }  \mathds{1}_{\{\Myt <  x \leq 1 \}} \\
    &= \frac{\left[\omega_{0}(1-\mu(x)) + \omega_{1}\mu(x)\right]\mathds{1}_{\{\Myt <  x \leq 1 \}}}{ p(1-p) - \omega_{0}(\Myt-\alpha \Myt^2 /2 -\beta \Myt) - \omega_{1} (\alpha \Myt^2 /2 +\beta \Myt) }.
\end{align}

Then,
\begin{align}
    &\mathds{P}_{A_R}(Y^{\star}=1) = \int_{0}^{1} \mathds{P}f_{Y^{\star}|X^{\star}}(1|x)f_{X^{\star}}^{A_R}(x)\mathrm{d}x \\
    &= \int_{0}^{1} \frac{\omega_{1}\mu(x)}{ \omega_{0}(1-\mu(x)) + \omega_{1}\mu(x) } * \frac{\left[\omega_{0}(1-\mu(x)) + \omega_{1}\mu(x)\right] \mathds{1}_{\{\Myt <  x \leq 1 \}}}{ p(1-p) - \omega_{0}(\Myt-\alpha \Myt^2 /2 -\beta \Myt) - \omega_{1} (\alpha \Myt^2 /2 +\beta \Myt) } \mathrm{d}x  \\
    &= \frac{\omega_{1}}{p(1-p) -\omega_{0}(t-\alpha \Myt^2 /2 -\beta \Myt) - \omega_{1} (\alpha \Myt^2 /2 +\beta \Myt)} \int_{0}^{1}\mu(x)\mathds{1}_{\{t <  x \leq 1 \}}\mathrm{d}x \\
    &= \frac{\omega_{1} \left[(\alpha/2 + \beta)- (\alpha t^{2}/2 + \beta \Myt)\right]}{p(1-p) - \omega_{0}(t-\alpha \Myt^2 /2 -\beta \Myt) - \omega_{1} (\alpha t^2 /2 +\beta \Myt)} \\
    &= \frac{\omega_{1} \left(\alpha/2 + \beta- \alpha \Myt^{2}/2 - \beta t \right)}{p(1-p) - \omega_{0}(\Myt-\alpha \Myt^2 /2 -\beta \Myt) - \omega_{1} (\alpha \Myt^2 /2 +\beta \Myt)}.
\end{align}
    
\end{proof}

\section{Impurity Gain: General Case}

\subsection{Default stumps split}
\begin{proof}[Proof of \Cref{lemma:ig-default}]
\label{ig-default}
We denote by $L_0$ the impurity gain the decision tree without applying a rebalancing strategy (none).

We have,
\begin{align}
    GI = \mathds{P}\left(X \in A_L\right)\mathds{P}\left(X \in A_R \right) \left[\mathds{E}\left[Y|X \in A_L\right] - \mathds{E}\left[Y|X \in A_R\right]\right]^2
\end{align}, thus $L_0$ can also be expressed without any asumtions on $(X,Y)$ as,
\begin{align}
    &L_0(\Myt) \\
    &=I_L I_R\left[ \frac{J_L}{I_L} -\frac{J_R}{I_R}  \right]^2 %\\
    %& = t(1-t) \left[ \frac{(\frac{\alpha}{2}t^2 + \beta t )}{t} - \frac{\frac{\alpha}{2} \left(1-t^2\right)+\beta(1-t) }{1-t}\right]^2 \\
    %& = t(1-t) \left[ (\frac{\alpha}{2} t + \beta ) - \frac{\alpha}{2} \left(1+t\right)+\beta \right]^2 \\
\end{align}

\end{proof}

\subsection{Ponderate stumps split}

\begin{proof}[Proof of \Cref{lemma:ig-ponderate} ]
\label{proof:ig-ponderate}
We have,
\begin{align}
    &f_{X^{\star}}(x)\\ &= f_{X^{\star}, Y^{\star}}(x,0)  + f_{X^{\star}, Y^{\star}}(x,1) \\
    &= f_{X^{\star}| Y^{\star}}(x|0) f_{Y^{\star}}(0) +f_{X^{\star}| Y^{\star}}(x|1) f_{Y^{\star}}(1)  \\
    &= f_{X| Y}(x|0) (1-p^{\star}) +f_{X| Y}(x|1) p^{\star}  \; \;  \left( f_{X^{\star}|Y^{\star}}(x|y) = f_{X|Y}(x|y) \right) \\ 
    &= \left[\frac{f_{X, Y}(x,0)}{f_{X}(0)}\right] (1-p^{\star}) + \left[\frac{f_{X, Y}(x,1)}{f_{X}(1)}\right] p^{\star} \\
    &= \left[\frac{f_{Y|X}(0|x)f_{X}(x)}{f_{Y}(0)}\right] (1-p^{\star}) + \left[\frac{f_{Y|X}(1|x)f_{X}(x)}{f_{Y}(1)}{f_{Y}(1)}\right] p^{\star} \\
    &= \left[\frac{\left(1-f_{Y|X}(1|x)\right)f_{X}(x)}{f_{Y}(0)}\right] (1-p^{\star}) + \left[\frac{f_{Y|X}(1|x)f_{X}(x)}{f_{Y}(1)}{f_{Y}(1)}\right] p^{\star} \\
    &= \left[\frac{(1-\mu(x)) f_{X}(x)}{(1-p)}\right] (1-p^{\star}) + \left[\frac{\mu(x) f_{X}(x)}{p}\right] p^{\star}\\
    &= \frac{p(1-p^{\star})(1-\mu(x))f_{X}(x) + p^{\star}(1-p)\mu(x)f_{X}(x) }{p(1-p)}\\
    &= \frac{ f_{X}(x)\left[p(1-p^{\star})(1-\mu(x)) + p^{\star}(1-p)\mu(x) \right]}{p(1-p)}\\
    &= \frac{ f_{X}(x)\left[\omega_0(1-\mu(x)) + \omega_1\mu(x) \right]}{p(1-p)}.
\end{align}

Thus,
\begin{align}
    &\int_{\mathcal{X}} f_{X^{\star}}(x) \mathrm{d}x\\ 
    &=\int_{\mathcal{X}} \frac{\omega_1 - \omega_0 }{p(1-p)}  f_{X}(x) \mu(x) + \frac{\omega_0}{p(1-p)} f_{X}(x) \mathrm{d}x\\
    & = \frac{\omega_0}{p(1-p)} + \frac{\omega_1 - \omega_0 }{p(1-p)}\int_{\mathcal{X}} f_{X}(x) \mu(x)\mathrm{d}x.
\end{align}
Finally, we obtain the following property:
\begin{align}
    \int_{\mathcal{X}} f_{X}(x) \mu(x)\mathrm{d}x = \left(1 - \frac{\omega_0}{p(1-p)}\right)\frac{p(1-p)}{\omega_1 - \omega_0 } && \left( p \not = p^{\star}\right),
\end{align}
which is also equivalent to:
\begin{align}
    p = \left(1 - \frac{\omega_0}{p(1-p)}\right)\frac{p(1-p)}{\omega_1 - \omega_0 } && \left( p \not = p^{\star}\right).
\end{align}

Let's derive  $f_{X^{\star}}^{A_L}(x)$,
\begin{align}
    &f_{X^{\star}}^{A_L}(x) \\
    &= \frac{ f_{X}(x)\frac{ \left[\omega_0(1-\mu(x)) + \omega_1\mu(x) \right]}{p(1-p)} \mathds{1}_{\{x\in A_L\}}}{ \int_{A_L} f_{X}(z)\frac{ \left[\omega_0(1-\mu(z)) + \omega_1\mu(z) \right]}{p(1-p)} \mathrm{d}z}
\end{align}

Then,
\begin{align}
    &f_{Y^{\star}}^{A_L}(1) \\
    &= \int_{ x \in \mathcal{X}} f_{Y^{\star}|X^{\star}}(1|x)  f_{X^{\star}}^{A_L}(x) \\
    &= \int_{ x \in \mathcal{X}} \frac{\omega_{1}\mu(x)}{ \omega_{0}(1-\mu(x)) + \omega_{1}\mu(x) } \frac{ f_{X}(x)\frac{ \left[\omega_0(1-\mu(x)) + \omega_1\mu(x) \right]}{p(1-p)} \mathds{1}_{\{x\in A_L\}}}{ \int_{A_L} f_{X}(z)\frac{ \left[\omega_0(1-\mu(z)) + \omega_1\mu(z) \right]}{p(1-p)} \mathrm{d}z} \mathrm{d}x \\
    &= \frac{\omega_1}{\int_{A_L} f_{X}(z) \left[\omega_0(1-\mu(z)) + \omega_1\mu(z) \right] \mathrm{d}z} \int_{ x \in \mathcal{X}} f_{X}(x) \mu(x) \mathds{1}_{\{x\in A_L\}}
\end{align}

Let's derive the Gini impurity gain for the ponderate tree:
\begin{align}
    &L(s) = \mathds{P}\left(X^{\star} \in A_L\right)\mathds{P}\left(X^{\star} \in A_R \right) \left[\mathds{E}\left[Y^{\star}|X^{\star} \in A_L\right] - \mathds{E}\left[Y^{\star}|X^{\star} \in A_R\right]\right]^2 \\
    &= \left( \int_{x\in A_{L}} f_{X^{\star}}(x) \mathrm{d}x \right) \left( \int_{x\in A_{R}} f_{X^{\star}}(x) \mathrm{d}x \right) \left[ \int_{x\in A_{L}} f_{X^{\star}}(x) \mu^{\star}(x)\mathrm{d}x - \int_{x\in A_{R}} f_{X^{\star}}(x) \mu^{\star}(x)\mathrm{d}x \right]^2 \\
    &= \left( \left[ \frac{\omega_1-\omega_0}{p(1-p)}\int_{x\in A_{L}} f_{X}(x) \mu(x)\mathrm{d}x \right] + \left[ \frac{\omega_0}{p(1-p)} \int_{x\in A_{L}} f_{X}(x)  \mathrm{d}x \right] \right) \\
    & \qquad \times \left( \left[\frac{\omega_1-\omega_0}{p(1-p)}\int_{x\in A_{R}} f_{X}(x)\mu(x) \mathrm{d}x \right] + \left[ \frac{\omega_0}{p(1-p)} \int_{x\in A_{R}} f_{X}(x) \mathrm{d}x  \right]  \right) \\
    & \qquad \times \Biggl[ \frac{\omega_1}{\int_{A_L} f_{X}(z) \left[\omega_0(1-\mu(z)) + \omega_1\mu(z) \right] \mathrm{d}z}  \int_{x\in A_{L}} f_{X}(x) \mu(x)\mathrm{d}x \\
    & \qquad \qquad - \frac{\omega_1}{\int_{A_R} f_{X}(z) \left[\omega_0(1-\mu(z)) + \omega_1\mu(z) \right] \mathrm{d}z}  \int_{x\in A_{R}} f_{X}(x) \mu(x)\mathrm{d}x \Biggl]^2 \\
    &= \left( \left[ \frac{\omega_1-\omega_0}{p(1-p)}\int_{x\in A_{L}} f_{X}(x) \mu(x)\mathrm{d}x \right] + \left[ \frac{\omega_0}{p(1-p)} \int_{x\in A_{L}} f_{X}(x)  \mathrm{d}x \right] \right) \\
    & \qquad \times \left( \left[\frac{\omega_1-\omega_0}{p(1-p)}\int_{x\in A_{R}} f_{X}(x)\mu(x) \mathrm{d}x \right] + \left[ \frac{\omega_0}{p(1-p)} \int_{x\in A_{R}} f_{X}(x) \mathrm{d}x  \right]  \right) \\
    & \qquad \times \Biggl[ \frac{\omega_1}{\omega_0 \int_{z \in A_L} f_X(z) \mathrm{d}z + (\omega_1 - \omega_0)\int_{z \in A_L} f_X(z) \mu(z) \mathrm{d}z}  \int_{x\in A_{L}} f_{X}(x) \mu(x)\mathrm{d}x \\
    & \qquad \qquad - \frac{\omega_1}{\omega_0 \int_{z \in A_R} f_X(z) \mathrm{d}z + (\omega_1 - \omega_0)\int_{z \in A_R} f_X(z) \mu(z) \mathrm{d}z}  \int_{x\in A_{R}} f_{X}(x) \mu(x)\mathrm{d}x \Biggl]^2.
\end{align}

Let,
\begin{align}
   &I_L = \int_{x \in A_L} f_X(x) \mathrm{d}x && I_R = \int_{x \in A_R} f_X(x) \mathrm{d}x  \\
    &J_L =  \int_{x \in A_L} f_X(x) \mu(x) \mathrm{d}x 
    &&J_R = \int_{x \in A_R} f_X(x) \mu(x) \mathrm{d}x.
\end{align}

If we suppose $p^{\star} = p + \varepsilon$, then we have
\begin{align}
    & L(s) \\
    &= \left( \left[ \frac{\varepsilon}{p(1-p)}\int_{x\in A_{L}} f_{X}(x) \mu(x)\mathrm{d}x \right] + \left[ \frac{p(1-p-\varepsilon)}{p(1-p)} \int_{x\in A_{L}} f_{X}(x)  \mathrm{d}x \right] \right) \\
    & \qquad \times \left( \left[\frac{\varepsilon}{p(1-p)}\int_{x\in A_{R}} f_{X}(x)\mu(x) \mathrm{d}x \right] + \left[ \frac{p(1-p-\varepsilon)}{p(1-p)} \int_{x\in A_{R}} f_{X}(x) \mathrm{d}x  \right]  \right) \\
    & \qquad \times \Biggl[ \frac{(p+\varepsilon)(1-p)}{p(1-p-\varepsilon) \int_{z \in A_L} f_X(z) \mathrm{d}z + \varepsilon \int_{z \in A_L} f_X(z) \mu(z) \mathrm{d}z}  \int_{x \in A_{L}} f_{X}(x) \mu(x)\mathrm{d}x \\
    & \qquad \qquad - \frac{(p+\varepsilon)(1-p)}{p(1-p-\varepsilon) \int_{z \in A_R} f_X(z) \mathrm{d}z + \varepsilon \int_{z \in A_R} f_X(z) \mu(z) \mathrm{d}z}  \int_{x\in A_{R}} f_{X}(x) \mu(x)\mathrm{d}x \Biggl]^2 \\
    & = \left( \frac{\varepsilon}{p(1-p)} J_L + \frac{p(1-p-\varepsilon)}{p(1-p)} I_L\right) \left( \frac{\varepsilon}{p(1-p)} J_R + \frac{p(1-p-\varepsilon)}{p(1-p)} I_R\right) \\
    & \qquad \times \left[(p+\varepsilon)(1-p) \left( \frac{1}{p(1-p-\varepsilon)I_L + \varepsilon J_L} J_L-\frac{1}{p(1-p-\varepsilon)I_R + \varepsilon J_R} J_R \right)\right]^2  \\
    & = \left( \frac{\varepsilon}{p(1-p)} J_L + \left(1 - \frac{\varepsilon}{(1-p)}\right) I_L\right) \left( \frac{\varepsilon}{p(1-p)} J_R + \left(1 - \frac{\varepsilon}{(1-p)}\right) I_R\right) \\
    & \qquad \times \left[(p+\varepsilon)(1-p) \left( \frac{J_L}{p(1-p-\varepsilon)I_L + \varepsilon J_L} -\frac{J_R}{p(1-p-\varepsilon)I_R + \varepsilon J_R}  \right)\right]^2 
\end{align}
\begin{align}
    & = \left(I_L +\frac{\varepsilon}{p(1-p)} \left(J_L - I_L \right)\right) \left(I_R +\frac{\varepsilon}{p(1-p)} \left(J_R - I_R \right)\right) \\
    & \qquad \times \left[(p+\varepsilon)(1-p) \left( \frac{J_L}{p(1-p)I_L + \varepsilon (J_L-pI_L)} -\frac{J_R}{p(1-p)I_R + \varepsilon (J_R-pI_R)}  \right)\right]^2 \\
    & = \left(I_L I_R +\frac{\varepsilon}{p(1-p)} \left( I_R(J_L - I_L) + I_L(J_R - I_R) + \frac{\varepsilon}{p(1-p)}(J_L - I_L) (J_R - I_R) \right)\right) \\
    & \qquad \times \left[\frac{(p+\varepsilon)(1-p)}{p(1-p)} \left( \frac{J_L}{I_L + \varepsilon \frac{(J_L-pI_L)}{p(1-p)}} -\frac{J_R}{I_R + \varepsilon \frac{(J_R-pI_R)}{p(1-p)}}  \right)\right]^2 \\
    & = \left(I_L I_R +\frac{\varepsilon}{p(1-p)} \left( I_R(J_L - I_L) + I_L(J_R - I_R) + \frac{\varepsilon}{p(1-p)}(J_L - I_L) (J_R - I_R) \right)\right) \\
    & \qquad \times \left[\left(1 +\frac{\varepsilon}{p}\right) \left( \frac{J_L}{I_L} \frac{1}{1 + \varepsilon \frac{(J_L-pI_L)}{I_L p(1-p)}} -\frac{J_R}{I_R} \frac{1}{1 + \varepsilon \frac{(J_R-pI_R)}{I_R p(1-p)}}  \right)\right]^2. 
\end{align}

Finally:
\begin{align}
    &L(s) \\
    & = I_L I_R G \\
    & \qquad +\varepsilon \left[ \frac{1}{p(1-p)} \left( I_R(J_L - I_L) + I_L(J_R - I_R)  \right) + I_L I_R \frac{2}{p} \right]G \\
    & \qquad + \varepsilon^2 \Biggl[ \frac{1}{p^2(1-p)^2} (J_L - I_L) (J_R - I_R) \left(1+ \frac{2\varepsilon}{p} + \frac{\varepsilon^2}{p^2}\right) \\
    & \qquad \qquad + \frac{1}{p(1-p)} \left( I_R(J_L - I_L) + I_L(J_R - I_R) \left(\frac{2}{p} + \frac{\varepsilon}{p^2}\right) \right) + I_LI_R \frac{1}{p^2}\Biggl]G, 
\end{align}

where,
\begin{align}
    G = \left( \frac{J_L}{I_L} \frac{1}{1 + \varepsilon \frac{(J_L-pI_L)}{I_L p(1-p)}} -\frac{J_R}{I_R} \frac{1}{1 + \varepsilon \frac{(J_R-pI_R)}{I_R p(1-p)}}  \right)^2
\end{align}

For $\varepsilon$ small enough we have,
\begin{align}
    \frac{1}{1 +\varepsilon \frac{ J_L - pI_L}{I_L p (1-p)} } \approx 1 - \varepsilon \frac{ J_L - pI_L}{I_L p (1-p)} + \mathcal{O}(\varepsilon^2), \\
    \frac{1}{1 +\varepsilon \frac{ J_R - pI_R}{I_R p (1-p)} } \approx 1 - \varepsilon \frac{ J_R - pI_R}{I_R p (1-p)} + \mathcal{O}(\varepsilon^2).
\end{align}

Thus,
\begin{align}
    &  \frac{J_L}{I_L}\frac{1}{1 +\varepsilon \frac{ J_L - pI_L}{I_L p (1-p)} } -\frac{J_R}{I_R} \frac{1}{1 + \varepsilon \frac{ J_R - pI_R }{I_R p (1-p)} }    \\
    &  \approx \frac{J_L}{I_L} \left(1 - \varepsilon \frac{ J_L - pI_L}{I_L p (1-p)}\right) -\frac{J_R}{I_R} \left(1 - \varepsilon \frac{ J_R - pI_R}{I_R p (1-p)}\right)    + \mathcal{O}(\varepsilon^2) \\
    & =  \frac{J_L}{I_L} - \frac{J_R}{I_R} - \varepsilon \left[ \frac{J_L}{I_L}\frac{ J_L - pI_L}{I_L p (1-p)} + \frac{J_R}{I_R}\frac{ J_R - pI_R}{I_R p (1-p)} \right]    + \mathcal{O}(\varepsilon^2) \\
    &= G_0 + \varepsilon G_1. 
\end{align}

Thus,
\begin{align}
    &L(s) \\
    & \approx  I_L I_R \left( G_0 + \varepsilon G_1 \right)^2 \\
    & \qquad +\varepsilon \left[ \frac{1}{p(1-p)} \left( I_R(J_L - I_L) + I_L(J_R - I_R)  \right) + I_L I_R \frac{2}{p} \right]\left( G_0 + \varepsilon G_1 \right)^2 \\
    & \qquad + \varepsilon^2 \Biggl[ \frac{1}{p^2(1-p)^2} (J_L - I_L) (J_R - I_R) \left(1+ \frac{2\varepsilon}{p} + \frac{\varepsilon^2}{p^2}\right) \\
    & \qquad \qquad + \frac{1}{p(1-p)} \left( I_R(J_L - I_L) + I_L(J_R - I_R) \left(\frac{2}{p} + \frac{\varepsilon}{p^2}\right) \right) + I_LI_R \frac{1}{p^2}\Biggl]\left( G_0 + \varepsilon G_1 \right)^2 \\
    & \qquad + \mathcal{O}(\varepsilon^2)
\end{align}
\begin{align}
    & =  I_L I_RG_0^2 \\
    & \qquad + \varepsilon 2G_0G_1 I_L I_R + \varepsilon \left[ \frac{1}{p(1-p)} \left( I_R(J_L - I_L) + I_L(J_R - I_R)  \right) + I_L I_R \frac{2}{p} \right]\left( G_0 + \varepsilon G_1 \right)^2  \\
    & \qquad +\varepsilon^2 G_1^2 I_L I_R + \varepsilon^2 \Biggl[ \frac{1}{p^2(1-p)^2} (J_L - I_L) (J_R - I_R) \left(1+ \frac{2\varepsilon}{p} + \frac{\varepsilon^2}{p^2}\right) \\
    & \qquad \qquad + \frac{1}{p(1-p)} \left( I_R(J_L - I_L) + I_L(J_R - I_R) \left(\frac{2}{p} + \frac{\varepsilon}{p^2}\right) \right) + I_LI_R \frac{1}{p^2}\Biggl]\left( G_0 + \varepsilon G_1 \right)^2 \\
    & \qquad + \mathcal{O}(\varepsilon^2)
\end{align}
\begin{align}
    & =  L_0(s) \\
    & + \varepsilon \left[2G_0G_1 I_L I_R +  G_0^2 \left(\frac{1}{p(1-p)} \left( I_R(J_L - I_L) + I_L(J_R - I_R)  \right) + I_L I_R \frac{2}{p} \right) \right]  \\
    &   + \varepsilon^2 \Biggl[ G_1^2 I_L I_R \\
    & \qquad + \left(\frac{1}{p(1-p)} \left( I_R(J_L - I_L) + I_L(J_R - I_R)  \right) + I_L I_R \frac{2}{p} \right) \left( 2\varepsilon G_0G_1 + \varepsilon^2 G_1^2 \right) \\
    &  \qquad + \Bigg( \frac{1}{p^2(1-p)^2} (J_L - I_L) (J_R - I_R) \left(1+ \frac{2\varepsilon}{p} + \frac{\varepsilon^2}{p^2}\right) \\
    &  \qquad \qquad  + \frac{1}{p(1-p)} \left( I_R(J_L - I_L) + I_L(J_R - I_R) \left(\frac{2}{p} + \frac{\varepsilon}{p^2}\right) \right) + I_LI_R \frac{1}{p^2}\Bigg)\left( G_0 + \varepsilon G_1 \right)^2 \Biggl] \\
    &  + \mathcal{O}.(\varepsilon^2).
\end{align}

Then,
\begin{align}
    &2G_0G_1 I_L I_R +  G_0^2 \left(\frac{1}{p(1-p)} \left( I_R(J_L - I_L) + I_L(J_R - I_R)  \right) + I_L I_R \frac{2}{p} \right) \\
    &=\frac{2 G_1}{G_0}L_0(s) +  G_0^2\left(\frac{1}{p(1-p)} \left( I_RJ_L - I_RI_L + I_LJ_R - I_LI_R  \right) + I_L I_R \frac{2}{p} \right)  && (\text{suppose } G_0 \not = 0) \\
    &= \frac{2 G_1}{G_0}L_0(s) + L_0 \left(\frac{2}{p} -  \frac{2}{p(1-p)}\right) + \frac{G_02}{p(1-p)} (I_RJ_L + I_LJ_R) \\
    &= 2 L_0 \left(\frac{G_1}{G_0} - \frac{1}{1-p} \right) + \frac{G_0^2}{p(1-p)}(I_RJ_L + I_LJ_R) \\
    &=L_1(s).
\end{align}

Thus, we have,
\begin{align}
    &L(s) \approx L_0(s) + \varepsilon L_1(s) + \varepsilon^2 L_2(t) + \mathcal{O}(\varepsilon^2).
\end{align}

\end{proof}

\section{Linear Regression Function: Impurity Gain}

\subsection{Default stump}

\begin{proof}[Another Proof of result from  \cite{ishwaran2015effect} ]
\label{proof:lemma:IG_UNweighted_optimal_treshold}
We denote by $L_0(s)$ the impurity gain the decision tree without applying weighting (None). We want to maximize this impurity gain and derive the optimal split value of our stump with respect to $\Myt$.

\begin{align}
&L(\Myt) \\ &= \mathds{P}(Y=1)(1-\mathds{P}(Y=1))\\  & \qquad - \Myt\mathds{P}_{A_L}(Y=1) (1-\mathds{P}_{A_L}(Y=1)) - (1-\Myt)\mathds{P}_{A_R}(Y=1) (1-\mathds{P}_{A_R}(Y=1)) \\
&= \left(\frac{\alpha}{2}+ \beta\right)\left(1 - \frac{\alpha}{2}-\beta\right) - \Myt\left[\frac{\alpha}{2}t+ \beta\right] \left[1- \frac{\alpha}{2}t- \beta\right] \\
& \qquad - (1-\Myt)\left[\frac{\alpha}{2}(1+\Myt)+ \beta\right] \left[1 - \frac{\alpha}{2}(1+\Myt)- \beta\right] \nonumber \\
&= \frac{\alpha}{2} - \frac{\alpha^2}{4} - \frac{\alpha\beta}{2} + \beta - \frac{\alpha \beta}{2} - \beta^2 -\Myt\left[ \frac{\alpha}{2} \Myt - \frac{\alpha^2}{4}\Myt^2  - \frac{\alpha \beta}{2}\Myt + \beta - \frac{\alpha \beta}{2}\Myt -\beta^2 \right] \\
   & \nonumber \qquad -(1-\Myt) \left[ \frac{\alpha}{2}(1+\Myt) - \left[\frac{\alpha}{2}(1+\Myt)\right]^2 - \frac{\alpha \beta}{2}(1+\Myt) + \beta - \frac{\alpha \beta}{2}(1+\Myt) - \beta^2 \right]   \\
&= \frac{\alpha}{2} - \frac{\alpha^2}{4} - \frac{\alpha\beta}{2} + \beta - \frac{\alpha \beta}{2} - \beta^2 - \frac{\alpha}{2} t^2 + \frac{\alpha^2}{4}\Myt^3  + \frac{\alpha \beta}{2}\Myt^2 - \beta \Myt + \frac{\alpha \beta}{2}t^2 +\beta^2\Myt  \\
    & \nonumber \qquad - \frac{\alpha}{2}(1+\Myt) + \left[\frac{\alpha}{2}(1+\Myt)\right]^2 + \frac{\alpha \beta}{2}(1+\Myt) - \beta + \frac{\alpha \beta}{2}(1+\Myt) + \beta^2    \\
    & \nonumber \qquad  \frac{\alpha}{2}(1+\Myt)\Myt - \left[\frac{\alpha}{2}(1+\Myt)\right]^2\Myt - \frac{\alpha \beta}{2}(1+\Myt)\Myt + \beta \Myt - \frac{\alpha \beta}{2}(1+\Myt)\Myt - \beta^2\Myt    \\
&=  -\frac{\alpha^2}{4} +\frac{\alpha^2}{4} \Myt^3+ \left[\frac{\alpha}{2}(1+t)\right]^2 + \frac{\alpha \beta}{2}\Myt + \frac{\alpha \beta}{2}t - \left[\frac{\alpha}{2}(1+\Myt)\right]^2\Myt - \frac{\alpha \beta}{2}t -\frac{\alpha \beta }{2} \Myt \\
&=  -\frac{\alpha^2}{4} +\frac{\alpha^2}{4} \Myt^3+ \left[\frac{\alpha}{2}(1+t)\right]^2  - \left[\frac{\alpha}{2}(1+\Myt)\right]^2\Myt  \\
&= -\frac{\alpha^2}{4} +\frac{\alpha^2}{4} \Myt^3+ \frac{\alpha^2}{4} + \frac{\alpha^2}{2}\Myt + \frac{\alpha^2}{4} \Myt^2 - \frac{\alpha^2}{4}\Myt - \frac{\alpha^2}{2}\Myt^2 - \frac{\alpha^2}{4} \Myt^3\\
&=  \frac{\alpha^2}{2}\Myt + \frac{\alpha^2}{4} \Myt^2 - \frac{\alpha^2}{4}\Myt - \frac{\alpha^2}{2}\Myt^2\\
&= \Myt^2 \frac{\alpha^2 - 2 \alpha^2}{4} + \Myt \frac{2\alpha^2 -  \alpha^2}{4} \\
&= -\Myt^2 \frac{\alpha^2}{4} + \Myt \frac{\alpha^2}{4}.
\end{align}
$L_0$ is maximized for $\Myt= \frac{- \frac{\alpha^2}{4}}{2 \frac{-\alpha^2}{4}} = \frac{1}{2}$.
\end{proof}

\subsection{Ponderate stump}
\begin{proof}[Proof of \Cref{lemma:IG_weighted_optimal_treshold}]
\label{proof:lemma:IG_weighted_optimal_treshold}
We evaluate how CART optimizes its root split under \cref{assump-1}. 
We have:
\begin{align}
f_{X^*}(x) &= \frac{f_X(x)[\omega_0(1-\mu(x)) + \omega_1\mu(x)]}{p(1-p)}\\
&= \frac{\mathds{1}_{x \in [0,1]}[\omega_0(1-\mu(x)) + \omega_1\mu(x)]}{p(1-p)}\\
&= \frac{\mathds{1}_{x \in [0,1]}[\omega_0 + \mu(x)(\omega_1 - \omega_0)]}{p(1-p)}\\
&= \frac{\mathds{1}_{x \in [0,1]}[p(1-p^{\star}) + \mu(x)(p^{\star} - p)]}{p(1-p)}\\
&=\mathds{1}_{x \in [0,1]}  \left(\frac{p(1-p^{\star})}{p(1-p)} + \frac{\mu(x)(p^{\star} - p)}{p(1-p)} \right)\\
&=\mathds{1}_{x \in [0,1]}  \left(\frac{p(1-p^{\star}) + p^2 -p^2}{p(1-p)} + \frac{\mu(x)(p^{\star} - p)}{p(1-p)} \right)\\
&=\mathds{1}_{x \in [0,1]}  \left(\frac{p(1-p) - p(p^{\star} -p)}{p(1-p)} + \frac{\mu(x)(p^{\star} - p)}{p(1-p)} \right)\\
&=\mathds{1}_{x \in [0,1]} \left(1 + \frac{p^{\star} - p}{p(1-p)} \left( \mu(x) - p \right)\right) \\
&= \mathds{1}_{x \in [0,1]}\left(1 + 2\lambda \left( \mu(x) - p \right)\right), \quad \text{where } \lambda = \frac{p^* - p}{2p(1-p)}.
\end{align}

Let $A_L = \{X_j < \Myt\}$ and $A_R = \{X_j \ge \Myt\}$. Integrating $f_{X^*}(x)$ over these partitions yields the weighted mass components:
\begin{equation}
\mathbb{P}(X^* \in A_L) = \Myt \left[ 1 - \lambda \alpha (1-\Myt) \right], \qquad \mathbb{P}(X^* \in A_R) = (1-\Myt) \left[ 1 + \lambda \alpha \Myt \right].
\end{equation}
The unweighted partition expectations are $\bar{\mu}(A_L) = p - \frac{\alpha(1-\Myt)}{2}$ and $\bar{\mu}(A_R) = p + \frac{\alpha \Myt}{2}$. \Cref{lemma:main-mustar} gives $\hat{\mu}^*(A_L) = g(\bar{\mu}(A_L))$ and $\hat{\mu}^*(A_R) = g(\bar{\mu}(A_R))$ (for fixed leaf $A_L$ and $A_R$), yielding the leaf score gap:
\begin{equation}
\hat{\mu}^*(A_R) - \hat{\mu}^*(A_L) = \frac{\omega_1 \left( \frac{1}{2} - \lambda p \right) \alpha}{p(1-p) \left[ 1 - \lambda \alpha (1-\Myt) \right] \left[ 1 + \lambda \alpha \Myt \right]}.
\end{equation}
Substituting $\mathbb{P}(X^* \in A_L)$, $\mathbb{P}(X^* \in A_R)$, and $[\hat{\mu}^*(A_R) - \hat{\mu}^*(A_L)]^2$ into the Gini impurity gain equation $L(\Myt) = \mathbb{P}(X^* \in A_L) \mathbb{P}(X^* \in A_R) [\hat{\mu}^*(A_R) - \hat{\mu}^*(A_L)]^2$,  we obtain for feature $j$ over the full unit domain:
\begin{equation}
\label{eq:collapsed_gain}
L(s) = C \cdot \frac{  t(1-\Myt)}{\left[ 1 - \lambda \alpha (1-\Myt) \right] \cdot \left[ 1 + \lambda \alpha \Myt \right]},
\end{equation}
where $C = \alpha^2\left[ \frac{\omega_1}{p(1-p)} \left( \frac{1}{2} - \lambda p \right) \right]^2$. Let $\kappa = \frac{p^* - p}{2p(1-p)} \alpha$

Next, we differentiate $L(\Myt)$ directly with respect to $\Myt$ using the quotient rule.

Let $L(\Myt) = C  \frac{N(\Myt)}{D(\Myt)}$, where:
\begin{align}
N(t) &= \Myt(1-\Myt) = \Myt - \Myt^2, \label{eq:N_def} \\
D(t) &= \left[ 1 - \kappa(1-\Myt) \right] \left[ 1 + \kappa \Myt \right] = (1 - \kappa + \kappa \Myt)(1 + \kappa \Myt). \label{eq:D_def}
\end{align}

Expanding $D(\Myt)$ explicitly:
\begin{equation}\label{eq:D_expanded}
D(\Myt) = 1 - \kappa + \kappa \Myt + \kappa t - \kappa^2 \Myt + \kappa^2 \Myt^2 = (1 - \kappa) + \kappa(2 - \kappa)\Myt + \kappa^2 t^2.
\end{equation}

Differentiating $N(\Myt)$ and $D(\Myt)$ directly with respect to $t$:
\begin{align}
N'(\Myt) &= 1 - 2\Myt, \label{eq:N_prime} \\
D'(t) &= \kappa(2 - \kappa) + 2\kappa^2 \Myt = \kappa(2 - \kappa + 2\kappa \Myt). \label{eq:D_prime}
\end{align}

By the quotient rule, setting $\frac{dL^*(\Myt)}{d\Myt} = 0$ requires setting the numerator of the derivative to zero:
\begin{equation}\label{eq:quotient_num_zero}
N'(\Myt) D(\Myt) - N(\Myt) D'(\Myt) = 0.
\end{equation}

We evaluate $N'(\Myt) D(\Myt)$ and $N(\Myt) D'(\Myt)$ independently:

%\subparagraph{Evaluation of $N'(\Myt) D(\Myt)$:}
We have,
\begin{align}
N'(\Myt) D(\Myt) &= (1 - 2\Myt) \left[ (1 - \kappa) + \kappa(2 - \kappa)\Myt + \kappa^2 \Myt^2 \right] \nonumber \\
&= (1 - \kappa) + \kappa(2 - \kappa)\Myt + \kappa^2 \Myt^2 - 2\Myt(1 - \kappa) - 2\kappa(2 - \kappa)\Myt^2 - 2\kappa^2 \Myt^3 \nonumber \\
&= (1 - \kappa) + \left[ \kappa(2 - \kappa) - 2(1 - \kappa) \right] \Myt + \left[ \kappa^2 - 2\kappa(2 - \kappa) \right] \Myt^2 - 2\kappa^2 \Myt^3 \nonumber \\
&= (1 - \kappa) + (4\kappa - \kappa^2 - 2)\Myt + (3\kappa^2 - 4\kappa)\Myt^2 - 2\kappa^2 \Myt^3. \label{eq:ND_product}
\end{align}

%\subparagraph{Evaluation of $N(\Myt) D'(\Myt)$:}
Then we have,
\begin{align}
N(\Myt) D'(\Myt) &= (\Myt - \Myt^2) \kappa(2 - \kappa + 2\kappa \Myt) = (\Myt - \Myt^2)(2\kappa - \kappa^2) + (\Myt - \Myt^2)(2\kappa^2 \Myt) \nonumber \\
&= (2\kappa - \kappa^2)\Myt - (2\kappa - \kappa^2)\Myt^2 + 2\kappa^2 \Myt^2 - 2\kappa^2 \Myt^3 \nonumber \\
&= (2\kappa - \kappa^2)\Myt + (3\kappa^2 - 2\kappa)\Myt^2 - 2\kappa^2 \Myt^3. \label{eq:ND_prime_product}
\end{align}

Subtracting Equation~\eqref{eq:ND_prime_product} from Equation~\eqref{eq:ND_product}:
\begin{align}\label{eq:diff_expanded}
N'(\Myt)D(\Myt) - N(\Myt)D'(\Myt) &= (1 - \kappa) + \left[ (4\kappa - \kappa^2 - 2) - (2\kappa - \kappa^2) \right] \Myt \nonumber \\
&\quad + \left[ (3\kappa^2 - 4\kappa) - (3\kappa^2 - 2\kappa) \right] \Myt^2 + \left[ -2\kappa^2 - (-2\kappa^2) \right] \Myt^3 \nonumber \\
&= (1 - \kappa) + (2\kappa - 2)\Myt - 2\kappa \Myt^2.
\end{align}

Notice that the cubic $\Myt^3$ terms cancel out completely! Setting Equation~\eqref{eq:diff_expanded} to zero and multiplying by $-1$:
\begin{equation}\label{eq:quadratic_t_final}
2\kappa \Myt^2 + 2(1 - \kappa)\Myt - (1 - \kappa) = 0.
\end{equation}

Applying the standard quadratic formula $A \Myt^2 + B \Myt + C = 0$ with coefficients $A = 2\kappa$, $B = 2(1-\kappa)$, and $C = -(1-\kappa)$:
\begin{equation}
\Myt = \frac{-B \pm \sqrt{B^2 - 4AC}}{2A}.
\end{equation}

Evaluating the discriminant $\Delta = B^2 - 4AC$:
\begin{align}
\Delta &= [2(1-\kappa)]^2 - 4(2\kappa)[-(1-\kappa)] \nonumber \\
&= 4(1-\kappa)^2 + 8\kappa(1-\kappa) \nonumber \\
&= 4(1-\kappa) \left[ (1-\kappa) + 2\kappa \right] = 4(1-\kappa)(1+\kappa) = 4(1 - \kappa^2).
\end{align}

Substituting $\sqrt{\Delta} = 2\sqrt{1 - \kappa^2}$ back into the quadratic formula:
\begin{equation}
t = \frac{-2(1-\kappa) \pm 2\sqrt{1 - \kappa^2}}{2(2\kappa)} = \frac{\kappa - 1 \pm \sqrt{1 - \kappa^2}}{2\kappa}.
\end{equation}

To determine the root belonging to $\Myt_{p}^{\star} \in (0,1)$, we evaluate the limit as $\kappa \to 0$ (unweighted baseline):
\begin{itemize}
    \item Selecting the positive root $+\sqrt{1 - \kappa^2}$:
    \begin{align}
    \lim_{\kappa \to 0} \frac{\kappa - 1 + \sqrt{1 - \kappa^2}}{2\kappa} &= \lim_{\kappa \to 0} \left( \frac{1}{2} + \frac{\sqrt{1 - \kappa^2} - 1}{2\kappa} \right) \nonumber \\
    &\stackrel{\text{L'H\^opital}}{=} \frac{1}{2} + \lim_{\kappa \to 0} \frac{\frac{-2\kappa}{2\sqrt{1-\kappa^2}}}{2} = \frac{1}{2} - 0 = \frac{1}{2},
    \end{align}
    which correctly recovers the unweighted midpoint split $t_p^{\star} = \frac{1}{2}$.
    \item Selecting the negative root $-\sqrt{1 - \kappa^2}$:
    \begin{equation}
    \lim_{\kappa \to 0} \frac{\kappa - 1 - \sqrt{1 - \kappa^2}}{2\kappa} = \lim_{\kappa \to 0} \frac{-2}{2\kappa} = \pm \infty,
    \end{equation}
    which diverges outside the domain $(0,1)$.
\end{itemize}

Therefore, the unique plausible split value is:
\begin{equation}
\Myt_{p}^{\star} = \frac{\kappa - 1 + \sqrt{1 - \kappa^2}}{2\kappa} = \frac{1}{2} + \frac{\sqrt{1 - \kappa^2} - 1}{2\kappa}.
\end{equation}
\end{proof}

\section{Univariate Linear Regression Function: Proofs of Derivation Metrics}

\subsection{Default stump}
\label{proof:default-stump}
\begin{proof}[Proof of \Cref{lemma:default-stump}]
Grant \Cref{assump-infinite-sample} and \Cref{assump-1}.
Let $\Myt$ be the split value of our stump $\hat{m}_{k}$. Let $t \in [0,1]$, be the threshold for our stump to predict $\hat{Y}=1$ for each cell. More precisely,
\begin{align}
    \mathds{E}[Y|X \in A_{R}]  \geq t \Longrightarrow \hat{Y} = 1 \\
    \mathds{E}[Y|X \in A_{L}] \geq t \Longrightarrow \hat{Y} = 1 
\end{align}

We have,
\begin{align}
    &\mathds{E}[Y|X \in A_{R}] \\&= \mathds{E}[\mathds{E}[Y|X]|X \in A_{R}] && (\text{Tower property}) \\
    & = \mathds{E}[P(Y=1|X)| X \in A_{R}] \\
    & = \mathds{E}[\alpha X + \beta| X \in A_{R}] \\
    & = \alpha \mathds{E}[ X| X \in A_{R}]+ \beta   \\
    &= \alpha \left(\frac{1+\Myt}{2}\right) + \beta. && ( \sim \mathcal{U}([\Myt,1]))
\end{align}

Similarly,
\begin{align}
    &\mathds{E}[Y|X \in A_{L}] \\&= \mathds{E}[\mathds{E}[Y|X]|X \in A_{L}] && (\text{Tower property}) \\
    & = \mathds{E}[P(Y=1|X)| X \in A_{L}] \\
    & = \mathds{E}[\alpha X + \beta| X \in A_{L}] \\
    & = \alpha \mathds{E}[ X| X \in A_{L}]+ \beta   \\
    &= \frac{\alpha \Myt}{2} + \beta. && ( \sim \mathcal{U}([0,\Myt]))
\end{align}

Then,
\begin{align}
    x \leq \Myt \Longrightarrow x \in A_{L} \Longrightarrow \hat{m}_{k}(x)=1 \; \text{if} \; \frac{\alpha \Myt}{2} + \beta \geq t.
\end{align}
On the other hand,
\begin{align}
    x > \Myt \Longrightarrow x \in A_{R} \Longrightarrow \hat{m}_{k}(x)=1 \; \text{if} \; \alpha \left(\frac{1+\Myt}{2}\right) + \beta \geq t.
\end{align}

Thus,
\begin{align}
    \hat{m}_{k}(x) = \left\{ 
        \begin{array}{ll}
            1 & \text{if} \; x \leq \Myt \; \text{and} \; \frac{\alpha \Myt}{2} + \beta \geq t \\
            1 & \text{if} \; x > \Myt \; \text{and} \; \alpha \left(\frac{1+\Myt}{2}\right) + \beta \geq t \; \\
            0 & \text{otherwise}.
        \end{array}
    \right.
\end{align}

\end{proof}

\subsection{Default metrics}
\label{subsec-proofs-metrics-default-stump}
\begin{proof}[Proof of \Cref{prop:main1-metrics}]
\label{proof:main1-metrics}
Similarly to \Cref{proof:default-stump} we have,
\begin{align}
    \hat{m}_{k}(x) = \left\{ 
        \begin{array}{ll}
            1 & \text{if} \; t \in [0, \frac{\alpha \Myt}{2} + \beta] \\
            \mathds{1}_{\{x \geq \Myt\}} & \text{if} \; t \in [ \frac{\alpha \Myt}{2} + \beta, \frac{\alpha (\Myt+1)}{2} + \beta] \\
            0 & \text{if} \; t \in [ \frac{\alpha (\Myt+1)}{2} + \beta, 1].
        \end{array}
    \right.
\end{align}

\subsubsection{Precision}

\begin{definition}
We define the precision as below,
\begin{align}
     Precision = \frac{TP}{TP + FP}.
\end{align}
\end{definition}

We aim at computing the precision in our framework. We have,
\begin{align}
    Precision &= \frac{\mathds{P}(Y=1, \hat{Y}=1)}{\mathds{P}(Y=1, \hat{Y}=1) + \mathds{P}(Y=0, \hat{Y}=1)}.
\end{align}

We have, 
\begin{align}
    \mathds{P}(Y=1, \hat{Y}=1) &= \mathds{E}\left[\mathds{1}_{\{Y=1\}}\mathds{1}_{\{\hat{m}_{k}(X)=1\}}\right] \\
    &=\mathds{E}\left[\mathds{E}\left[\mathds{1}_{\{Y=1\}}\mathds{1}_{\{\hat{m}_{k}(X)=1\}}|X\right] \right]  && (\text{Tower property})\\
    &=\mathds{E}\left[\mathds{1}_{\{\hat{m}_{k}(X)=1\}}\mathds{E}\left[\mathds{1}_{\{Y=1\}}|X\right] \right] \\
    &=\mathds{E}\left[\mathds{1}_{\{\hat{m}_{k}(X)=1\}}\mathds{P}\left(Y=1|X\right) \right] \\
    &=\mathds{E}\left[\mathds{1}_{\{\hat{m}_{k}(X)=1\}}\left(\alpha X + \beta\right) \right] \\
    &=\alpha\mathds{E}\left[\mathds{1}_{\{\hat{m}_{k}(X)=1\}} X  \right] + \beta \mathds{P}(\hat{m}_{k}(x)=1).
\end{align}

Similarly, 
\begin{align}
    \mathds{P}(Y=0, \hat{Y}=1) &= \mathds{E}\left[\mathds{1}_{\{Y=0\}}\mathds{1}_{\{\hat{m}_{k}(X)=1\}}\right] \\
    &=\mathds{E}\left[\mathds{E}\left[\mathds{1}_{\{Y=0\}}\mathds{1}_{\{\hat{m}_{k}(X)=1\}}|X\right] \right]  && (\text{Tower property})\\
    &=\mathds{E}\left[\mathds{1}_{\{\hat{m}_{k}(X)=1\}}\mathds{E}\left[\mathds{1}_{\{Y=0\}}|X\right] \right] \\
    &=\mathds{E}\left[\mathds{1}_{\{\hat{m}_{k}(X)=1\}}\mathds{P}\left(Y=0|X\right) \right] \\
    &=\mathds{E}\left[\mathds{1}_{\{\hat{m}_{k}(X)=1\}}\left(1-\mathds{P}\left(Y=1|X\right)\right) \right] \\
    &=\mathds{E}\left[\mathds{1}_{\{\hat{m}_{k}(X)=1\}}\left(1 -\alpha X - \beta\right) \right] \\
    &=\mathds{P}(\hat{m}_{k}(X)=1) - \alpha\mathds{E}\left[\mathds{1}_{\{\hat{m}_{k}(X)=1\}} X  \right] - \beta \mathds{P}(\hat{m}_{k}(X)=1).
\end{align}

\paragraph{Case :} $t \in [0, \frac{\alpha \Myt}{2} + \beta]$

Then, $\hat{m}_{k}(X)= 1$ for all $x \in [0,1]$. Thus,
\begin{align}
    \mathds{P}(Y=1, \hat{Y}=1) = \alpha / 2 + \beta \\ 
    \mathds{P}(Y=0, \hat{Y}=1) = 1 - \alpha / 2 - \beta.
\end{align}
Therefore,
\begin{align}
    Precision = \frac{\alpha / 2 + \beta}{\alpha / 2 + \beta + 1 - \alpha / 2 - \beta} = \alpha / 2 + \beta.
\end{align}

\paragraph{Case :} $t \in [ \frac{\alpha \Myt}{2} + \beta, \frac{\alpha (\Myt+1)}{2} + \beta]$

Then, $\hat{m}_{k}(X)=\mathds{1}_{\{X\geq t\}}$  for all $x \in [0,1]$.
Thus,
\begin{align}
     \mathds{P}(Y=1, \hat{Y}=1) &= \alpha\mathds{E}\left[\mathds{1}_{\{\hat{m}_{k}(X)=1\}} X  \right] + \beta \mathds{P}(\hat{m}_{k}(X)=1) \\
     &= \alpha\mathds{E}\left[\mathds{1}_{\{X\geq \Myt\}} X  \right] + \beta \mathds{P}(X\geq \Myt) \\
     &=\frac{\alpha}{2}(1-\Myt^{2}) + \beta (1-\Myt).  \; \;   \left(\mathds{E}\left[\mathds{1}_{\{X \geq t\}} X  \right] = \int_{\Myt}^{1}x\mathrm{d}x= \frac{1}{2}(1-\Myt^{2}) \right)
\end{align}

Besides,
\begin{align}
    &\mathds{P}(Y=0, \hat{Y}=1) \\
    &=\mathds{P}(\hat{m}_{k}(X)=1) - \alpha\mathds{E}\left[\mathds{1}_{\{\hat{m}_{k}(X)=1\}} X  \right] - \beta \mathds{P}(\hat{m}_{k}(X)=1) \\
    &=\mathds{P}(X \geq t) - \alpha\mathds{E}\left[\mathds{1}_{\{X \geq \Myt\}} X  \right] - \beta \mathds{P}(X \geq \Myt) \\
    &= (1-\Myt) -\alpha\frac{1}{2}(1-\Myt^{2}) - \beta (1-\Myt)  \; \;  \left(\mathds{E}\left[\mathds{1}_{\{X \leq \Myt\}} X  \right] = \int_{\Myt}^{1}x\mathrm{d}x= \frac{1}{2}(1-\Myt^{2}) \right) \\
    &= \frac{\alpha}{2}\Myt^{2} + \beta \Myt + 1-\Myt -\beta - \frac{\alpha}{2}.
\end{align}

Therefore,
\begin{align}
    Precision &= \frac{\frac{\alpha}{2}(1-\Myt^{2}) + \beta (1-\Myt)}{\frac{\alpha}{2}(1-\Myt^{2}) + \beta (1-\Myt) + \frac{\alpha}{2}\Myt^{2} + \beta \Myt + 1-\Myt -\beta - \frac{\alpha}{2} } \\
    &= \frac{\frac{\alpha}{2}(1-t^{2}) + \beta (1-\Myt)}{ 1-\Myt } \\
    &= \frac{\alpha}{2} (1+\Myt) + \beta.
\end{align}

\paragraph{Case :} $t \in [ \frac{\alpha (\Myt+1)}{2} + \beta, 1]$
Then, $\hat{m}_{k}(X)= 0$ for all $x \in [0,1]$. Thus,
\begin{align}
    \mathds{P}(Y=1, \hat{Y}=1) = 0 \\ 
    \mathds{P}(Y=0, \hat{Y}=1) = 0.
\end{align}
Therefore,
\begin{align}
    Precision = \frac{0}{0} = ?.
\end{align}

Finally,
\begin{align}
    Precision = \left\{ 
        \begin{array}{ll}
             \alpha/2 + \beta& \text{if} \; t \in [0, \frac{\alpha \Myt}{2} + \beta] \\
            \frac{\alpha}{2} (1+\Myt) + \beta & \text{if} \; t \in [ \frac{\alpha \Myt}{2} + \beta, \frac{\alpha (\Myt+1)}{2} + \beta] \\
            Undefined  & \text{if} \; t \in [ \frac{\alpha (\Myt+1)}{2} + \beta, 1].
        \end{array}
    \right.
\end{align}

\subsubsection{Recall}
\begin{definition}
We define the Recall as below,
\begin{align}
     Recall = \frac{TP}{TP + FN}.
\end{align}
\end{definition}

We aim at computing the precision in our framework. We have,
\begin{align}
    Recall &= \frac{\mathds{P}(Y=1, \hat{Y}=1)}{\mathds{P}(Y=1, \hat{Y}=1) + \mathds{P}(Y=1, \hat{Y}=0)}.
\end{align}

We have, 
\begin{align}
    \mathds{P}(Y=1, \hat{Y}=0) &= \mathds{E}\left[\mathds{1}_{\{Y=1\}}\mathds{1}_{\{\hat{m}_{k}(X)=0\}}\right] \\
    &=\mathds{E}\left[\mathds{E}\left[\mathds{1}_{\{Y=1\}}\mathds{1}_{\{\hat{m}_{k}(X)=0\}}|X\right] \right]  && (\text{Tower property})\\
    &=\mathds{E}\left[\mathds{1}_{\{\hat{m}_{k}(X)=0\}}\mathds{E}\left[\mathds{1}_{\{Y=1\}}|X\right] \right] \\
    &=\mathds{E}\left[\mathds{1}_{\{\hat{m}_{k}(X)=0\}}\mathds{P}\left(Y=1|X\right) \right] \\
    &=\mathds{E}\left[\mathds{1}_{\{\hat{m}_{k}(X)=0\}}\left(\alpha X + \beta\right) \right] \\
    &=\alpha\mathds{E}\left[\mathds{1}_{\{\hat{m}_{k}(X)=0\}} X  \right] + \beta \mathds{P}(\hat{m}_{k}(X)=0).
\end{align}

\paragraph{Case :} $t \in [0, \frac{\alpha \Myt}{2} + \beta]$
Then, $\hat{m}_{k}(X)= 1$ for all $X \in [0,1]$. Thus,
\begin{align}
    \mathds{P}(Y=1, \hat{Y}=1)& = \alpha/2 + \beta \\ 
    \mathds{P}(Y=1, \hat{Y}=0) &= \alpha\mathds{E}\left[0* X  \right] + \beta *0 = 0.
\end{align}
Therefore,
\begin{align}
    Recall = \frac{\alpha/2 + \beta}{\alpha/2 + \beta} = 1.
\end{align}

\paragraph{Case :} $t \in [ \frac{\alpha \Myt}{2} + \beta, \frac{\alpha (\Myt+1)}{2} + \beta]$

Then, $\hat{m}_{k}(X)=\mathds{1}_{\{X\geq \Myt\}}$  for all $X \in [0,1]$.
\begin{align}
    \mathds{P}(Y=1, \hat{Y}=1) 
    &= \alpha\mathds{E}\left[\mathds{1}_{\{X\geq \Myt\}} X  \right] + \beta \mathds{P}(\hat{m}_{k}(X)=1)
    = \frac{\alpha}{2}(1-\Myt^{2}) + \beta (1-\Myt)\\ 
    \mathds{P}(Y=1, \hat{Y}=0) &= \alpha\mathds{E}\left[\mathds{1}_{\{X < \Myt\}} X  \right] + \beta \mathds{P}(\hat{m}_{k}(X)=0) =  \frac{\alpha}{2}\Myt^{2} + \beta t.
\end{align}
Therefore,
\begin{align}
    Recall &= \frac{\frac{\alpha}{2}(1-\Myt^{2}) + \beta (1-\Myt)}{ \frac{\alpha}{2}(1-\Myt^{2}) + \beta (1-\Myt) + \frac{\alpha}{2}\Myt^{2} + \beta \Myt} \\
    &= \frac{\frac{\alpha}{2}(1-\Myt^{2}) + \beta (1-\Myt)}{ \frac{\alpha}{2} + \beta}. 
\end{align}

\paragraph{Case :} $t \in [ \frac{\alpha (\Myt+1)}{2} + \beta, 1]$
Then, $\hat{m}_{k}(X)= 0$ for all $X \in [0,1]$. Thus,
\begin{align}
    \mathds{P}(Y=1, \hat{Y}=1) &= 0 \\ 
    \mathds{P}(Y=1, \hat{Y}=0) &= \alpha\mathds{E}\left[ X  \right] + \beta =  \alpha/2 + \beta.
 \end{align}
Therefore,
\begin{align}
    Recall = \frac{0}{0+\alpha/2 + \beta} = 0.
\end{align}

Finally,
\begin{align}
    Recall = \left\{ 
        \begin{array}{ll}
            1 & \text{if} \; t \in [0, \frac{\alpha \Myt}{2} + \beta] \\
            \frac{\frac{\alpha}{2}(1-\Myt^{2}) + \beta (1-\Myt)}{ \frac{\alpha}{2} + \beta} & \text{if} \; t \in [ \frac{\alpha \Myt}{2} + \beta, \frac{\alpha (\Myt+1)}{2} + \beta] \\
            0 & \text{if} \; t \in [ \frac{\alpha (\Myt+1)}{2} + \beta, 1].
        \end{array}
    \right.
\end{align}

\subsubsection{False Positive Rate (FPR)}
\begin{definition}
We define the Recall as below,
\begin{align}
     FPR = \frac{FP}{FP + TN}.
\end{align}
\end{definition}

We aim at computing the precision in our framework. We have,
\begin{align}
    FPR &= \frac{\mathds{P}(Y=0, \hat{Y}=1)}{\mathds{P}(Y=0, \hat{Y}=1) + \mathds{P}(Y=0, \hat{Y}=0)}.
\end{align}

We have,
\begin{align}
    \mathds{P}(Y=0, \hat{Y}=0) &= \mathds{E}\left[\mathds{1}_{\{Y=0\}}\mathds{1}_{\{\hat{m}_{k}(X)=0\}}\right] \\&=\mathds{E}\left[\mathds{E}\left[\mathds{1}_{\{Y=0\}}\mathds{1}_{\{\hat{m}_{k}(X)=0\}}|X\right] \right]  && (\text{Tower property})\\
    &=\mathds{E}\left[\mathds{1}_{\{\hat{m}_{k}(X)=0\}}\mathds{E}\left[\mathds{1}_{\{Y=0\}}|X\right] \right] \\
    &=\mathds{E}\left[\mathds{1}_{\{\hat{m}_{k}(X)=0\}}\mathds{P}\left(Y=0|X\right) \right] \\
    &=\mathds{E}\left[\mathds{1}_{\{\hat{m}_{k}(X)=0\}}\left(1 -\alpha X - \beta\right) \right] \\
    &=\mathds{P}(\hat{m}_{k}(X)=0) -\alpha\mathds{E}\left[\mathds{1}_{\{\hat{m}_{k}(X)=0\}} X  \right] + \beta \mathds{P}(\hat{m}_{k}(X)=0).
\end{align}

\paragraph{Case :} $t \in [0, \frac{\alpha \Myt}{2} + \beta]$
Then, $\hat{m}_{k}(X)= 1$ for all $X \in [0,1]$. Thus,
\begin{align}
    \mathds{P}(Y=0, \hat{Y}=1) = 1 - \alpha / 2 - \beta.
\end{align}

Besides,
\begin{align}
    \mathds{P}(Y=0, \hat{Y}=0) = 0.
\end{align}

Then,
\begin{align}
    FPR = \frac{1 - \alpha / 2 - \beta}{1 - \alpha / 2 - \beta +0  } = 1
\end{align}

\paragraph{Case :} $t \in [ \frac{\alpha \Myt}{2} + \beta, \frac{\alpha (\Myt+1)}{2} + \beta]$
Then, $\hat{m}_{k}(X)=\mathds{1}_{\{X\geq \Myt\}}$  for all $X \in [0,1]$.
Thus,
\begin{align}
    \mathds{P}(Y=0, \hat{Y}=1) &= \frac{\alpha}{2}t^{2} + \beta t + 1-\Myt -\beta - \frac{\alpha}{2}.
\end{align}
Furthermore,
\begin{align}
    &\mathds{P}(Y=0, \hat{Y}=0) \\
    &= \mathds{P}(\hat{m}_{k}(X)=0) -\alpha\mathds{E}\left[\mathds{1}_{\{\hat{m}_{k}(X)=0\}} X  \right] + \beta \mathds{P}(\hat{m}_{k}(X)=0) \\
    &= \mathds{P}(X< \Myt) -\alpha\mathds{E}\left[\mathds{1}_{\{X< t\}} X  \right] + \beta \mathds{P}(X< \Myt) \\
    &= \Myt -\alpha\mathds{E}\left[\mathds{1}_{\{X< \Myt\}} X  \right] + \beta t\\
    &= \Myt -\alpha\frac{1}{2} \Myt^{2} - \beta \Myt &&  \left(\mathds{E}\left[\mathds{1}_{\{X < \Myt\}} X  \right] = \int_{0}^{t}x\mathrm{d}x= \frac{1}{2}\Myt^{2}\right) \\
    &= -\frac{\alpha}{2} \Myt^{2} + \Myt (1 -\beta).
\end{align}

Then,
\begin{align}
    FPR &= \frac{-\frac{\alpha}{2} \Myt^{2} + \Myt (1 -\beta)}{-\frac{\alpha}{2} t^{2} + \Myt (1 -\beta)+ \frac{\alpha}{2}\Myt^{2} + \beta \Myt + 1-\Myt -\beta - \frac{\alpha}{2}} \\
    &= \frac{-\frac{\alpha}{2} \Myt^{2} + \Myt (1 -\beta)}{  1 -\beta - \frac{\alpha}{2}}.
\end{align}

\paragraph{Case :} $t \in [ \frac{\alpha (\Myt+1)}{2} + \beta, 1]$
Then, $\hat{m}_{k}(X)= 0$ for all $X \in [0,1]$. Thus,
\begin{align}
    \mathds{P}(Y=0, \hat{Y}=1) = 0
\end{align}
Furthermore,
\begin{align}
 \mathds{P}(Y=0, \hat{Y}=0) = 1 - \frac{\alpha}{2} + \beta
\end{align}

Then,
\begin{align}
    FPR = \frac{0}{0 + 1 - \frac{\alpha}{2} + \beta  } = 0
\end{align}

Finally,
\begin{align}
    FPR = \left\{ 
        \begin{array}{ll}
            1 & \text{if} \; t \in [0, \frac{\alpha \Myt}{2} + \beta] \\
            \frac{-\frac{\alpha}{2} \Myt^{2} + \Myt (1 -\beta)}{  1 -\beta - \frac{\alpha}{2}} & \text{if} \; t \in [ \frac{\alpha \Myt}{2} + \beta, \frac{\alpha (\Myt+1)}{2} + \beta] \\
            0 & \text{if} \; t \in [ \frac{\alpha (\Myt+1)}{2} + \beta, 1].
        \end{array}
    \right.
\end{align}

\subsubsection{F1-score}

We have,
\begin{align}
    \text{F1-score} = \frac{Precision \times Recall}{Precision + Recall}
\end{align}

\paragraph{Case :} $t \in [ \frac{\alpha \Myt}{2} + \beta, \frac{\alpha (\Myt+1)}{2} + \beta]$

\begin{align}
    \text{F1-score} &= \frac{\left(\frac{\alpha}{2} (1+\Myt) + \beta\right) \left(\frac{\frac{\alpha}{2}(1-\Myt^{2}) + \beta (1-\Myt)}{ \frac{\alpha}{2} + \beta}\right) }{\frac{\alpha}{2} (1+\Myt) + \beta + \frac{\frac{\alpha}{2}(1-\Myt^{2}) + \beta (1-\Myt)}{ \frac{\alpha}{2} + \beta}} \\
    &= \frac{\left(\frac{\alpha}{2} (1+\Myt) + \beta\right) \left(\frac{\frac{\alpha}{2}(1-\Myt^{2}) + \beta (1-\Myt)}{ \frac{\alpha}{2} + \beta}\right) }{ \frac{(\frac{\alpha}{2} + \beta)(\frac{\alpha}{2} (1+\Myt) + \beta )  +\frac{\alpha}{2}(1-\Myt^{2}) + \beta (1-\Myt)}{ \frac{\alpha}{2} + \beta}} \\
    &= \frac{\left(\frac{\alpha}{2} (1+\Myt) + \beta\right) \left(\frac{\alpha}{2}(1-\Myt^{2}) + \beta (1-\Myt)\right) }{ (\frac{\alpha}{2} + \beta)(\frac{\alpha}{2} (1+\Myt) + \beta )  +\frac{\alpha}{2}(1-\Myt^{2}) + \beta (1-\Myt)} \\
    &= \frac{\left(\frac{\alpha}{2} (1+t) + \beta\right) \left(\frac{\alpha}{2}(1+t) + \beta\right)(1-\Myt) }{ (\frac{\alpha}{2} + \beta)(\frac{\alpha}{2} (1+\Myt) + \beta ) +  \left(\frac{\alpha}{2}(1+\Myt) + \beta\right)(1-\Myt)} \\
    &= \frac{ \left(\frac{\alpha}{2}(1+\Myt) + \beta\right)(1-\Myt) }{ (\frac{\alpha}{2} + \beta)  + (1-\Myt)}.
\end{align}

Finally,
\begin{align}
    \text{F1-score} = \left\{ 
        \begin{array}{ll}
             \frac{\alpha/2 + \beta}{1+ \alpha/2 + \beta}& \text{if} \; t \in [0, \frac{\alpha \Myt}{2} + \beta] \\
            \frac{ \left(\frac{\alpha}{2}(1+\Myt) + \beta\right)(1-\Myt) }{ (\frac{\alpha}{2} + \beta)  + (1-\Myt)} & \text{if} \; t \in [ \frac{\alpha \Myt}{2} + \beta, \frac{\alpha (\Myt+1)}{2} + \beta] \\
            Undefined  & \text{if} \; t \in [ \frac{\alpha (\Myt+1)}{2} + \beta, 1].
        \end{array}
    \right.
\end{align}

\subsubsection{PR AUC}
We have

\begin{align}
    &\text{PR AUC} \\
    &= Precision_1(Recall_1-Recall_0) + Precision_2(Recall_2 - Recall_1) \\
    &= \left(\frac{\alpha}{2} (1+\Myt) + \beta\right)\left(\frac{\frac{\alpha}{2}(1-\Myt^{2}) + \beta (1-\Myt)}{ \frac{\alpha}{2} + \beta}\right) + (\alpha/2 + \beta) \left(1-\left(\frac{\frac{\alpha}{2}(1-\Myt^{2}) + \beta (1-\Myt)}{ \frac{\alpha}{2} + \beta}\right)\right) \\
    &= \frac{\left(\frac{\alpha}{2} (1+\Myt) + \beta\right) \left(\frac{\alpha}{2}(1-\Myt^{2}) + \beta (1-\Myt)\right)}{ \frac{\alpha}{2} + \beta} + \frac{(\alpha/2 + \beta)^2 - (\alpha/2 + \beta)\left(\frac{\alpha}{2}(1-\Myt^{2}) + \beta (1-t)\right) }{ \frac{\alpha}{2} + \beta} \\
    &= \frac{\left(\frac{\alpha}{2}(1-\Myt^{2}) + \beta (1-\Myt)\right) \left(\left(\frac{\alpha}{2} (1+\Myt) + \beta\right) - \left(\frac{\alpha}{2}+ \beta \right)\right) }{\frac{\alpha}{2} + \beta} + (\alpha/2 + \beta) \\
    &= \frac{\left(\frac{\alpha}{2}(1-\Myt^{2}) + \beta (1-\Myt)\right) \left(\frac{\alpha}{2}\Myt \right) }{\frac{\alpha}{2} + \beta} + (\alpha/2 + \beta) \\
    &= \frac{\left(\frac{\alpha^2}{4}t \right)- \left(\frac{\alpha^2}{4}\Myt^3 \right) + \left(\frac{\alpha \beta }{2}\Myt \right)   - \left(\frac{\alpha \beta }{2}\Myt^2 \right) }{\frac{\alpha}{2} + \beta} + \alpha/2 + \beta.
\end{align}

\subsubsection{ROC AUC}
We have,
\begin{align}
    ROC AUC &= Recall_1(FPR_1 - FPR_0) + Recall_2(FPR_2 - FPR_1)\\
    &= \frac{\frac{\alpha}{2}(1-\Myt^{2}) + \beta (1-\Myt)}{ \frac{\alpha}{2} + \beta} \left(\frac{-\frac{\alpha}{2} \Myt^{2} + \Myt (1 -\beta)}{  1 -\beta - \frac{\alpha}{2}} \right) + \left(1- \frac{-\frac{\alpha}{2} \Myt^{2} + \Myt (1 -\beta)}{  1 -\beta - \frac{\alpha}{2}}\right) \\
    &=  \left(\frac{-\frac{\alpha}{2} \Myt^{2} + t (1 -\beta)}{  1 -\beta - \frac{\alpha}{2}}\right) \left(\frac{\frac{\alpha}{2}(1-\Myt^{2}) + \beta (1-t)}{ \frac{\alpha}{2} + \beta} -1\right) + 1 \\
    &=  \left(\frac{-\frac{\alpha}{2} \Myt^{2} + t (1 -\beta)}{  1 -\beta - \frac{\alpha}{2}}\right) \left(\frac{\frac{\alpha}{2}(1-\Myt^{2}) + \beta (1-t) - \frac{\alpha}{2} - \beta}{ \frac{\alpha}{2} + \beta} \right) + 1 \\
    &=  \left(\frac{-\frac{\alpha}{2} \Myt^{2} + \Myt  -\beta \Myt}{  1 -\beta - \frac{\alpha}{2}}\right) \left(\frac{\frac{-\alpha}{2}\Myt^{2} - \beta t }{ \frac{\alpha}{2} + \beta} \right) + 1 \\
    &=  \frac{ \left(-\frac{\alpha}{2} \Myt^{2} + \Myt  -\beta \Myt \right)   \left(\frac{-\alpha}{2}\Myt^{2} - \beta \Myt  \right)}{  \left(1 -\beta - \frac{\alpha}{2}\right) \left(\frac{\alpha}{2} + \beta\right)}   + 1 \\
    &=  \frac{ \Myt\left(-\frac{\alpha}{2} \Myt + 1  -\beta \right)   \Myt\left(\frac{-\alpha}{2}\Myt - \beta  \right)}{  \left(1 -\beta - \frac{\alpha}{2}\right) \left(\frac{\alpha}{2} + \beta\right)}   + 1 \\
    &=  \frac{ \Myt^2 \left(1 -\frac{\alpha}{2} \Myt -\beta \right)   (-1)\left(\frac{\alpha}{2}\Myt + \beta  \right)}{  \left(1 -\beta - \frac{\alpha}{2}\right) \left(\frac{\alpha}{2} + \beta\right)} + 1 \\   
    &=  1 - \frac{ \Myt^2 \left(1 -\frac{\alpha}{2} \Myt -\beta \right) \left(\frac{\alpha}{2}\Myt + \beta  \right)}{  \left(1 -\beta - \frac{\alpha}{2}\right) \left(\frac{\alpha}{2} + \beta\right)}.
\end{align}

\end{proof}

\subsection{Derivations of metrics: ponderate case}
\label{subsec-proofs-metrics-ponderate-stump}

\subsubsection{Ponderate stump}
\begin{proof}[Proof of \Cref{prop:main1-metrics}]
\label{proof:ponderate-stump}
We have,
\begin{align}
    \mathds{P}(X^{\star} \leq y) &= \int_{\mathds{R}} \mathds{P} (X^{\star}=x) \mathrm{d}x \\
    &= \int_{0}^{y} \frac{\omega_{0}(1- \alpha x - \beta) + \omega_{1}(\alpha x + \beta )}{p(1-p)}\mathrm{d}x \\
    &= \frac{1}{p(1-p)} \omega_{0}[x- \frac{\alpha}{2} x^2 - \beta x ]_0^y + \omega_{1}[\frac{\alpha}{2} x^2 + \beta x]_0^y \\
    &= \frac{1}{p(1-p)} \omega_{0}(y- \frac{\alpha}{2} y^2 - \beta y ) + \omega_{1}(\frac{\alpha}{2} y^2 + \beta y).
\end{align}

Then,
\begin{align}
    f^{\star}(x|0 < X^{\star} \leq \Myt) &= \frac{\mathds{1}_{\{0 < X^{\star} \leq \Myt\}} f^{\star}(x)}{\mathds{P}(X^{\star} \leq \Myt) - \mathds{P}(X^{\star} \leq 0)} \\
    &= \frac{\mathds{1}_{\{0 < X^{\star} \leq \Myt\}} \frac{\omega_{0}(1- \alpha x - \beta) + \omega_{1}(\alpha x + \beta )}{p(1-p)}}{\frac{1}{p(1-p)} \omega_{0}(\Myt- \frac{\alpha}{2} \Myt^2 - \beta t ) + \omega_{1}(\frac{\alpha}{2} \Myt^2 + \beta \Myt) - 0} \\
    &=\mathds{1}_{\{0 < X^{\star} \leq \Myt\}}\frac{ \omega_{0}(1- \alpha x - \beta) + \omega_{1}(\alpha x + \beta )}{ \omega_{0}(\Myt- \frac{\alpha}{2} \Myt^2 - \beta \Myt ) + \omega_{1}(\frac{\alpha}{2} \Myt^2 + \beta \Myt)}
\end{align}

Similarly,
\begin{align}
    f^{\star}(x|\Myt < X^{\star} \leq 1) &= \frac{\mathds{1}_{\{\Myt < X^{\star} \leq 1\}} f^{\star}(x)}{\mathds{P}(X^{\star} \leq 1) - \mathds{P}(X^{\star} \leq \Myt)} \\
    &= \frac{\mathds{1}_{\{\Myt < X^{\star} \leq 1\}} \frac{\omega_{0}(1- \alpha x - \beta) + \omega_{1}(\alpha x + \beta )}{p(1-p)}}{1 - \frac{1}{p(1-p)} \left[ \omega_{0}(\Myt- \frac{\alpha}{2} \Myt^2 - \beta \Myt ) + \omega_{1}(\frac{\alpha}{2} \Myt^2 + \beta \Myt) \right]} \\
    &=\mathds{1}_{\{\Myt < X^{\star} \leq 1\}}\frac{ \omega_{0}(1- \alpha x - \beta) + \omega_{1}(\alpha x + \beta )}{ p(1-p) - \omega_{0}(\Myt- \frac{\alpha}{2} \Myt^2 - \beta \Myt ) - \omega_{1}(\frac{\alpha}{2} \Myt^2 + \beta \Myt)}
\end{align}

Let $\Myt$ be the split value of our reweighted stump $\hat{m}_{k}^{\star}$. Let $s$, be the threshold of for our stump to predict $Y=1$ for each cell. More precisely,
\begin{align}
    \mathds{E}[Y^{\star}|X^{\star} \in A_{R}] \mathds{1}_{\{X^{\star} \in A_{R}\}} \geq t \Longrightarrow \hat{Y} = 1 \\
    \mathds{E}[Y^{\star}|X^{\star} \in A_{L}]\mathds{1}_{\{X^{\star} \in A_{L}\}} \geq t \Longrightarrow \hat{Y} = 1 
\end{align}

We have,
\begin{align}
    &\mathds{E}[Y^{\star}|X^{\star} \in A_{L}] \\&= \mathds{E}[\mathds{E}[Y^{\star}|X^{\star}]|X^{\star} \in A_{L}] && (\text{Tower property}) \\
    & = \mathds{E}[P(Y^{\star}=1|X^{\star})| X^{\star} \in A_{L}] \\
    & = \mathds{E}\left[ \frac{\omega_{1}\mu(X^{\star})}{ \omega_{0}(1-\mu(X^{\star})) + \omega_{1}\mu(X^{\star}) }| X^{\star} \in A_{L}\right] \\
    &= \int_{\mathds{R}}  \frac{\omega_{1}\mu(x)}{ \omega_{0}(1-\mu(x)) + \omega_{1}\mu(x) } \mathds{1}_{\{0 < x \leq t\}}\frac{ \omega_{0}(1- \alpha x - \beta) + \omega_{1}(\alpha x + \beta )}{ \omega_{0}(\Myt- \frac{\alpha}{2} \Myt^2 - \beta \Myt ) + \omega_{1}(\frac{\alpha}{2} \Myt^2 + \beta t)} \mathrm{d}x \\
    &= \int_{0}^{\Myt}  \frac{\omega_{1}\mu(x)}{  \omega_{0}(\Myt- \frac{\alpha}{2} \Myt^2 - \beta \Myt ) + \omega_{1}(\frac{\alpha}{2} \Myt^2 + \beta \Myt) }  \mathrm{d}x \\
    &=   \frac{1}{  \omega_{0}(\Myt- \frac{\alpha}{2} \Myt^2 - \beta \Myt ) + \omega_{1}(\frac{\alpha}{2} \Myt^2 + \beta \Myt) }  \int_{0}^{\Myt}  \omega_{1}(\alpha x + \beta ) \mathrm{d}x \\
    &=   \frac{\omega_{1}(\frac{\alpha}{2}\Myt^{2} + \beta \Myt )}{  \omega_{0}(\Myt- \frac{\alpha}{2} \Myt^2 - \beta \Myt ) + \omega_{1}(\frac{\alpha}{2} \Myt^2 + \beta \Myt) } .
\end{align}

Similarly,
\begin{align}
    &\mathds{E}[Y^{\star}|X^{\star} \in A_{R}] \\&= \mathds{E}[\mathds{E}[Y^{\star}|X^{\star}]|X^{\star} \in A_{R}] \; \;  (\text{Tower property}) \\
    & = \mathds{E}[P(Y^{\star}=1|X^{\star})| X^{\star} \in A_{R}] \\
    & = \mathds{E}\left[ \frac{\omega_{1}\mu(X^{\star})}{ \omega_{0}(1-\mu(X^{\star})) + \omega_{1}\mu(X^{\star}) }| X^{\star} \in A_{R}\right] \\
    &= \int_{\mathds{R}}  \frac{\omega_{1}\mu(x)}{ \omega_{0}(1-\mu(x)) + \omega_{1}\mu(x) } \mathds{1}_{\{\Myt < x \leq 1\}}\frac{ \omega_{0}(1- \alpha x - \beta) + \omega_{1}(\alpha x + \beta )}{ p(1-p) - \omega_{0}(\Myt- \frac{\alpha}{2} \Myt^2 - \beta \Myt ) - \omega_{1}(\frac{\alpha}{2} \Myt^2 + \beta \Myt)} \mathrm{d}x \\
    &= \int_{\Myt}^{1}  \frac{\omega_{1}\mu(x)}{ p(1-p) - \omega_{0}(\Myt- \frac{\alpha}{2} \Myt^2 - \beta \Myt ) - \omega_{1}(\frac{\alpha}{2} \Myt^2 + \beta \Myt) }  \mathrm{d}x \\
    &=   \frac{1}{ p(1-p) - \omega_{0}(\Myt- \frac{\alpha}{2} \Myt^2 - \beta \Myt ) - \omega_{1}(\frac{\alpha}{2} \Myt^2 + \beta \Myt) }  \int_{\Myt}^{1}  \omega_{1}(\alpha x + \beta ) \mathrm{d}x \\
    &=   \frac{\omega_{1}(\frac{\alpha}{2}(1-\Myt^{2}) + \beta (1-\Myt) )}{p(1-p) -  \omega_{0}(\Myt- \frac{\alpha}{2} \Myt^2 - \beta \Myt ) - \omega_{1}(\frac{\alpha}{2} \Myt^2 + \beta \Myt) } .
\end{align}

Then,
\begin{align}
    x \leq t \Longrightarrow x \in A_{L} \Longrightarrow \hat{m}_{k}^{\star}(x)=1 \; \text{if} \; \frac{\omega_{1}(\frac{\alpha}{2}\Myt^{2} + \beta t )}{  \omega_{0}(\Myt- \frac{\alpha}{2} \Myt^2 - \beta \Myt ) + \omega_{1}(\frac{\alpha}{2} \Myt^2 + \beta \Myt) } \geq t.
\end{align}
On the other hand,
\begin{align}
    x > \Myt \Longrightarrow x \in A_{R,t} \Longrightarrow \hat{m}_{k}^{\star}(x)=1 \; \text{if} \; \frac{\omega_{1}(\frac{\alpha}{2}(1-\Myt^{2}) + \beta (1-\Myt) )}{p(1-p) -  \omega_{0}(\Myt- \frac{\alpha}{2} \Myt^2 - \beta \Myt ) - \omega_{1}(\frac{\alpha}{2} \Myt^2 + \beta \Myt) } \geq t.
\end{align}

Thus,
\begin{align}
    \hat{m}_{k}^{\star}(x) = \left\{ 
        \begin{array}{ll}
            1 & \text{if} \; x \leq \Myt \; \text{and} \; \frac{\omega_{1}(\frac{\alpha}{2}t^{2} + \beta \Myt )}{  \omega_{0}(\Myt- \frac{\alpha}{2} \Myt^2 - \beta \Myt ) + \omega_{1}(\frac{\alpha}{2} t^2 + \beta \Myt) } \geq t \\
            1 & \text{if} \; x > \Myt \; \text{and} \; \frac{\omega_{1}(\frac{\alpha}{2}(1-\Myt^{2}) + \beta (1-\Myt) )}{p(1-p) -  \omega_{0}(\Myt- \frac{\alpha}{2} \Myt^2 - \beta \Myt ) - \omega_{1}(\frac{\alpha}{2} \Myt^2 + \beta \Myt) } \geq t \; \\
            0 & \text{otherwise}.
        \end{array}
    \right.
\end{align}

Thus,
\begin{align}
    \hat{m}_{k}^{\star}(x) = \left\{ 
        \begin{array}{ll}
            1 & \text{if} \; t \in [0, \frac{\omega_{1}(\frac{\alpha}{2}\Myt^{2} + \beta \Myt )}{  \omega_{0}(\Myt- \frac{\alpha}{2} \Myt^2 - \beta \Myt ) + \omega_{1}(\frac{\alpha}{2} \Myt^2 + \beta t) }] \\
            \mathds{1}_{\{x \geq \Myt\}} & \text{if} \; t \in [ \frac{\omega_{1}(\frac{\alpha}{2}\Myt^{2} + \beta \Myt )}{  \omega_{0}(\Myt- \frac{\alpha}{2} \Myt^2 - \beta \Myt ) + \omega_{1}(\frac{\alpha}{2} \Myt^2 + \beta \Myt) }, \frac{\omega_{1}(\frac{\alpha}{2}(1-\Myt^{2}) + \beta (1-\Myt) )}{p(1-p) -  \omega_{0}(\Myt- \frac{\alpha}{2} \Myt^2 - \beta \Myt ) - \omega_{1}(\frac{\alpha}{2} \Myt^2 + \beta \Myt) }] \\
            0 & \text{if} \; t \in [\frac{\omega_{1}(\frac{\alpha}{2}(1-\Myt^{2}) + \beta (1-\Myt) )}{p(1-p) -  \omega_{0}(\Myt- \frac{\alpha}{2} \Myt^2 - \beta \Myt ) - \omega_{1}(\frac{\alpha}{2} \Myt^2 + \beta \Myt) }, 1].
        \end{array}
    \right.
\end{align}

\subsubsection{Precision}
We recall that,
\begin{align}
    Precision &= \frac{\mathds{P}(Y=1, \hat{m}_{k}^{\star}(X)=1)}{\mathds{P}(Y=1, \hat{m}_{k}^{\star}(X)=1) + \mathds{P}(Y=0, \hat{m}_{k}^{\star}(X)=1)}.
\end{align}

We have, 
\begin{align}
    \mathds{P}(Y=1, \hat{m}_{k}^{\star}(X)=1) &= \mathds{E}\left[\mathds{1}_{\{Y=1\}}\mathds{1}_{\{\hat{m}_{k}^{\star}(X^{\star})=1\}}\right] \\
    &=\mathds{E}\left[\mathds{E}\left[\mathds{1}_{\{Y=1\}}\mathds{1}_{\{\hat{m}_{k}^{\star}(X)=1\}}|X\right] \right]  && (\text{Tower property})\\
    &=\mathds{E}\left[\mathds{1}_{\{\hat{m}_{k}^{\star}(X)=1\}}\mathds{E}\left[\mathds{1}_{\{Y=1\}}|X\right] \right] \\
    &=\mathds{E}\left[\mathds{1}_{\{\hat{m}_{k}^{\star}(X)=1\}}\mathds{P}\left(Y=1|X\right) \right] \\
    &=\mathds{E}\left[\mathds{1}_{\{\hat{m}_{k}^{\star}(X)=1\}}(\alpha X + \beta ) \right] \\
    &=\alpha\mathds{E}\left[\mathds{1}_{\{\hat{m}_{k}^{\star}(X)=1\}}X \right]  + \beta \mathds{P}(\hat{m}_{k}^{\star}(X)=1).
\end{align}

Similarly, 
\begin{align}
    \mathds{P}(Y=0, \hat{m}_{k}^{\star}(X)=1) &= \mathds{E}\left[\mathds{1}_{\{Y=1\}}\mathds{1}_{\{\hat{m}_{k}^{\star}(X)=1\}}\right] \\
    &=\mathds{E}\left[\mathds{E}\left[\mathds{1}_{\{Y=0\}}\mathds{1}_{\{\hat{m}_{k}^{\star}(X)=1\}}|X\right] \right]  \; \;  (\text{Tower property})\\
    &=\mathds{E}\left[\mathds{1}_{\{\hat{m}_{k}^{\star}(X)=1\}}\mathds{E}\left[\mathds{1}_{\{Y=0\}}|X\right] \right] \\
    &=\mathds{E}\left[\mathds{1}_{\{\hat{m}_{k}^{\star}(X)=1\}}\mathds{P}\left(Y=0|X\right) \right] \\
    &=\mathds{E}\left[\mathds{1}_{\{\hat{m}_{k}^{\star}(X)=1\}} (1 -\mathds{P}\left(Y=1|X\right) ) \right] \\
    &=\mathds{E}\left[\mathds{1}_{\{\hat{m}_{k}^{\star}(X)=1\}}(1 -\alpha x - \beta ) \right] \\
    &=\mathds{P}(\hat{m}_{k}^{\star}(X)=1) - \alpha\mathds{E}\left[\mathds{1}_{\{\hat{m}_{k}^{\star}(X)=1\}}X \right]  - \beta \mathds{P}(\hat{m}_{k}^{\star}(X)=1).
\end{align}

\paragraph{Case :} $t \in [0, \frac{\omega_{1}(\frac{\alpha}{2}\Myt^{2} + \beta \Myt )}{  \omega_{0}(\Myt- \frac{\alpha}{2} \Myt^2 - \beta \Myt ) + \omega_{1}(\frac{\alpha}{2} \Myt^2 + \beta \Myt) }]$

Then, $\hat{m}_{k}(X)= 1$ for all $x \in [0,1]$. Thus,
\begin{align}
    \mathds{P}(Y=1, \hat{Y}^{\star}=1) &=\alpha\mathds{E}\left[\mathds{1}_{\{\hat{m}_{k}^{\star}(X)=1\}}X \right]  + \beta \mathds{P}(\hat{m}_{k}^{\star}(X)=1) \\
    &= \frac{\alpha }{2} + \beta.
\end{align}

Furthermore,
\begin{align}
    \mathds{P}(Y=0, \hat{Y}^{\star}=1) &=\mathds{P}(\hat{m}_{k}^{\star}(X)=1) - \alpha\mathds{E}\left[\mathds{1}_{\{\hat{m}_{k}^{\star}(X)=1\}}X \right]  - \beta \mathds{P}(\hat{m}_{k}^{\star}(X)=1) \\
    &= 1 - \frac{\alpha }{2} - \beta.
\end{align}

Therefore,
\begin{align}
    Precision = \frac{\frac{\alpha }{2} + \beta}{\frac{\alpha }{2} + \beta + 1 - \frac{\alpha }{2} - \beta} = \frac{\alpha }{2} + \beta.
\end{align}

\paragraph{Case :} $t \in [ \frac{\omega_{1}(\frac{\alpha}{2}\Myt^{2} + \beta \Myt )}{  \omega_{0}(\Myt- \frac{\alpha}{2} \Myt^2 - \beta \Myt ) + \omega_{1}(\frac{\alpha}{2} \Myt^2 + \beta \Myt) }, \frac{\omega_{1}(\frac{\alpha}{2}(1-\Myt^{2}) + \beta (1-\Myt) )}{p(1-p) -  \omega_{0}(\Myt- \frac{\alpha}{2} \Myt^2 - \beta \Myt ) - \omega_{1}(\frac{\alpha}{2} \Myt^2 + \beta \Myt) }]$

Then, $D^{\star}_{\Myt}(X)=\mathds{1}_{\{X\geq \Myt\}}$  for all $x \in [0,1]$.
Thus,
\begin{align}
     \mathds{P}(Y=1, \hat{m}_{k}^{\star}(X)=1) &=\alpha\mathds{E}\left[\mathds{1}_{\{\hat{m}_{k}^{\star}(X)=1\}} X  \right] + \beta \mathds{P}(\hat{m}_{k}^{\star}(X)=1) \\
     &= \alpha\mathds{E}\left[\mathds{1}_{\{X\geq t\}} X  \right] + \beta \mathds{P}(X\geq \Myt) \\
     &=\frac{\alpha}{2}(1-\Myt^{2}) + \beta \mathds{P}(X\geq \Myt)  \\
     &=\frac{\alpha}{2}(1-\Myt^{2}) + \beta (1-\Myt   ).
\end{align}

Besides,
\begin{align}
    &\mathds{P}(Y=0, \hat{Y}^{\star}=1) \\
    &=\mathds{P}(\hat{m}_{k}^{\star}(X^{\star})=1) - \alpha\mathds{E}\left[\mathds{1}_{\{\hat{m}_{k}^{\star}(X)=1\}} X  \right] - \beta \mathds{P}(\hat{m}_{k}^{\star}(X)=1) \\
    &=\mathds{P}(X \geq \Myt) - \alpha\mathds{E}\left[\mathds{1}_{\{X^{\star} \geq \Myt\}} X  \right] - \beta \mathds{P}(X^{\star} \geq \Myt) \\
    &= (1-\Myt) -\alpha\frac{1}{2}(1-\Myt^{2}) - \beta (1-\Myt)  \; \;  \left(\mathds{E}\left[\mathds{1}_{\{X \leq \Myt\}} X  \right] = \int_{\Myt}^{1}x\mathrm{d}x= \frac{1}{2}(1-\Myt^{2}) \right) \\
    &= \frac{\alpha}{2}\Myt^{2} + \beta \Myt + 1-\Myt -\beta - \frac{\alpha}{2}.
\end{align}

Therefore,
\begin{align}
    Precision &= \frac{\frac{\alpha}{2}(1-t^{2}) + \beta (1-\Myt)}{\frac{\alpha}{2}(1-\Myt^{2}) + \beta (1-\Myt) + \frac{\alpha}{2}\Myt^{2} + \beta \Myt + 1-\Myt -\beta - \frac{\alpha}{2} } \\
    &= \frac{\frac{\alpha}{2}(1-\Myt^{2}) + \beta (1-\Myt)}{ 1-\Myt } \\
    &= \frac{\alpha}{2} (1+\Myt) + \beta.
\end{align}

\paragraph{Case :} $t \in [ \frac{\omega_{1}(\frac{\alpha}{2}(1-\Myt^{2}) + \beta (1-\Myt) )}{p(1-p) -  \omega_{0}(\Myt- \frac{\alpha}{2} \Myt^2 - \beta \Myt ) - \omega_{1}(\frac{\alpha}{2} \Myt^2 + \beta \Myt) }, 1]$
Then, $\hat{m}_{k}^{\star}(X)= 0$ for all $x \in [0,1]$. Thus,
\begin{align}
    \mathds{P}(Y=1, \hat{Y}=1) = 0 \\ 
    \mathds{P}(Y=0, \hat{Y}=1) = 0.
\end{align}
Therefore,
\begin{align}
    Precision = \frac{0}{0} = Undefined.
\end{align}

Finally,
\begin{align}
    &Precision \\
    &= \left\{ 
        \begin{array}{ll}
             \alpha/2 + \beta& \text{if} \; t \in [0, \frac{\omega_{1}(\frac{\alpha}{2}\Myt^{2} + \beta \Myt )}{  \omega_{0}(\Myt- \frac{\alpha}{2} \Myt^2 - \beta \Myt ) + \omega_{1}(\frac{\alpha}{2} \Myt^2 + \beta \Myt) }] \\
            \frac{\alpha}{2} (1+\Myt) + \beta & \text{if} \; t \in [ \frac{\omega_{1}(\frac{\alpha}{2}\Myt^{2} + \beta \Myt )}{  \omega_{0}(\Myt- \frac{\alpha}{2} \Myt^2 - \beta \Myt ) + \omega_{1}(\frac{\alpha}{2} \Myt^2 + \beta \Myt) }, \frac{\omega_{1}(\frac{\alpha}{2}(1-\Myt^{2}) + \beta (1-\Myt) )}{p(1-p) -  \omega_{0}(\Myt- \frac{\alpha}{2} \Myt^2 - \beta \Myt ) - \omega_{1}(\frac{\alpha}{2} \Myt^2 + \beta \Myt) }] \\
            Undefined  & \text{if} \; t \in [ \frac{\omega_{1}(\frac{\alpha}{2}(1-\Myt^{2}) + \beta (1-\Myt) )}{p(1-p) -  \omega_{0}(\Myt- \frac{\alpha}{2} \Myt^2 - \beta \Myt ) - \omega_{1}(\frac{\alpha}{2} \Myt^2 + \beta \Myt) }, 1].
        \end{array}
    \right.
\end{align}

\subsubsection{Recall, $F_1$-score, PR AUC, and ROC AUC}

For Recall, $F_1$-score, PR AUC, and ROC AUC, the functional dependence of each metric on the split value $\Myt$ is identical for both weighted and unweighted trees. Consequently, any metric discrepancy between the two approaches arises solely from differences in their optimal split value $\Myt$.
\end{proof}

\subsection{Optimizations of metrics}
\begin{proof}[Proof of \Cref{lemma:opt-metrics-ts}]
\label{proof:opt-metrics-ts}

\textbf{Optimization of F1-score, on $\Myt,s$ :}

We have the maximal F1-score when $t \in [ \frac{\alpha \Myt}{2} + \beta, \frac{\alpha (\Myt+1)}{2} + \beta]$. In such case:
\begin{align}
    \text{F1-score} &= \frac{ \left(\frac{\alpha}{2}(1+\Myt) + \beta\right)(1-\Myt) }{ (\frac{\alpha}{2} + \beta)  + (1-\Myt)} \\
    &= \frac{ \frac{\alpha}{2} + \frac{\alpha}{2}\Myt + \beta - \frac{\alpha}{2}\Myt - \frac{\alpha}{2}\Myt^2 - \beta \Myt}{ -\Myt + 1 +\frac{\alpha}{2} + \beta  } \\
    &= \frac{- \frac{\alpha}{2}\Myt^2  - \beta \Myt + \frac{\alpha}{2}  + \beta  }{ -\Myt + 1 +\frac{\alpha}{2} + \beta  }.
\end{align}

Then,
\begin{align}
    &\left(\frac{- \frac{\alpha}{2}\Myt^2  - \beta \Myt + \frac{\alpha}{2}  + \beta  }{ -\Myt + 1 +\frac{\alpha}{2} + \beta  }\right)^{\prime} \\
    &=  \frac{(-\alpha \Myt - \beta)\left(-\Myt +1 + \frac{\alpha}{2} + \beta \right) - \left(-\frac{\alpha}{2}\Myt^2 - \beta \Myt + \frac{\alpha}{2} + \beta \right)(-1)}{ (-\Myt + 1 +\frac{\alpha}{2} + \beta)^2  } \\
    &= \frac{(-\alpha \Myt - \beta)\left(-\Myt +1 + \frac{\alpha}{2} + \beta \right) +\left(-\frac{\alpha}{2}\Myt^2 - \beta \Myt + \frac{\alpha}{2} + \beta \right)}{ (-\Myt + 1 +\frac{\alpha}{2} + \beta)^2  } \\
    &= \frac{\alpha \Myt^2-\alpha \Myt(1 + \frac{\alpha}{2} + \beta)+\beta \Myt -\beta\left( 1 + \frac{\alpha}{2} + \beta \right) -\frac{\alpha}{2}\Myt^2 - \beta \Myt + \frac{\alpha}{2} + \beta }{ (-\Myt + 1 +\frac{\alpha}{2} + \beta)^2  } \\
    &= \frac{\frac{\alpha}{2} \Myt^2-\alpha \Myt(1 + \frac{\alpha}{2} + \beta) -\beta\left( 1 + \frac{\alpha}{2} + \beta \right) + \frac{\alpha}{2} + \beta }{ (-\Myt + 1 +\frac{\alpha}{2} + \beta)^2  } \\
    &= \frac{\frac{\alpha}{2} \Myt^2-\alpha \Myt(1 + \frac{\alpha}{2} + \beta) -\beta\left( \frac{\alpha}{2} + \beta \right) + \frac{\alpha}{2}  }{ (-\Myt + 1 +\frac{\alpha}{2} + \beta)^2  }.
\end{align}

 We have,
 \begin{align}
     & \alpha^2(1 + \frac{\alpha}{2} + \beta)^2 - 4 \frac{\alpha}{2} \left[ \frac{\alpha}{2} - \beta\left( \frac{\alpha}{2} + \beta \right) \right] \\
     &=\alpha^2(1 + \frac{\alpha}{2} + \beta)^2 - 2 \alpha \left[ \frac{\alpha}{2} -  \frac{\alpha \beta}{2} - \beta^2  \right] \\
     &=\alpha^2(\frac{\alpha}{2} + \beta)^2 - 2 \alpha \left[ -  \frac{\alpha \beta}{2} - \beta^2  \right] \\
     &=\alpha^2(\frac{\alpha}{2} + \beta)^2 +   \alpha^2 \beta + 2 \alpha \beta^2 >0. 
 \end{align}
Then we have two solution to the equation $\frac{\alpha}{2} \Myt^2-\alpha \Myt(1 + \frac{\alpha}{2} + \beta) -\beta\left( \frac{\alpha}{2} + \beta \right) + \frac{\alpha}{2}  = 0$, which are :
\begin{align}
    &\frac{\alpha(1 + \frac{\alpha}{2} + \beta) - \sqrt{\alpha^2(1 + \frac{\alpha}{2} + \beta)^2 - 4 \frac{\alpha}{2} \left[ \frac{\alpha}{2} - \beta\left( \frac{\alpha}{2} + \beta \right) \right]}}{\alpha} \\
    &=\frac{\alpha(1 + \frac{\alpha}{2} + \beta) - \sqrt{\alpha^2(1 + \frac{\alpha}{2} + \beta)^2 - \alpha^2 + 2 \alpha \beta \left( \frac{\alpha}{2} + \beta \right) }}{\alpha} \\
    &=\frac{\alpha(1 + \frac{\alpha}{2} + \beta) - \sqrt{\alpha^2(1 + \frac{\alpha}{2} + \beta)^2 + \alpha \left[ \alpha(\beta-1) + 2 \beta^2 \right] }}{\alpha}.
\end{align}

We have, 
$$ 
\alpha(\beta-1) + 2 \beta^2 >0 \Longrightarrow \alpha < \frac{2 \beta^2}{1-\beta},
$$ thus,
$$
\frac{\alpha(1 + \frac{\alpha}{2} + \beta) - \sqrt{\alpha^2(1 + \frac{\alpha}{2} + \beta)^2 + \alpha \left[ \alpha(\beta-1) + 2 \beta^2 \right] }}{\alpha} <0
$$

Furthermore,
\begin{align}
    &\frac{\alpha(1 + \frac{\alpha}{2} + \beta) + \sqrt{\alpha^2(1 + \frac{\alpha}{2} + \beta)^2 - 4 \frac{\alpha}{2} \left[ \frac{\alpha}{2} - \beta\left( \frac{\alpha}{2} + \beta \right) \right]}}{\alpha} \\
    &= 1 + \frac{\alpha}{2} + \beta + \sqrt{\alpha(1 + \frac{\alpha}{2} + \beta)^2 -  \frac{4}{2} \left[ \frac{\alpha}{2} - \beta\left( \frac{\alpha}{2} + \beta \right) \right]}  \\
    &> 1
\end{align}

Then, if  $\alpha(\beta-1) + 2 \beta^2 >0$,the derivate of the F1-score with respect to $\Myt$ is negative on $[0,1]$,thus the F1-score is maximized for $\Myt=0$. 

Else, the derivate is positive on $\left[0,\frac{\alpha(1 + \frac{\alpha}{2} + \beta) - \sqrt{\alpha^2(1 + \frac{\alpha}{2} + \beta)^2 + \alpha \left[ \alpha(\beta-1) + 2 \beta^2 \right] }}{\alpha}\right]$ and negative in $\left[\frac{\alpha(1 + \frac{\alpha}{2} + \beta) - \sqrt{\alpha^2(1 + \frac{\alpha}{2} + \beta)^2 + \alpha \left[ \alpha(\beta-1) + 2 \beta^2 \right] }}{\alpha}, 1\right]$, thus F1-score is maximized for 
$$\Myt= \frac{\alpha(1 + \frac{\alpha}{2} + \beta) - \sqrt{\alpha^2(1 + \frac{\alpha}{2} + \beta)^2 + \alpha \left[ \alpha(\beta-1) + 2 \beta^2 \right] }}{\alpha}$$.

\paragraph{Optimization of PR AUC}
We have,
\begin{align}
    &\left(\left(\frac{\alpha^2}{4}t \right)- \left(\frac{\alpha^2}{4}t^3 \right) + \left(\frac{\alpha \beta }{2}t \right)   - \left(\frac{\alpha \beta }{2}t^2 \right)\right)^{\prime} \\
    &= -\frac{\alpha^2}{4} 3t^2 - \frac{\alpha \beta }{2}2t + \frac{\alpha^2}{4} + \frac{\alpha \beta}{2} \\
    &= -\frac{3 \alpha^2}{4} t^2 - \alpha \beta t + \frac{\alpha^2 + 2 \alpha \beta }{4} \\
    &= -\frac{3 \alpha^2}{4} t^2 - \alpha \beta t + \frac{\alpha(\alpha + 2 \beta) }{4}.
\end{align}

We have, 
\begin{align}
    &(-\alpha \beta)^2 - 4 \left(-\frac{3 \alpha^2}{4}\right) \frac{\alpha(\alpha + 2 \beta) }{4} \\
    &=(\alpha \beta)^2 + \frac{3 \alpha^3(\alpha + 2 \beta)}{4} \\
    &=(\alpha \beta)^2 + \frac{3}{4}\alpha^3(\alpha + 2 \beta) \\
    &= \alpha^2 \left( \beta^2 + \frac{3}{4}\alpha(\alpha + 2 \beta)\right) > 0.
\end{align}

The solution to the equation $-\frac{3 \alpha^2}{4} t^2 - \alpha t + \frac{\alpha(\alpha + 2 \beta) }{4}=0 $ are,
\begin{align}
    \frac{\alpha \beta + \sqrt{\alpha^2 \left( \beta^2 + \frac{3}{4}\alpha(\alpha + 2 \beta)\right)}}{-\frac{3 \alpha^2}{2}} < 0.
\end{align}

And,
\begin{align}
    & \frac{\alpha \beta - \sqrt{\alpha^2 \left( \beta^2 + \frac{3}{4}\alpha(\alpha + 2 \beta)\right)}}{-\frac{3 \alpha^2}{2}} \\
    &= \frac{ 2\sqrt{\alpha^2 \left( \beta^2 + \frac{3}{4}\alpha(\alpha + 2 \beta)\right)} - 2\alpha \beta}{3 \alpha^2} \\
    &=\frac{ 2\alpha\sqrt{  \beta^2 + \frac{3}{4}\alpha(\alpha + 2 \beta)} - 2\alpha \beta}{3 \alpha^2} \\
    &=\frac{ 2\sqrt{  \beta^2 + \frac{3}{4}\alpha(\alpha + 2 \beta)} - 2 \beta}{3 \alpha}.
\end{align}
\as{ABDOULAYE REPRENDRE DEPUIS ICI !!}

We have, if $\alpha>0$,
\begin{align}
    &\frac{ 2\sqrt{  \beta^2 + \frac{3}{4}\alpha(\alpha + 2 \beta)} - 2 \beta}{3 \alpha} < 1 \\
    & 2\sqrt{  \beta^2 + \frac{3}{4}\alpha(\alpha + 2 \beta)}  < 3\alpha + 2\beta \\
    & \sqrt{  \beta^2 + \frac{3}{4}\alpha(\alpha + 2 \beta)}  < \frac{3\alpha + 2\beta}{2} \\
    &   \beta^2 + \frac{3}{4}\alpha(\alpha + 2 \beta)  < \left(\frac{3\alpha + 2\beta}{2} \right)^2 \\
    &   \beta^2 + \frac{3}{4}\alpha(\alpha + 2 \beta)  < \frac{9}{4}\alpha^2 + \frac{12}{4}\alpha \beta + \beta^2\\
    &   \frac{3}{4}\alpha^2 + \frac{6}{4}\alpha\beta < \frac{9}{4}\alpha^2 + \frac{12}{4}\alpha \beta \\
    &   \frac{-6}{4}\alpha^2  < \frac{6}{4}\alpha \beta \\
    &   -\alpha  <  \beta. \\
\end{align}

We recall that, $\alpha, \beta \in \mathds{R}^{+}$, thus $-\alpha  <  \beta$ for all $\alpha, \beta$.

Finally the derivate of the PR AUC,with respect to $t$,  is positive on $\left[ 0, \frac{ 2\sqrt{  \beta^2 + \frac{3}{4}\alpha(\alpha + 2 \beta)} - 2 \beta}{3 \alpha}\right]$, then negative in $[\frac{ 2\sqrt{  \beta^2 + \frac{3}{4}\alpha(\alpha + 2 \beta)} - 2 \beta}{3 \alpha},1]$.
The PR AUC is then maximized for $t=\frac{ 2\sqrt{  \beta^2 + \frac{3}{4}\alpha(\alpha + 2 \beta)} - 2 \beta}{3 \alpha}$ if $\alpha >0$ and for $t\in [0,1]$ otherwise (derivate null).

\paragraph{Optimization  of ROC}
 The ROC AUC,

\begin{align}
 ROC AUC = 1 - \frac{ \Myt^2 \left(1 -\frac{\alpha}{2} \Myt -\beta \right) \left(\frac{\alpha}{2}t + \beta  \right)}{  \left(1 -\beta - \frac{\alpha}{2}\right) \left(\frac{\alpha}{2} + \beta\right)},
\end{align}
is maximized when the quantity,
\begin{align}
\label{eq-optROC-quantity-intermediaire}
     - \frac{ \Myt^2 \left(1 -\frac{\alpha}{2} \Myt -\beta \right) \left(\frac{\alpha}{2}\Myt + \beta  \right)}{  \left(1 -\beta - \frac{\alpha}{2}\right) \left(\frac{\alpha}{2} + \beta\right)} <0,
\end{align}
is minimized. 

We minimize,
\begin{align}
    g(\Myt) = -\Myt^2 \left(1 -\frac{\alpha}{2} \Myt -\beta \right) \left(\frac{\alpha}{2}\Myt + \beta  \right).
\end{align}

We have,
\begin{align}
    g^{\prime}(\Myt) &= -2\Myt\left(1-\frac{\alpha}{2}\Myt - \beta \right) \left(\frac{\alpha}{2}\Myt + \beta\right) + (-\Myt^2) \left[\frac{\alpha}{2}\left(\frac{\alpha}{2}\Myt + \beta \right)  + \frac{\alpha}{2}(1-\frac{\alpha}{2}\Myt - \beta)\right]\\
    &= -2\Myt\left(1-\frac{\alpha}{2}\Myt - \beta \right) \left(\frac{\alpha}{2}\Myt + \beta\right) - \Myt^2\frac{\alpha}{2} \left[(1-\frac{\alpha}{2}\Myt - \beta) - (\frac{\alpha}{2}\Myt + \beta)\right]\\
    &= -2\Myt\left(1-\frac{\alpha}{2}\Myt - \beta \right) \left(\frac{\alpha}{2}\Myt + \beta\right) - \Myt^2\frac{\alpha}{2} \left[(1-2\frac{\alpha}{2}\Myt - 2\beta) \right]\\
    &= -2\Myt\left(1-\frac{\alpha}{2}\Myt - \beta \right) \left(\frac{\alpha}{2}\Myt + \beta\right) - \Myt^2\frac{\alpha}{2} \left[2\left(1-\frac{\alpha}{2}\Myt - \beta\right) \right] + \Myt^2 \frac{\alpha}{2}\\
    &= -2\Myt\left(1-\frac{\alpha}{2}\Myt - \beta \right) \left(\frac{\alpha}{2}\Myt + \beta\right) - \Myt^2\alpha \left(1-\frac{\alpha}{2}\Myt - \beta\right) + \Myt^2 \frac{\alpha}{2}\\
    &= \left(1-\frac{\alpha}{2}\Myt - \beta\right) \left[ -2\Myt\left(\frac{\alpha}{2}\Myt + \beta\right) -\Myt^2 \alpha\right] + \Myt^2 \frac{\alpha}{2}\\
    &= \left(1-\frac{\alpha}{2}\Myt - \beta\right) \left[ -\alpha \Myt^2 - 2\beta \Myt -\Myt^2 \alpha\right] + \Myt^2 \frac{\alpha}{2}\\ 
    &= \left(1-\frac{\alpha}{2}\Myt - \beta\right) \left[ -2\alpha \Myt^2 - 2\beta \Myt \right] + \Myt^2 \frac{\alpha}{2}\\ 
    &= \left[-2\alpha \Myt^2 +\alpha^2 \Myt^3 +2\alpha \beta \Myt^2 - 2\beta \Myt + \alpha\beta \Myt^2 + 2\beta^2 \Myt \right] + \Myt^2 \frac{\alpha}{2} \as{Erreur de signe après}\\
    &= \Myt^3 \alpha^2 + \Myt^2 \left(-2\alpha +2 \alpha \beta + \alpha \beta + \frac{\alpha}{2}\right) + \Myt \left(-2 \beta +2 \beta^2\right) \\
    &= \Myt^3 \alpha^2 + \Myt^2 \left(-\frac{3\alpha}{2} +3 \alpha \beta \right) + \Myt \left(-2 \beta +2 \beta^2\right) \\
    &= \Myt \left[\Myt^2 \alpha^2 + \Myt \left(-\frac{3\alpha}{2} +3 \alpha \beta  \right) + \left(-2 \beta +2 \beta^2\right) \right].
\end{align}
The polynome $\Myt^2 \alpha^2 + \Myt \left(-\frac{3\alpha}{2} +3 \alpha \beta  \right) + \left(-2 \beta +2 \beta^2\right)$ admits as discriminant:
\begin{align}
    &\left(-\frac{3}{2}\alpha + 3 \alpha \beta \right)^2 - 4 \alpha^2 \left(-2 \beta +2 \beta^2\right) \\
    &\left(-\frac{3}{2}\alpha + 3 \alpha \beta \right)^2 + 4 \alpha^2 2\left( \beta - \beta^2\right) > 0 && \left(\beta \leq 1 \text{ and } x\to x-x^2  \text{ is in positive on } [0,1],  \right) \\
    & \frac{9}{4}\alpha^2 -9\alpha^2 \beta + 9\alpha^2\beta^2 + 8\alpha^2 \beta - 8\alpha^2 \beta^2 \\
    & \frac{9}{4}\alpha^2 -\alpha^2 \beta + \alpha^2\beta^2 
\end{align}

Thus, there is the equation $\Myt^2 \alpha^2 + \Myt \left(-\frac{3\alpha}{2} +3 \alpha \beta  \right) + \left(-2 \beta +2 \beta^2\right) = 0$ admits two solutions which are: 
\begin{align}
    &\frac{\frac{3}{2}\alpha - 3\alpha \beta \pm \alpha \sqrt{\frac{9}{4}- \beta + \beta^2 } }{2 \alpha^2} \\
    &\frac{\frac{3}{2} - 3 \beta \pm \sqrt{\frac{9}{4}- \beta + \beta^2 } }{2 \alpha} \\
    &\frac{3 - 6 \beta \pm \sqrt{9- 4\beta + 4\beta^2 } }{4 \alpha}.
\end{align}

We remark that,
\begin{align}
    &3 - 6 \beta + \sqrt{9- 4\beta + 4\beta^2 } \geq 3 - 6 \beta + \sqrt{9} = 6 (1-\beta) \geq 0,
\end{align}
and 
\begin{align}
    &3 - 6 \beta - \sqrt{9- 4\beta + 4\beta^2 } \leq 3 - 6 \beta - \sqrt{9} = 6 (\beta-1) \leq 0.
\end{align}

Finally, we have,
\begin{align*}
\begin{tikzpicture}
   \tkzTabInit{$\Myt$ / 1 ,$g^{\prime}$/1, $g$/2}{$-\infty$,$\frac{3 - 6 \beta - \sqrt{9- 4\beta + 4\beta^2 } }{4 \alpha}$,$0$, $\frac{3 - 6 \beta + \sqrt{9- 4\beta + 4\beta^2 } }{4 \alpha}$, $+\infty$}
   \tkzTabLine{, -,z, +,z,-,z,+,}
   \tkzTabVar{ +/ , -/ ,+/, -/,+/}
   %\tkzTabLine{, +,, +,,+,z,-,}
   %\tkzTabLine{, -,z, +,,+,,+,}
   %\tkzTabLine{, -,z, +,,+,,+,}
   %\tkzTabVar{+/ , -/$2 + \ln(\frac{1}{16}) +1 $ , +/ }
\end{tikzpicture}
\end{align*}

Besides, if we suppose $\alpha >0$,
\begin{align}
    &\frac{3 - 6 \beta + \sqrt{9- 4\beta + 4\beta^2 } }{4 \alpha} > 1 \\
    &3 - 6 \beta + \sqrt{9- 4\beta + 4\beta^2 }  > 4 \alpha  && \left(4 \alpha > 0 \right) \\
    & \sqrt{9- 4\beta + 4\beta^2 }  > 4 \alpha - 3 + 6 \beta \\
    & 9- 4\beta + 4\beta^2  > \left(4 \alpha - 3 + 6 \beta \right)^2 \\
    & 9- 4\beta + 4\beta^2  > \left(4 \alpha - 3\right)^2 +2 \left(4 \alpha - 3\right)6\beta + 36 \beta^2  \\
    & 9- 4\beta + 4\beta^2  > 16 \alpha^2 -24\alpha + 9 +48 \alpha\beta -36\beta + 36 \beta^2  \\
    & 0  > 16 \alpha^2 -24\alpha +48 \alpha\beta -32\beta + 32 \beta^2 \\
    & 0  > \alpha \left(16 \alpha -24\right) + \beta\left(48 \alpha -32 + 32 \beta \right) \\
    &\alpha \left(24 - 16 \alpha\right)  >  \beta\left(48 \alpha -32 + 32 \beta \right). \\
\end{align}
\end{proof}

\section{Probabilities Ordering}

\subsection{Precision-Recall: general expression.}

\begin{proof}[Proof of \Cref{th:general-prec-rec-correspondance}]
\label{proof_th:general-prec-rec-correspondance}
Under \Cref{assump-infinite-sample}, we suppose we have a trained decision tree (weighted or not) $\hat{m}_{k}$ operating at an output probability threshold $t$.

We have,
\begin{align}
    \mathds{P}\left(Y=1,\hat{Y}=1\right) &= \mathds{E}\left[\mathds{1}_{\{Y=1\}} \mathds{1}_{\{\hat{Y}=1\}}\right] \\
    &= \mathds{E}\left[\mathds{1}_{\{Y=1\}}\mathds{1}_{\{\hat{m}_{k}(X)=1\}}\right] \\
    &= \mathds{E}\left[\mathds{E}\left[\mathds{1}_{\{Y=1\}}\mathds{1}_{\{\hat{m}_{k}(X)=1\}}| X\right] \right] \\
    &= \mathds{E}\left[\mu\left(X\right)\mathds{1}_{\{\hat{m}_{k}(X)=1\}}\right].
\end{align}

\begin{align}
    \mathds{P}\left(Y=0,\hat{Y}=1\right) &= \mathds{E}\left[(1-\mu\left(X\right))\mathds{1}_{\{\hat{m}_{k}(X)=1\}}\right] \\
    &= \mathds{E}\left[\mathds{1}_{\{\hat{m}_{k}(X)=1\}}\right] - \mathds{E}\left[\mu\left(X\right)\mathds{1}_{\{\hat{m}_{k}(X)=1\}}\right].
\end{align}

\begin{align}
    \mathds{P}\left(Y=1,\hat{Y}=0\right) &= \mathds{E}\left[\mu\left(X\right)\mathds{1}_{\{\hat{m}_{k}(X)=0\}}\right] \\
    &= \mathds{E}\left[\mu\left(X\right)\left(1-\mathds{1}_{\{\hat{m}_{k}(X)=1\}}\right)\right] \\
    &= \mathds{E}\left[\mu\left(X\right)\right] - \mathds{E}\left[\mu\left(X\right)\mathds{1}_{\{\hat{m}_{k}(X)=1\}}\right].
\end{align}

Thus we have,
\begin{align}
    Precision_{s} &= \frac{\mathds{E}\left[\mu\left(X\right)\mathds{1}_{\{\hat{m}_{k}(X)=1\}}\right]}{\mathds{E}\left[\mu\left(X\right)\mathds{1}_{\{\hat{m}_{k}(X)=1\}}\right] + \mathds{E}\left[\mathds{1}_{\{\hat{m}_{k}(X)=1\}}\right] - \mathds{E}\left[\mu\left(X\right)\mathds{1}_{\{\hat{m}_{k}(X)=1\}}\right]} \\
    &=\frac{\mathds{E}\left[\mu\left(X\right)\mathds{1}_{\{\hat{m}_{k}(X)=1\}}\right]}{\mathds{E}\left[\mathds{1}_{\{\hat{m}_{k}(X)=1\}}\right]}.
\end{align}

\begin{align}
    &Recall_{s} \\
    &=\frac{\mathds{E}\left[\mu\left(X\right)\mathds{1}_{\{\hat{m}_{k}(X)=1\}}\right]}{\mathds{E}\left[\mu\left(X\right)\mathds{1}_{\{\hat{m}_{k}(X)=1\}}\right] + \mathds{E}\left[\mu\left(X\right)\right] - \mathds{E}\left[\mu\left(X\right)\mathds{1}_{\{\hat{m}_{k}(X)=1\}}\right]} \\
    &= \frac{\mathds{E}\left[\mu\left(X\right)\mathds{1}_{\{\hat{m}_{k}(X)=1\}}\right]}{\mathds{E}\left[\mu\left(X\right)\right]} \; \;  \left(\mathds{E}\left[\mu\left(X\right)\right] = \mathds{P}(Y=1) > 0  \right).
\end{align}

Finally we have,
\begin{align}
    Precision_{\hat{m}_{k}, t} = Recall_{\hat{m}_{k}, t} \times \frac{\mathds{E}\left[\mu\left(X\right)\right]}{\mathds{E}\left[\mathds{1}_{\{\hat{m}_{k}(X)=1\}}\right]}.
\end{align}

\begin{align}
    Precision_{\hat{m}_{k}, t} =  \frac{Recall_{\hat{m}_{k}, t} \times \mathds{P}\left(Y=1\right)}{\mathds{P}\left(\hat{m}_{k}(X)=1\right)}.
\end{align}
\end{proof}

\subsection{Univariate with linear regression function: depth=\texorpdfstring{$k$}{k}}
\label{sec-proof-univariate-linear}

\begin{proof}[Proof of \Cref{th:depthk-metrics}]
\label{prrof_th:depthk-metrics}

% \begin{figure}[h]
%     \centering
%     \includegraphics[width=0.8\linewidth]{images/dessin.pdf}
%     \caption{Caption}
%     \label{fig:2}
% \end{figure}
We assume that trees are grown to a maximum depth $k \in \mathbb{N}$, containing $2^k - 1$ internal split nodes. Let $\{\Myt_p^{h}\}_{h=1}^{2^k - 1}$ and $\{\Myt_p^{\star h}\}_{h=1}^{2^H - 1}$ denote the collections of split values for the standard and weighted trees, respectively, indexed level-by-level from left to right. We reorder these split values monotonically via permutations $\sigma$ and $\sigma^\star$ such that $\Myt_p^{\sigma(1)} \le \Myt_p^{\sigma(2)} \le \dots \le \Myt_p^{\sigma(2^k - 1)} \in [0,1]$ for the standard tree, and $\Myt_p^{ \sigma^\star(1)} \le \dots \le \Myt_p^{ \sigma^\star(2^k - 1)} \in [0,1]$ for the weighted tree. Let $t \in [0, 1]$ denote the decision threshold used by a CART tree to predict $\hat{Y} = 1$ within each leaf cell.

For simplicity, let $\Myt_1 \le \Myt_2 \le \dots \le \Myt_{2^k - 1}$ denote the ordered split values of a given tree (weighted or unweighted) and $A_1,A_2 \hdots A_{2^{k}}$ their associated ordered leaves. 
Under \Cref{assump-infinite-sample} and \Cref{assump-1}, for any tree of maximum depth $H \in \mathbb{N}$, we have:

\begin{align}
    \hat{m}_{k}(x) = \left\{ 
        \begin{array}{ll}
            1 & \text{if} \; t \in \left[0, \mathds{E}[\tilde{Y}|\tilde{X} \in A_{1}]\right] \\
            \mathds{1}_{\{x \geq \Myt_1\}} & \text{if} \; t \in \left[\mathds{E}[\tilde{Y}|\tilde{X} \in A_{1}], \mathds{E}[\tilde{Y}|\tilde{X} \in A_{2}]\right] \\
            \mathds{1}_{\{x \geq \Myt_2\}} & \text{if} \; t \in \left[\mathds{E}[\tilde{Y}|\tilde{X} \in A_{2}], \mathds{E}[\tilde{Y}|\tilde{X} \in A_{3}]\right] \\
            \vdots &\vdots\\
            \mathds{1}_{\{x \geq \Myt_{2^{k}-1}\}} & \text{if} \; t \in \left[\mathds{E}[\tilde{Y}|\tilde{X} \in A_{2^{k}-1}], \mathds{E}[\tilde{Y}|\tilde{X} \in A_{2^{k}}]\right] \\
            0 & \text{if} \; t \in \left[\mathds{E}[\tilde{Y}|\tilde{X} \in A_{2^{k}}], 1\right].
        \end{array}
    \right.
\end{align}

\paragraph{None tree} :
We have,
\begin{align}
    &\mathds{E}[Y|X \in A_{i}] \\&= \mathds{E}[\mathds{E}[Y|X]|X \in A_{i}] && (\text{Tower property}) \\
    & = \mathds{E}[P(Y=1|X)| X \in A_{i}] \\
    & = \mathds{E}[\alpha X + \beta| X \in A_{i}] \\
    & = \alpha \mathds{E}[ X| X \in A_{i}]+ \beta   \\
    &= \alpha \left(\frac{\Myt_{i-1}+\Myt_{i}}{2}\right) + \beta. && ( \sim \mathcal{U}([\Myt_{i-1},\Myt_{i}]))
\end{align}
and we know that the None tree iteratively cuts in the middle of the sub-segments. Thus,

\begin{align}
    \hat{m}_{k}(x) = \left\{ 
        \begin{array}{ll}
            1 & \text{if} \; t \in \left[0, \frac{1}{2^{k+1}}\alpha + \beta\right] \\
            \mathds{1}_{\{x \geq \frac{1}{2^{k}}\}} & \text{if} \; t \in \left[\frac{1}{2^{k+1}}\alpha + \beta, \frac{3}{2^{k+1}}\alpha + \beta\right] \\
            \mathds{1}_{\{x \geq \frac{2}{2^{k}}\}} & \text{if} \; t \in \left[\frac{3}{2^{k+1}}\alpha + \beta, \frac{5}{2^{k+1}}\alpha + \beta\right] \\
            \vdots &\vdots\\
            \mathds{1}_{\{x \geq \frac{2^{\ell:}-1}{2^{k}}\}} & \text{if} \; t \in \left[\frac{2\times(\ell-1)+1}{2^{k+1}}\alpha + \beta, \frac{2\times(\ell)+1}{2^{k+1}}\alpha + \beta\right] \\
            \vdots & \vdots\\
            \mathds{1}_{\{x \geq \frac{2^{k}-1}{2^{k}}\}} & \text{if} \; t \in \left[\frac{2\times(k-1-1)+1}{2^{k+1}}\alpha + \beta, \frac{2\times(k-1)+1}{2^{k+1}}\alpha + \beta\right] \\
            0 & \text{if} \; t \in \left[\mathds{E}[\frac{2\times(k-1)+1}{2^{k+1}}\alpha + \beta, 1\right].
        \end{array}
    \right.
\end{align}
\paragraph{Metrics}
When evaluated on $(X,Y)$ distribution, we have for all trees $\hat{m}_{k}(x)$,

\begin{align}
    \mathds{P}({Y}=1, \hat{{Y}}=1) &=\alpha\mathds{E}\left[\mathds{1}_{\{\hat{m}_{k}(X)=1\}} {X}  \right] + \beta \mathds{P}(\hat{m}_{k}(x)=1), \\
    \mathds{P}({Y}=0, \hat{{Y}}=1) &=\mathds{P}(\hat{m}_{k}({X})=1) - \alpha\mathds{E}\left[\mathds{1}_{\{\hat{m}_{k}({X})=1\}} {X}  \right] - \beta \mathds{P}(\hat{m}_{k}({X})=1) \\
    \mathds{P}({Y}=1, \hat{{Y}}=0),
    &=\alpha\mathds{E}\left[\mathds{1}_{\{\hat{m}_{k}({X})=0\}} {X}  \right] + \beta \mathds{P}(\hat{m}_{k}({X})=0).
\end{align}
Thus, the metric expressions in function of the split value $\Myt$ is the same for unweighted and weighted trees. Besides, under our main assumptions, we have 
\begin{align}
    \mathds{E}\left[\mathds{1}_{\{x\geq t\}} {X}  \right] &=  \int_{\Myt}^{1} x \mathrm{d}x = \frac{1}{2}(1-\Myt^2), \\
    \mathds{E}\left[\mathds{1}_{\{x < \Myt\}} {X}  \right] &= \int_{0}^{\Myt} x \mathrm{d}x = \frac{1}{2}\Myt^2,\\
    \mathds{P}(x\geq \Myt) &= 1-\Myt, \\
    \mathds{P}(x < \Myt) &= \Myt.
\end{align}

Thus,
    \begin{align}
    Precision = \left\{ 
        \begin{array}{ll}
            \alpha/2 + \beta & \text{if} \; t \in \left[0, \mathds{E}[\tilde{Y}|\tilde{X} \in A_{1}]\right] \\
            \frac{\alpha}{2} (1+\Myt_1) + \beta & \text{if} \; t \in \left[\mathds{E}[\tilde{Y}|\tilde{X} \in A_{1}], \mathds{E}[\tilde{Y}|\tilde{X} \in A_{2}]\right] \\
             \frac{\alpha}{2} (1+\Myt_2) + \beta & \text{if} \; t \in \left[\mathds{E}[\tilde{Y}|\tilde{X} \in A_{2}], \mathds{E}[\tilde{Y}|\tilde{X} \in A_{3}]\right] \\
            \vdots &\vdots\\
            \frac{\alpha}{2} (1+\Myt_{h}) + \beta  & \text{if} \; t \in \left[\mathds{E}[\tilde{Y}|\tilde{X} \in A_{h}], \mathds{E}[\tilde{Y}|\tilde{X} \in A_{h+1}]\right] \\
            \vdots &\vdots\\
            \frac{\alpha}{2} (1+\Myt_{2^{H}-1}) + \beta  & \text{if} \; t \in \left[\mathds{E}[\tilde{Y}|\tilde{X} \in A_{2^{H}-1}], \mathds{E}[\tilde{Y}|\tilde{X} \in A_{2^{H}}]\right] \\
            NA & \text{if} \; t \in \left[\mathds{E}[\tilde{Y}|\tilde{X} \in A_{2^{H}}], 1\right].
        \end{array}
    \right.
\end{align}

\begin{align}
    Recall = \left\{ 
        \begin{array}{ll}
            1 & \text{if} \; t \in \left[0, \mathds{E}[\tilde{Y}|\tilde{X} \in A_{1}]\right] \\
            \frac{\frac{\alpha}{2}(1-t_1^{2}) + \beta (1-\Myt_1)}{ \frac{\alpha}{2} + \beta} & \text{if} \; t \in \left[\mathds{E}[\tilde{Y}|\tilde{X} \in A_{1}], \mathds{E}[\tilde{Y}|\tilde{X} \in A_{2}]\right] \\
            \frac{\frac{\alpha}{2}(1-\Myt_2^{2}) + \beta (1-\Myt_2)}{ \frac{\alpha}{2} + \beta} & \text{if} \; t \in \left[\mathds{E}[\tilde{Y}|\tilde{X} \in A_{2}], \mathds{E}[\tilde{Y}|\tilde{X} \in A_{3}]\right] \\
            \vdots &\vdots\\
            \frac{\frac{\alpha}{2}(1-\Myt_h^{2}) + \beta (1-t_h)}{ \frac{\alpha}{2} + \beta}  & \text{if} \; t \in \left[\mathds{E}[\tilde{Y}|\tilde{X} \in A_{h}], \mathds{E}[\tilde{Y}|\tilde{X} \in A_{h+1}]\right] \\
            \vdots &\vdots\\
            \frac{\frac{\alpha}{2}(1-\Myt_{2^{H}-1}^{2}) + \beta (1-\Myt_{2^{H}-1})}{ \frac{\alpha}{2} + \beta}  & \text{if} \; t \in \left[\mathds{E}[\tilde{Y}|\tilde{X} \in A_{2^{H}-1}], \mathds{E}[\tilde{Y}|\tilde{X} \in A_{2^{H}}]\right] \\
            0 & \text{if} \; t \in \left[\mathds{E}[\tilde{Y}|\tilde{X} \in A_{2^{H}}], 1\right].
        \end{array}
    \right.
\end{align}

\end{proof}

\begin{proof}[Proof of \Cref{th:depthk-prec-rec-correspondance}]
Here we place ourselves under \Cref{assump-infinite-sample} and \Cref{assump-1}, using the previous result and notations of  \Cref{th:depthk-metrics} demonstrated in \Cref{prrof_th:depthk-metrics}. Let $t \in \mathds{R}$ be a threshold on output probabilities.

We assume to be given an ordered split position $\Myt_i$, such that
%where $g^{+}(\hat{m}_{k},t)$ yields the smallest split point $\Myt_p^{(\sigma(i))}$ such that 
$\mathds{E}[Y \mid \Myt_{i-1} \leq X \leq \Myt_{i}] \leq s \leq \mathds{E}[Y \mid \Myt_{i} \leq X \leq \Myt_{i+1}]$.
Thus, we denote $Precision_{\hat{m}_{k}, t}$ denotes the precision of decision tree $\hat{m}_{k}$ at probability threshold $t$.
We have, from \Cref{th:depthk-metrics}, that:
\label{proof_th:depthk-prec-rec-correspondance}
    \begin{align}
    &Recall_{\hat{m}_{k}, t} = \frac{\frac{\alpha}{2}(1-\Myt_i^{2}) + \beta (1-\Myt_i)}{ \frac{\alpha}{2} + \beta} \\
    &\Myt_i(\frac{\alpha}{2}\Myt_i+\beta) + (\frac{\alpha}{2}+ \beta)(Recall_{\hat{m}_{k}, t} - 1) = 0.
\end{align}
The only positive solution to this equation is : $\Myt_i  = \frac{-\beta + \sqrt{\beta^2 + 2\alpha(\frac{\alpha}{2}+ \beta)(1-Recall_{\hat{m}_{k}, t}(\Myt_i)})}{\alpha}$. 

Besides, from \Cref{th:depthk-metrics}, we also know that, 
\begin{align}
    Precision_{\hat{m}_{k}, t} =  \frac{\alpha}{2} \left( 1 +s_i\right) + \beta.
\end{align}
Thus, we obtain:
\begin{align}
    Precision_{\hat{m}_{k}, t}(R_i) = \frac{\alpha}{2}\left( 1 + \frac{-\beta + \sqrt{\beta^2 + 2\alpha(\frac{\alpha}{2}+ \beta)(1-Recall_{\hat{m}_{k}, t})}}{\alpha}\right) +\beta.
\end{align}
\end{proof}

%%%%%%%%%%%%%%%%%%%%%%
%%%%%%%%%%%%%%%%%%%%%%
%%%%%%%%%%%%%%%%%%%%%%
\subsection{Results within univariate increasing regression function setting}

\begin{proof}[Proof of \Cref{th:increasing-prec-rec-correspondance}]
\label{proof_th:increasing-prec-rec-correspondance}

Let $\mathcal{A}$ be the set of leaves/cells associated to the tree $\hat{m}_{k}$. Let $\mathcal{A}_{t}$ be the set of cells that are associated to an expectation higher than the threshold $t$.

We have,
\begin{align}
    \mathds{E}\left[ \mathds{1}_{\{\hat{m}_{k}(X)=1\}}\right] = \mathds{P}\left(\hat{m}_{k}(X)=1\right) = \sum_{A_t \in \mathcal{A}_{t}} \int_{A_t} f_{X}(x) \mathrm{d}x =\sum_{A_t \in \mathcal{A}_{t}} \mathds{P}\left(X \in A_t \right).
\end{align}

\begin{align}
    \mathds{E}\left[ \mu(X) \mathds{1}_{\{\hat{m}_{k}(X)=1\}}\right] 
    &=\sum_{A_t \in \mathcal{A}_{t}} \int_{A_t} \mu(x) f_{X}(x)dx.
\end{align}

Thus,
\begin{align}
    Precision = \frac{\sum_{A_t \in \mathcal{A}_{t}} \int_{A_t} \mu(x) f_{X}(x)dx}{\sum_{A_t \in \mathcal{A}_{t}} \mathds{P}\left(X \in A_t \right)},
\end{align}
and, 
\begin{align}
    Recall = \frac{\sum_{A_t \in \mathcal{A}_{t}} \int_{A_t} \mu(x) f_{X}(x)dx}{\sum_{A_t \in \mathcal{A}} \int_{A_t} \mu(x) f_{X}(x)dx}.
\end{align}
%\as{Copules.}

Grant \Cref{assump-infinite-sample} and  \Cref{assump:increasing-regfunc}.
We assume that trees are grown to a maximum depth $k \in \mathbb{N}$, containing $2^k - 1$ internal split nodes. Let $\{\Myt_p^{h}\}_{h=1}^{2^k - 1}$ and $\{\Myt_p^{\star h}\}_{h=1}^{2^k - 1}$ denote the collections of split values for the standard and weighted trees, respectively, indexed level-by-level from left to right. We reorder these split values monotonically via permutations $\sigma$ and $\sigma^\star$ such that $\Myt_p^{\sigma(1)} \le \Myt_p^{\sigma(2)} \le \dots \le \Myt_p^{\sigma(2^k - 1)} \in [0,1]$ for the standard tree, and $\Myt_p^{\sigma^\star(1)} \le \dots \le \Myt_p^{ \sigma^\star(2^k - 1)} \in [0,1]$ for the weighted tree. Let $t \in [0, 1]$ denote the decision threshold used by a CART tree to predict $\hat{Y} = 1$ within each leaf cell.

\Cref{assump:increasing-regfunc} leads to the fact that given a threshold $t \in \mathds{R}$ on output probabilities, the tree prediction takes the following form:
\begin{align}
    \hat{m}_{k}(x) = \mathds{1}_{\{x \geq g^{+}(\hat{m}_{k},t)\}},
\end{align}
where $g^{+}(\hat{m}_{k},t)$ yields the smallest split point $\Myt_p^{i}$ such that $\mathds{E}[Y \mid \Myt_p^{\sigma(i-1)} \leq X \leq \Myt_p^{\sigma(i)}] \leq s \leq \mathds{E}[Y \mid \Myt_p^{\sigma(i)} \leq X \leq \Myt_p^{\sigma(i+1)}]$. In the following, we denote by $(\hat{m}_{k},t)$, a tree associated to a threshold $t$ on output probabilities.

We assume that the threshold $t$ is fixed such that we are not in a pathological case, i.e. the tree does not systematically predict $Y=1$ or $Y=0$.

%For simplicity, let $\Myt_1 \le \Myt_2 \le \dots \le \Myt_{2^H - 1}$ denote the ordered split thresholds of a given tree (weighted or unweighted). 

We have,
\begin{align}
    \mathds{E}\left[ \mathds{1}_{\{D_{\mathbf{t}}(X)=1\}} \right] &= \mathds{E}\left[ \mathds{1}_{\{X \geq g^{+}(\hat{m}_{k},t)\}} \right] = \mathds{P}(X \geq g^{+}(\hat{m}_{k},t)).
\end{align}

We also have,
\begin{align}
    Recall_{\hat{m}_{k}, t} &= \frac{\mathds{E}\left[ \mathds{1}_{\{X\geq g^{+}(\hat{m}_{k},t)\}} \mu(X)\right] }{\mathds{P}(Y=1)} \\
    &= \frac{ \int_{g^{+}(\hat{m}_{k},t)}^{b} \mu(X)f_{X}(x)\mathrm{d}x}{\mathds{P}(Y=1)} \\
    &= \frac{F_{Y}(b)-F_{Y}(g^{+}(\hat{m}_{k},t))}{\mathds{P}(Y=1)}.
\end{align}
Thus,
\begin{align}
    g^{+}(\hat{m}_{k},t) = F_{Y}^{-1}\left[ F_{Y}(b)- Recall_{\hat{m}_{k}, t} \mathds{P}(Y=1) \right].
\end{align}

On the other hand,
\begin{align}
    Precision_{\hat{m}_{k}, t} &= \frac{\mathds{E}\left[ \mathds{1}_{\{X\geq g^{+}(\hat{m}_{k},t)\}} \mu(X)\right]}{\mathds{E}\left[ \mathds{1}_{\{D_{g^{+}(\hat{m}_{k},t)}(X)=1\}} \right]} \\
    &= \frac{\int_{g^{+}(\hat{m}_{k},t)}^{b} \mu(X)f_{X}(x)\mathrm{d}x}{ \mathds{P}(X\geq g^{+}(\hat{m}_{k},t))}.
\end{align}

Finally,
\begin{align}
    Precision_{\hat{m}_{k}, t} &= \frac{\int_{F_{Y}^{-1}\left[ F_{Y}(b)- Recall_{\hat{m}_{k}, t} \mathds{P(Y=1)} \right]}^{b} \mu(X)f_{X}(x)\mathrm{d}x}{ \mathds{P}(X\geq F^{-1}\left[ F_{Y}(b)- Recall_{\hat{m}_{k}, t} \mathds{P}(Y=1) \right])} \\
    &= \frac{F_{Y}(b) - F_{Y}\left( F_{Y}^{-1}\left[ F_{Y}(b)- Recall_{\hat{m}_{k}, t} \mathds{P}(Y=1) \right] \right)}{ \mathds{P}(X\geq F_{Y}^{-1}\left[ F_{Y}(b)- Recall_{\hat{m}_{k}, t} \mathds{P}(Y=1) \right])} \\
    &= \frac{Recall_{\hat{m}_{k}, t} \mathds{P}(Y=1)}{ \mathds{P}(X\geq F_{Y}^{-1}\left[ F_{Y}(b)- Recall_{\hat{m}_{k}, t} \mathds{P}(Y=1) \right])}.
\end{align}
\end{proof}

\end{document}